\documentclass[preprint,number,1p]{elsarticle}

\usepackage[english]{babel}

\usepackage[fig,chgbar]{fmocdmac}

\AtEndPreamble
  {

  }

\newtheorem{definition}{Definition}

\newtheorem{problem}{Problem}
\newtheorem{proposition}{Proposition}
\newtheorem{lemma}{Lemma}

\newtheorem{theorem}{Theorem}
\newtheorem{corollary}{Corollary}

\cmdtxtoparname{SHACL}

\cmdtxtoparname{SCL}
\cmdtxtoparname{MSCL}
\DeclareRobustCommand{\ESCL}
  {\ensuremath{\exists}\SCL}
\DeclareRobustCommand{\USCL}
  {\ensuremath{\forall}\SCL}

\cmdmthsym{isA}
\cmdmthsymelm{hasShape}[\Sigma]

\usepackage{listings,multicol}

\usepackage{stmaryrd}
\usepackage{xparse}
\usepackage{listings}
\usepackage{caption}
\usepackage{dsfont}
\usepackage{cleveref}
\newcommand{\trdf}[0]{RDF\xspace}

\newcommand{\vShapeDocument}[0]{\ensuremath{M}}
\newcommand{\trip}[3]{\ensuremath{\langle #1,\allowbreak #2,\allowbreak #3 \rangle}}

\newcommand{\vn}[0]{\ensuremath{n}}
\newcommand{\vj}[0]{\ensuremath{j}}

\newcommand{\templatesat}[0]{template satisfiability}
\newcommand{\vx}[0]{\ensuremath{x}}
\newcommand{\vy}[0]{\ensuremath{y}}
\newcommand{\vz}[0]{\ensuremath{z}}
\newcommand{\conc}[0]{\ensuremath{\texttt{c}}}
\newcommand{\conce}[0]{\ensuremath{\texttt{q}}}
\NewDocumentCommand\nodesof{mg}{%
    \ensuremath{\mathsf{nodes}(\mathsf{#1}\IfNoValueTF{#2}{}{,\mathsf{#2}})}%
}
\newcommand{\shapesof}[1]{\ensuremath{\mathsf{shapes}(\mathsf{#1})}}
\newcommand{\GsigmaModels}[4]{\ensuremath{\llbracket #3 \rrbracket^{#4,#1,#2}}}
\newcommand{\vtrue}{\ensuremath{\mathsf{True}}}
\newcommand{\vfalse}{\ensuremath{\mathsf{False}}}
\newcommand{\vundefined}{\ensuremath{\mathsf{Undefined}}}
\newcommand{\isfaithful}[3]{\ensuremath{\ensuremath{(#1,#2) \models #3}}}
\newcommand{\isnotfaithful}[3]{\ensuremath{\ensuremath{(#1,#2) \not \models #3}}}
\newcommand{\minustargets}[1]{\ensuremath{#1^{\setminus t}}}
\newcommand{\uri}[1]{\ensuremath{{:}{\texttt{#1}}}}

\newcommand{\assignmentsPartial}[2]{\ensuremath{A^{#1,#2}}}
\newcommand{\assignmentsTotal}[2]{\ensuremath{A_{T}^{#1,#2}}}
\newcommand{\GsigmaModelsThree}[4]{\ensuremath{\llbracket #3 \rrbracket^{#4,#1,#2}}}
\newcommand{\hasshapePredicate}[0]{\ensuremath{\Sigma}\xspace}
\newcommand{\hasshapePredicateS}[1]{\ensuremath{\Sigma_{#1}}\xspace}
\newcommand{\hasshape}[2]{\ensuremath{\Sigma_{#2}(#1)}\xspace}
\newcommand{\isA}[0]{\texttt{isA}\xspace}
\newcommand{\induced}[1]{\ensuremath{{\texttt{#1}^{\tau}}}}
\newcommand{\tripsquare}[3]{\ensuremath{{<}#1,\allowbreak #2,\allowbreak #3{>}}}
\newcommand{\sconst}{\texttt{s}}
\newcommand{\sigmagrammar}{\ensuremath{\varsigma}}
\newcommand{\qexists}[1]{\ensuremath{\exists(#1)}}
\newcommand{\qforall}[1]{\ensuremath{\forall(#1)}}
\newcommand{\axiomatisation}[1]{\ensuremath{\axiom(#1)}}
\newcommand{\boundedaxiomatisation}[1]{\ensuremath{\bar{\axiom}(#1)}}
\newcommand{\axiom}[0]{\ensuremath{\alpha}}
\newcommand{\hasshapenosubscript}[2]{\ensuremath{\Sigma(#1)}\xspace}
\newcommand{\countingaxiomatisation}[1]{\ensuremath{\axiom(#1)}}
\newcommand{\vf}[0]{\ensuremath{f}}
\newcommand{\constantList}[0]{\ensuremath{C}}
\newcommand{\vShapes}[0]{\ensuremath{\bar{S}}}
\newcommand{\tr}[0]{\ensuremath{r}}
\newcommand{\transeq}[0]{\ensuremath{\; \doteq \;}}

\cmdmthfun{til}[\eth]

\renewcommand{\S}{\SSym\xspace}
\newcommand{\Z}{\ZSym\xspace}
\renewcommand{\A}{\ASym\xspace}
\newcommand{\T}{\TSym\xspace}
\newcommand{\D}{\DSym\xspace}
\renewcommand{\E}{\ESym\xspace}
\renewcommand{\O}{\OSym\xspace}
\newcommand{\C}{\CSym\xspace}
\renewcommand{\X}{\ensuremath{\varnothing}\xspace}

\newcommand{\ALC}{\ensuremath{\mathcal{ALC}}\xspace}
\newcommand{\ALCOI}{\ensuremath{\mathcal{ALCOI}}\xspace}
\newcommand{\ALCOIQ}{\ensuremath{\mathcal{ALCOIQ}}\xspace}

\newcommand{\genericiri}[1]{\textnormal{\ensuremath{{:}{\texttt{#1}}}}}
\newcommand{\sh}[1]{\textnormal{\ensuremath{\texttt{sh}{:}{\texttt{#1}}}}}
\newcommand{\rdf}[1]{\textnormal{\ensuremath{\texttt{rdf}{:}{\texttt{#1}}}}}

\newcommand{\vshape}[0]{\texttt{s}}
\newcommand{\vshapet}[0]{\ensuremath{t}}
\newcommand{\vshapec}[0]{\ensuremath{d}}

\newcommand{\isIRI}[0]{\text{IRI}}
\newcommand{\isLiteral}[0]{\text{literal}}
\newcommand{\isBlank}[0]{\text{blank}}
\newcommand{\hasdatatype}[0]{\text{dt}}

\AfterEndPreamble{

\usetikzlibrary{arrows.meta}

\tikzstyle{every node} +=
  [thick, dashed]

\tikzstyle{unk} =
  [brown]
\tikzstyle{dec} =
  [rounded rectangle, draw = blue, blue]
\tikzstyle{decthm} =
  [dec, very thick, solid]
\tikzstyle{und} =
  [rectangle, draw = red, red]
\tikzstyle{undthm} =
  [und, very thick, solid]
\tikzstyle{fmp} =
  [fill = blue!20!white]
\tikzstyle{notfmp} =
  [fill = red!10!white]

\tikzstyle{every edge} +=
  [very thin, dashed]

\tikzstyle{decder} =
  [blue, thick, solid]
\tikzstyle{decimp} =
  [decder, dotted]
\tikzstyle{undder} =
  [red, thick, solid]
\tikzstyle{undimp} =
  [undder, dotted]

\newcommand{\figfrg}{
\begin{tikzpicture}[node distance = 4em, bend angle = 22.5]

  \node [decthm, fmp]
        (0)
        []
        {\X};

  \node [dec, fmp]
        (A)
        [above of = 0, yshift = 1em]
        {\A};
  \node [unk, notfmp]
        (O)
        [left of = A, xshift = -4em]
        {\O};
  \node [dec, fmp]
        (S)
        [left of = O, xshift = -4em]
        {\S};
  \node [dec, notfmp]
        (C)
        [right of = A, xshift = 4em]
        {\C};
  \node [dec, fmp]
        (E)
        [right of = C, xshift = 4em]
        {\E};

  \node [undthm, notfmp]
        (SE)
        [above of = A, yshift = 1.5em]
        {\S\E};
  \node [unk, notfmp]
        (AO)
        [left of = SE]
        {\A\O};
  \node [unk, notfmp]
        (SC)
        [left of = AO]
        {\S\C};
  \node [undthm, notfmp]
        (SO)
        [left of = SC]
        {\S\O};
  \node [dec, fmp]
        (SA)
        [left of = SO]
        {\S\A};
  \node [dec, notfmp]
        (AC)
        [right of = SE]
        {\A\C};
  \node [unk, notfmp]
        (EO)
        [right of = AC]
        {\E\O};
  \node [dec, notfmp]
        (EC)
        [right of = EO]
        {\E\C};
  \node [dec, fmp]
        (AE)
        [right of = EC]
        {\A\E};

  \node [und, notfmp]
        (SAE)
        [above of = SE, yshift = 1.5em]
        {\S\A\E};
  \node [und, notfmp]
        (SEO)
        [left of = SAE, xshift = -1.5em]
        {\S\E\O};
  \node [undthm, notfmp]
        (SAC)
        [left of = SEO]
        {\S\A\C};
  \node [und, notfmp]
        (SAO)
        [left of = SAC]
        {\S\A\O};
  \node [undthm, notfmp]
        (SEC)
        [right of = SAE, xshift = 1.5em]
        {\S\E\C};
  \node [unk, notfmp]
        (AEO)
        [right of = SEC]
        {\A\E\O};
  \node [dec, notfmp]
        (AEC)
        [right of = AEO]
        {\A\E\C};

  \node [undthm, notfmp]
        (SZAE)
        [above of = SAE, xshift = -2.5em]
        {\S\Z\A\E};
  \node [decthm, fmp]
        (SZAD)
        [above of = SA, yshift = 5.5em, xshift = -1.5em]
        {\S\Z\A\D};
  \node [und, notfmp]
        (SADE)
        [above of = SAE, xshift = 2.5em]
        {\S\A\D\E};
  \node [decthm, fmp]
        (ZADE)
        [above of = AE, yshift = 5.5em, xshift = 1.5em]
        {\Z\A\D\E};

  \node [und, notfmp]
        (SZADE)
        [node distance = 9em, above of = SAE]
        {\S\Z\A\D\E};
  \node [und, notfmp]
        (SAEO)
        [left of = SZADE, xshift = -2em]
        {\S\A\E\O};
  \node [und, notfmp]
        (SZADC)
        [left of = SAEO, xshift = -1em]
        {\S\Z\A\D\C};
  \node [und, notfmp]
        (SZADO)
        [left of = SZADC, xshift = -1em]
        {\S\Z\A\D\O};
  \node [und, notfmp]
        (SAEC)
        [right of = SZADE, xshift = 2em]
        {\S\A\E\C};
  \node [unk, notfmp]
        (ZADEO)
        [right of = SAEC, xshift = 1em]
        {\Z\A\D\E\O};
  \node [decthm, notfmp]
        (ZADEC)
        [right of = ZADEO, xshift = 1em]
        {\Z\A\D\E\C};

  \node [und, notfmp]
        (SZATE)
        [above of = SZADE, xshift = -2.5em]
        {\S\Z\A\T\E};
  \node [decthm, notfmp]
        (SZATD)
        [above of = SZAD, yshift = 5em]
        {\S\Z\A\T\D};
  \node [und, notfmp]
        (SATDE)
        [above of = SZADE, xshift = 2.5em]
        {\S\A\T\D\E};
  \node [unk]
        (ZATDE)
        [above of = ZADE, yshift = 5em]
        {\Z\A\T\D\E};

  \node [und, notfmp]
        (SZATDE)
        [node distance = 8em, above of = SZADE]
        {\S\Z\A\T\D\E};
  \node [und, notfmp]
        (SZATDO)
        [above of = SZATD]
        {\S\Z\A\T\D\O};
  \node [und, notfmp]
        (SZATDC)
        [node distance = 6em, right of = SZATDO]
        {\S\Z\A\T\D\C};
  \node [unk, notfmp]
        (ZATDEC)
        [above of = ZATDE]
        {\Z\A\T\D\E\C};
  \node [unk, notfmp]
        (ZATDEO)
        [node distance = 6em, left of = ZATDEC]
        {\Z\A\T\D\E\O};

  \node [und, notfmp]
        (SZATDEO)
        [above of = SAEO, yshift = 8em]
        {\S\Z\A\T\D\E\O};
  \node [und, notfmp]
        (SZATDOC)
        [above of = SZATDO, xshift = 3em]
        {\S\Z\A\T\D\O\C};
  \node [und, notfmp]
        (SZATDEC)
        [above of = SAEC, yshift = 8em]
        {\S\Z\A\T\D\E\C};
  \node [unk, notfmp]
        (ZATDEOC)
        [above of = ZATDEC, xshift = -3em]
        {\Z\A\T\D\E\O\C};

  \path[Latex-Latex]
    (0)     edge  [decder]
                  (A)
    (0.west)
            edge  [decder]
                  (S)
    ;
  \path[-Latex]
    (0)     edge  []
                  (O)
            edge  [decimp]
                  (C)
    (0.east)
            edge  [decimp]
                  (E)
    ;

  \path[Latex-Latex]
    (A)     edge  [decder]
                  (SA.south)
    (O.north)
            edge  []
                  (AO.south)
    (S.north)
            edge  [decder]
                  (SA.south)
    ;
  \path[-Latex]
    (A.north)
            edge  []
                  (AO.south)
            edge  [decimp]
                  (AC.south)
    (A)     edge  [decimp]
                  (AE.south)
    (O.north)
            edge  []
                  (SO.south)
            edge  []
                  (EO.south)
    (S.north)
            edge  []
                  (SE.south)
            edge  []
                  (SC.south)
            edge  []
                  (SO.south)
            edge  [decder]
                  (SA.south)
    (C.north)
            edge  []
                  (SC.south)
            edge  [decder]
                  (AC.south)
            edge  [decder]
                  (EC.south)
    (E.north)
            edge  []
                  (SE.south)
            edge  []
                  (EO.south)
            edge  [decimp]
                  (EC.south)
            edge  [decder]
                  (AE.south)
    ;

  \path[-Latex]
    (SE.north)
            edge  [undder]
                  (SAE.south)
            edge  [undimp]
                  (SEC.south)
            edge  [undimp]
                  (SEO.south)
    (AO.north)
            edge  []
                  (SAO.south)
            edge  []
                  (AEO.south)
    (SC.north)
            edge  []
                  (SAC.south)
            edge  []
                  (SEC.south west)
    (SO.north)
            edge  [undder]
                  (SAO.south)
            edge  [undder]
                  (SEO.south)
    (SA.north)
            edge  []
                  (SAE.south)
            edge  []
                  (SAC.south)
            edge  []
                  (SAO.south)
    (AC.north)
            edge  []
                  (SAC.south)
            edge  [decder]
                  (AEC.south)
    (EO.north)
            edge  []
                  (SEO.south east)
            edge  []
                  (AEO.south)
    (EC.north)
            edge  []
                  (SEC.south)
            edge  [decder]
                  (AEC.south)
    (AE.north)
            edge  []
                  (SAE.south)
            edge  [decimp]
                  (AEC.south)
            edge  []
                  (AEO.south)
    ;

  \path[-Latex]
    (SA.north)
            edge  [decimp]
                  (SZAD.south)
    (AE.north)
            edge  [decder]
                  (ZADE.south)
    ;

  \path[-Latex]
    (SAE)   edge  [undimp]
                  (SZAE)
            edge  [undder]
                  (SADE)
    ;

  \path[-Latex]
    (SAE)   edge  [undimp, bend angle = 50, bend left]
                  (SAEO.south east)
            edge  [undimp, bend angle = 50, bend right]
                  (SAEC.south west)
    (SEO)   edge  [undder]
                  (SAEO.south)
    (SAC.north)
            edge  [undder]
                  (SZADC.south)
            edge  [undder, bend left]
                  (SAEC.south west)
    (SAO.north)
            edge  [undder]
                  (SZADO.south)
            edge  [undder]
                  (SAEO.south)
    (SEC)   edge  [undder]
                  (SAEC.south)
    (AEO.north)
            edge  []
                  (ZADEO.south)
            edge  [bend right]
                  (SAEO.south east)
    (AEC.north)
            edge  [decder]
                  (ZADEC.south)
            edge  []
                  (SAEC.south)
    ;

  \path[-Latex]
    (SZAE.north)
            edge  [undder]
                  (SZADE.south)
    (SZAD)  edge  []
                  (SZADE.south)
    (SZAD.north)
            edge  []
                  (SZADO.south)
            edge  []
                  (SZADC.south)
    (SADE.north)
            edge  [undimp]
                  (SZADE.south)
    (ZADE)  edge  []
                  (SZADE.south)
    (ZADE.north)
            edge  []
                  (ZADEO.south)
            edge  [decimp]
                  (ZADEC.south)
    ;

  \path[-Latex]
    (SZAE.north)
            edge  [undder]
                  (SZATE.south)
    (SZAD.north)
            edge  [decimp, bend left]
                  (SZATD.south)
    (SADE.north)
            edge  [undder]
                  (SATDE.south)
    (ZADE.north)
            edge  [bend right]
                  (ZATDE.south)
    ;

  \path[-Latex]
    (SZADE) edge  [undder]
                  (SZATDE)
    (SZADC) edge  [undder]
                  (SZATDC.south)
    (SZADO) edge  [undder, bend angle = 35, bend right]
                  (SZATDO)
    (ZADEO) edge  []
                  (ZATDEO.south)
    (ZADEC) edge  [bend angle = 35, bend left]
                  (ZATDEC)
    ;

  \path[-Latex]
    (SAEO)  edge  [undder]
                  (SZATDEO.south)
    (SAEC)  edge  [undder]
                  (SZATDEC.south)
    ;

  \path[-Latex]
    (SZATE) edge  [undder]
                  (SZATDE)
    (SZATD) edge  []
                  (SZATDE)
    (SZATD.north)
            edge  []
                  (SZATDO.south)
            edge  []
                  (SZATDC.south)
    (SATDE) edge  [undimp]
                  (SZATDE)
    (ZATDE) edge  []
                  (SZATDE)
    (ZATDE.north)
            edge  []
                  (ZATDEO.south)
            edge  []
                  (ZATDEC.south)
    ;

  \path[-Latex]
    (SZATDE.north)
            edge  [undder]
                  (SZATDEO.south)
            edge  [undder]
                  (SZATDEC.south)
    (SZATDC.north)
            edge  [undder]
                  (SZATDEC.south west)
            edge  [undder]
                  (SZATDOC.south)
    (SZATDO.north)
            edge  [undder]
                  (SZATDEO.south)
            edge  [undder]
                  (SZATDOC.south)
    (ZATDEO.north)
            edge  []
                  (SZATDEO.south east)
            edge  []
                  (ZATDEOC.south)
    (ZATDEC.north)
            edge  []
                  (SZATDEC.south)
            edge  []
                  (ZATDEOC.south)
    ;

\end{tikzpicture}
}

}

\AfterEndPreamble{

\newcommand{\tabtargets}{
\begin{tabular}{|l|c|}
  \hline                                              & \\[-0.95em]
  \textbf{Types of target declarations in $\vshapet$} &
  \textbf{\SCL target axiom}
  \\\hline                                            & \\[-0.95em]
  Node target (node ${\conElm}$)                      &
  ${\hasShapeElm[\vshape](\conElm)}$
  \\\hline                                            & \\[-0.95em]
  Class target (class \texttt{c})                     &
  ${\forall \varElm \ldotp \isASym(\varElm, \conElm) \rightarrow
  \hasShapeElm[\vshape](\varElm)}$
  \\\hline                                            & \\[-0.95em]
  Subjects-of target (relation ${\RRel}$)             &
  ${\forall \varElm, \yvarElm \ldotp \RRel(\varElm, \yvarElm) \rightarrow
  \hasShapeElm[\vshape](\varElm)}$
  \\\hline                                            & \\[-0.95em]
  Objects-of target (relation ${\RRel}$)              &
  ${\forall \varElm, \yvarElm \ldotp \RRel[][-](\varElm, \yvarElm) \rightarrow
  \hasShapeElm[\vshape](\varElm)}$
  \\\hline
\end{tabular}
}

\newcommand{\tabcomponents}{
\begin{tabular}{|@{\,}c@{\,}|l|l|c|}
  \hline          &                             &                           &
  \\[-0.95em]
  \textbf{Abbr.}  & \textbf{Name}               & \textbf{\SHACL component} &
  \textbf{\SCL expression}
  \\\hline        &                             &                           &
  \\[-0.95em]
  {\D}            & Property pair disjointness  & \sh{disjoint}             &
  ${\neg \exists \yvarElm \ldotp \piFrm(\varElm, \yvarElm) \wedge \RRel(\varElm,
  \yvarElm)}$
  \\\hline        &                             &                           &
  \\[-0.95em]
  {\E}            & Property pair equality      & \sh{equals}               &
  ${\forall \yvarElm \ldotp \piFrm(\varElm, \yvarElm) \leftrightarrow
  \RRel(\varElm, \yvarElm)}$
  \\\hline        &                             &                           &
  \\[-0.95em]
  {\O}            & Property pair order         & \begin{tabular}{@{}l@{}}
                                                    \sh{lessThan} \\
                                                    \sh{lessThanOrEquals}
                                                  \end{tabular}             &
  ${\varElm <^{\pm} \yvarElm}$ and ${\varElm \leq^{\pm} \yvarElm}$
  \\\hline        &                             &                           &
  \\[-0.95em]
  {\C}            & Cardinality constraints     & \begin{tabular}{@{}l@{}}
                                                    \sh{qualifiedValueShape} \\
                                                    \sh{qualifiedMinCount} \\
                                                    \sh{qualifiedMaxCount}
                                                  \end{tabular}             &
  \begin{tabular}{@{}l@{}}
    ${\exists^{\geq n} \yvarElm \ldotp \piFrm(\varElm, \yvarElm) \wedge
    \psiFrm(\yvarElm)}$ \\
    with $n \neq 1$
  \end{tabular}
  \\\hline        &                             &                           &
  \\[-0.95em]
  {\S}            & Sequence paths              & \SHACL list                       &
  ${\exists \zvarElm \ldotp \piFrm(\varElm, \zvarElm) \wedge \piFrm(\zvarElm,
  \yvarElm)}$
  \\\hline        &                             &                           &
  \\[-0.95em]
  {\Z}            & Zero-or-one paths           & \sh{zeroOrOnePath}        &
  ${\varElm = \yvarElm \vee \piFrm(\varElm,  \yvarElm)}$
  \\\hline        &                             &                           &
  \\[-0.95em]
  {\A}            & Alternative paths           & \sh{alternativePath}      &
  ${\piFrm(\varElm, \yvarElm) \vee \piFrm(\varElm, \yvarElm)}$
  \\\hline        &                             &                           &
  \\[-0.95em]
  {\T}            & Transitive paths            & \begin{tabular}{@{}l@{}}
                                                    \sh{zeroOrMorePath} \\
                                                    \sh{oneOrMorePath}
                                                  \end{tabular}             &
  ${(\piFrm(\varElm, \yvarElm))^{\star}}$
  \\\hline
\end{tabular}
}

}

\hypersetup
  {
  pdftitle  = {Satisfiability and Containment of Recursive SHACL},
  pdfauthor = {P. Pareti, G. Konstantinidis, F. Mogavero}
  }

\begin{document}

  \newpage

  \setcounter{page}{1}

  \title{Satisfiability and Containment of Recursive \SHACL}

  \author[1]{Paolo Pareti\corref{cor1}}
  \ead{paolo.pareti@winchester.ac.uk}

  \author[2]{George Konstantinidis}
  \ead{g.konstantinidis@soton.ac.uk}

  \author[3]{Fabio Mogavero}
  \ead{fabio.mogavero@unina.it}

  \affiliation[1]{organization={University of Winchester},
  addressline={Sparkford Road},
  postcode={SO22 4NR},
  city={Winchester},
  country={United Kingdom}}

  \affiliation[2]{organization={University of Southampton},
  addressline={University Road},
  postcode={SO17 1BJ},
  city={Southampton},
  country={United Kingdom}}
  \affiliation[3]{organization={Universita degli Studi di Napoli Federico II},
  addressline={Corso Umberto I 40},
  postcode={80138},
  city={Napoli},
  country={Italy}}

  \cortext[cor1]{Corresponding author.}



\begin{abstract}

The \emph{Shapes Constraint Language} (\SHACL) is the recent W3C recommendation
language for validating RDF data, by verifying certain \emph{shapes} on graphs.
Previous work has largely focused on the validation problem, while the standard
decision problems of \emph{satisfiability} and \emph{containment}, crucial for
design and optimisation purposes, have only been investigated for simplified
versions of \SHACL.
Moreover, the \SHACL specification does not define the semantics of
recursively-defined constraints, which led to several alternative
\emph{recursive} semantics being proposed in the literature.
The interaction between these different semantics and important decision
problems has not been investigated yet.
In this article we provide a comprehensive study of the different features of
\SHACL, by providing a translation to a new first-order language, called \SCL,
that precisely captures the semantics of \SHACL.
We also present \MSCL, a second-order extension of \SCL, which allows us to
define, in a single formal logic framework, the main recursive semantics of
\SHACL.
Within this language we also provide an effective treatment of filter
constraints which are often neglected in the related literature.
Using this logic we provide a detailed map of (un)decidability and complexity
results for the satisfiability and containment decision problems for different
\SHACL fragments.
Notably, we prove that both problems are undecidable for the full language, but
we present decidable combinations of interesting features, even in the face of
recursion.


\end{abstract}


  \maketitle



\section{Introduction}

\emph{Data validation} is the process of ensuring data is clean,
correct, and useful.
The \emph{Shapes Constraint Language} (\SHACL, for short)~\citep{KK17} is a
recent W3C recommendation language for validation of data in the
form of RDF graphs~\citep{CWG14} and is quickly becoming an established technology.
Similar to ontology languages like OWL~\citep{OWL2}, \SHACL can be seen as a language that strictly imposes a schema on
graph data models, such as RDF, which are inherently schemaless.
Unlike ontology languages, \SHACL focuses more on the structural properties of a graph rather than the semantic ones, 
and it is not intended for inference. 
A \SHACL \emph{shape graph}, which we will call \SHACL \emph{document} in this paper, validates an RDF graph by evaluating it against a set of
constraints.
In \SHACL, constraints are modelled as a set of \emph{shapes} which,
intuitively, define the structure that certain entities in the graph must
conform to.

Despite its ongoing widespread adoption (see~\cite{pareti2021shaclreview} for a recent review), 
many aspects of \SHACL remain unexplored.
Several important theoretical properties of the language have
not been studied. Among these are the decidability and complexity of satisfiability and containment of \SHACL documents, and this is the main focus of this work.
These problems have important roles in the design and optimisation of SHACL applications. For example, satisfiability can support an editor that checks whether a developing SHACL document becomes inconsistent, or an integration system that tracks conflicts when integrating datasets subject to different SHACL documents.
Containment (and consequently document equivalence, which is based on containment) studies whether one document is subsumed by another one and has important applications in optimisation and minimisation of documents \cite{AHV95}, detecting independence of documents from data updates \cite{levy1993queries}, data integration \cite{lenzerini2002data,konstantinidis2011scalable}, maintenance of integrity constraints \cite{gupta1994constraint}, and semantic data caching \cite{dar1996semantic}.
We study both problems for entire documents and individual shapes. At the level of shapes, an unsatisfiable shape constraint might not necessarily
cause the unsatisfiability of a whole SHACL document, but it is likely an indication of a
design error. Being able to decide containment/equivalence for individual shapes offers more design choices to the author of a SHACL document and it is an avenue for optimization. Moreover, shape containment is a problem that has been studied in literature in connection with important practical applications such as type checking in software \cite{leinberger2019type, martin2020shapecontainment}.

Note that satisfiability has two prevalent versions in the literature: finite and infinite/unrestricted. These adjectives refer to the size of an underlying model (here a data graph) that whenever exists proves the theory (e.g., SHACL document) satisfiable. In practice, finite satisfiability is what we are usually interested in. Commonly however, infinite satisfiability is a starting point for theoretical studies as, being less restricted, it is often considered easier to address. Indeed, very often the techniques for deciding finite satisfiability are revealed through studying the infinite case as a first approximation. As well, if the infinite case provides a quick decidability result then finite decidability also holds. Moreover, there are cases where a theory describes only part of the model (possibly finite), but there is an infinite domain for numbers or other parts of the theory. 
In this case (which is possible for SHACL), unrestricted satisfiability is of our interest. Other practical cases that might imply infinite models, appear in the face of reasoning with intensional knowledge, e.g., ontology TBoxes  (see \cite{pareti2019shacl} for a work on SHACL in combination with reasoning).

Additionally, the W3C specification does not define the semantics of \SHACL in 
its full generality, since it does not describe how to handle recursive constraints.
Recent work~\citep{CRS18} has suggested a theoretical modelling of the language
in order to formally define a recursive semantics; the same work also studied
the complexity of the validation problem.
Alternative recursive semantics for \SHACL have been further suggested in
\citep{ACORSS20}.

In this article, we extend~\citep{pareti2020} to capture \SHACL
semantics using mathematical logic.
This is an important contribution on its own, as it offers a standard and well-established
modelling of the
language, where \SHACL documents are translated into logical sentences that are 
interpreted in the usual way.
This makes \SHACL semantics easier to understand and study compared to existing
approaches that rely on auxiliary \adhoc constructs and functions.
In particular,~\citep{CRS18} defines 
validation based on the existence
of an assignment of \SHACL shapes to data nodes. This assignment captures which 
shapes are satisfied/violated by which nodes, while at at same time the target nodes
of the validation process are verified. As~\citep{CRS18} argues, in the face of 
\SHACL recursion one may consider \emph{partial
assignments}, where the truth value of a constraint at some nodes may be left
unknown. In addition,~\citep{ACORSS20} identifies two major ways, called 
\emph{brave} and \emph{cautious} validation, to verify the target nodes 
during the validation process.
Deciding between partial or total assignments, and between brave or cautious 
validations gives rise to four different semantics for recursive \SHACL, each 
with its own definition of validation. 
Using our logical approach we are able to capture all four 
semantics in a clear and uniform way, providing for a better understanding of \SHACL
features and taking advantage of the rich field of computational logic.

Our contributions are the following:
\begin{itemize}
  \item
    
    In Theorem~\ref{thm:non-rec-semantics}, we prove that all four major semantics of \SHACL coincide for non-recursive 
    documents, and in Theorem~\ref{theorem:partialToTotal} that validation under the partial semantics (brave or cautious) reduces to validation under the corresponding total semantics, for all \SHACL documents. This reduction allows us later to focus only on total semantics, such that any positive decidability and complexity results for total carry over to partial semantics.
    (Section~\ref{sec:shacl}) 
    
  \item 
    We formalise non-recursive \SHACL semantics by translating to a novel fragment 
    of first-order logic (\FOL)  extended with counting quantifiers and a transitive closure
    operator; we call this logic \SCL for \emph{Shapes Constraint Logic}. The 
    provided translation from \SHACL to \SCL is actually an one-to-one 
    correspondence between these languages and we have identified eight prominent 
    \SHACL features that translate to particular restrictions of \SCL. In effect, 
    \SCL is the logical counterpart of \SHACL; this is exhibited by Theorem~\ref{centralSCL2SHACLCorrespondenceTheorem} which proves that faithfulness of an assignment in \SHACL, a central notion used to define all semantics, translates to satisfiability in \SCL. 
    (Section~\ref{sec:scl}) 
    
  \item 
    We extend \SCL into a fragment of monadic second-order logic, 
    called \MSCL, that intuitively allows us to impose conditions over 
    the space of all possible assignments and captures all four major recursive \SHACL 
    semantics. 
    We also present Proposition~\ref{proposition:ESCLtoSCL} which considers \ESCL, the existential fragment of \MSCL, expressive enough to capture several interesting problems; Proposition~\ref{proposition:ESCLtoSCL} states that \ESCL and \SCL are equisatisfiable, and we can only focus on \SCL when studying decidabilty and complexity.
    We also demonstrate how our logical framework generalises 
    previous languages designed to model \SHACL.  (Section~\ref{sec:scl})

   \item
    We present a series of results (Corollaries~\ref{cor:brave-total-val}-\ref{cor:cautious-total-cont} and Lemma~\ref{lemma:nonrecursive_containment_to_sat}) that reduce \SHACL satisfiability and containment 
    under all semantics to the \MSCL satisfiability problem. Going further, the problems of finite/unrestricted satisfiability and containment for non-recursive documents and the finite/unrestricted satisfiability for recursive \SHACL under brave semantics can be captured by \ESCL. We additionally present other decision problems from literature, such as shape and constraint satisfiability, and show how they are also captured by \ESCL. 
    (Section~\ref{sec:shacl2scl})
    
  \item
    We pay particular attention to \SHACL filters (\eg, constraints on the value
    of particular elementary datatypes), which have not been previously addressed
    in the literature, and provide a corresponding axiomatisation in \MSCL.
    (Section \ref{sec:fltaxm}).
  \item
    Finally, we turn our focus to \SCL (in effect, \ESCL) to explore the interaction of the main language features we have identified and create a detailed map of decidability and complexity results for many interesting fragments, for all aforementioned problems captured by \ESCL.
    In general, satisfiability and containment for the the full logic are undecidable.
    However, the base language has an \ExpTimeC satisfiability and containment problem. (Section~\ref{sec:sclsat}).

\end{itemize}



\section{Preliminaries}
\label{sec:prl}

With the term \emph{graph} we implicitly refer to a set of \emph{triples}, where
each single triple $\tuple {\sSym} {\pSym} {\oSym}$ identifies an edge with
label $\pSym$, called \emph{predicate}, from a node $\sSym$, called
\emph{subject}, to a node $\oSym$, called \emph{object}.
Graphs in this article are represented in \emph{Turtle syntax}~\citep{CP14}
using common XML namespaces, such as \texttt{sh} to refer to \SHACL terms.
Usually, in the \emph{RDF data model}~\citep{CWG14}, subjects, predicates, and
objects are defined over different but overlapping domains.
For example, while \emph{IRIs} can occupy any position in an RDF triple,
\emph{literals} (representing datatype values) can only appear in the object
position.
These differences are not central to the problem discussed in this article, and
thus, for the sake of simplicity, we will assume that all elements of a triple
are drawn from a single and infinite domain.
This assumption actually corresponds to what is known in the literature as
\emph{generalised RDF}~\citep{CWG14}.
We model triples as binary relations in \FOL, \ie, we write the
atom $\RRel(\sElm, \oElm)$ as a shorthand for the tuple $\tuple {\sSym} {\RSym}
{\oSym}$, and call $\RRel$ a graph relation name.
We use a minus sign to identify the \emph{inverse} role, \ie, we write
$\RRel[][-](\sElm, \oElm)$ in place of $\RRel(\oElm, \sElm)$.
We also consider the distinguished binary relation name $\isASym$ to represent
class membership triples, that is, we write $\tuple {\sSym} {\texttt{rdf:type}}
{\oSym}$ as $\isASym(\sElm, \oElm)$.




\section{Shapes Constraint Language: \SHACL}
\label{sec:shacl}

In this section we describe the \emph{Shapes Constraint Language} (\SHACL), a W3C language to
define formal constraints for the validation of \trdf graphs~\citep{KK17}.
Firstly, we introduce the main elements of its syntax, and explain the role they play in the
validation process.
We then discuss \emph{assignments}~\citep{CRS18}, that is, mappings that allow us to capture
which nodes in a graph
satisfy or violate which constraints. Assignments have been used to formally define
\SHACL semantics and this can non-ambiguously happen for the non-recursive case. For recursive
\SHACL, the specification leaves the semantics of recursive constraints open for interpretation,
and there have been more than one ways to extend the assignments-based semantics for this.
We review and discuss the four major \emph{extended} semantics that have been proposed in the
literature to handle recursive constraints.
Notably, in the absence of recursion, we show the collapse of all four
extended semantics into the same one.
We also show that two of these extended semantics can be considered a special case of the other
two, by proving a reduction from \emph{partial assignment} to \emph{total assignment} semantics
(defined later in this section).
Having formalised \SHACL semantics, we define the satisfiability and containment decision
problems for \SHACL documents.

\subsection{\SHACL Syntax}
\label{sec:shacl;sub:syn}

Data validation in \SHACL requires
two inputs:
\begin{inparaenum}[(1)]
  \item
    an \trdf graph $G$ to be validated and
  \item
    a \SHACL document $\vShapeDocument$ that defines the conditions against
    which $G$ must be validated.
\end{inparaenum}
The \SHACL specification defines the output of the data validation process as a
\emph{validation report}, detailing all the violations of the conditions set by
$\vShapeDocument$ that were found in $G$.
If the violation report contains no violations, a graph $G$ is \emph{valid} \wrt
a \SHACL document $\vShapeDocument$.
Determining whether a graph is valid \wrt a \SHACL document is the decision
problem called \emph{validation}.

A \SHACL \emph{document} is a set of \emph{shapes}.
Shapes essentially restrict the structure that a valid graph should have, by
defining a set of constraints that are evaluated against a set of nodes, known
as the \emph{target} nodes.
Formally, a shape is a tuple \trip{\vshape}{\vshapet}{\vshapec} defined by three
components:
\begin{inparaenum}[(1)]
  \item
    a shape name $\vshape$, which uniquely identifies the shape;
  \item
    a \emph{target definition} $\vshapet$ which is a set of \emph{target
    declarations}; each target declaration can be represented by a unary query
    and identifies the \trdf nodes that must satisfy the constraints $\vshapec$;
  \item
    a set of \emph{constraints} which are used in conjunction, and hence
    hereafter referred to as the single constraint $\vshapec$.
\end{inparaenum}
The \SHACL specification defines several types of constraints, called
\emph{constraint components}.
The \sh{datatype} component, for example, constraints an \trdf term to be an
\trdf literal of a particular datatype. Without loss of generality, we assume
that shape names in a \SHACL document do not occur in other \SHACL documents or
graphs.
As we formally define later, a graph is valid \wrt a document whenever all
constraints of all shapes in the document are satisfied by the target nodes of
the corresponding shapes.

It is worth noting that one type of \SHACL target declaration might reference
specific nodes to be validated that do not actually appear in the graph under
consideration.
Given a document \vShapeDocument\ and a graph $G$, we denote by
$\nodesof{G}{\vShapeDocument}$ the set of nodes in $G$ together with those
referenced by the node target declarations in \vShapeDocument.
In the absence of a document, we use $\nodesof{G}$ to denote the nodes of a
graph $G$.
With \shapesof{\vShapeDocument} we refer to all the shape names in a document
\vShapeDocument.
When it is clear from the context, we might use a shape name $\vshape$ either to
refer to the name itself or to the entire shape tuple.

Constraints can refer to other constraints by using the name of a shape as a short-hand to refer to its constraints.
We call this a \emph{shape reference}.
Let $S_{0}^{\vshapec}$ be the set of all the shape names occurring in a
constraint $\vshapec$ of a shape $\trip{\vshape}{\vshapet}{\vshapec}$; these
are the \emph{directly-referenced} shapes of $\vshape$.
Let $S_{i + 1}^{\vshapec}$ be the set of shapes in  $S_{i}^{\vshapec}$ union the
directly-referenced shapes of the constraints of the shapes in
$S_{i}^{\vshapec}$.
A shape $\trip{\vshape}{\vshapet}{\vshapec}$ is \emph{recursive} if $\vshape
\in S_{\infty}^{\vshapec}$.
A \SHACL document $\vShapeDocument$ is said to be \emph{recursive} if it contains a recursive shape,
and \emph{non-recursive} otherwise.
For simplicity, all \SHACL documents we consider in this work do not contain the
\sh{xone} constraint over shape references, which models the logical operator of
exclusive-or.
Any \SHACL document, in fact, can be linearly transformed into an equivalent
document that does not contain the \sh{xone} operator using a standard logical
transformation.
The intuition behind this transformation is
that an \sh{xone} defined over shapes $\vshape_1$ to $\vshape_n$ is equivalent to an
\sh{xone} between two shapes $\vshape_n$ and $\vshape_k$, where $\vshape_k$ is a fresh
shape whose constraint is the \sh{xone} of shapes $\vshape_1$ to $\vshape_{n-1}$. Then, any exclusive-or between two shapes can be linearly transformed into an
equivalent expression that uses only conjunctions, disjunctions, and the
negation operators.

\subsection{Semantics of Non-Recursive \SHACL}
\label{sec:shacl;sub:nrcsem}

A target declaration $\vshapet$ is a unary query over a graph $G$.
We denote with $G \models \vshapet(\vn)$ that a node $\vn$ is \emph{in the
target} of $\vshapet$ \wrt a graph $G$.
The target declaration $\vshapet$ might be empty, in which case no node is in
the target of $\vshapet$.
To formally discuss about nodes satisfying the constraints of a shape we need to
introduce the concept of \emph{assignments}~\citep{CRS18}.
Intuitively, an assignment is used to keep track, for any \trdf node, of all the
shapes whose constraints the node satisfies and all of those that it does not.

\begin{definition}
  Given a graph $G$, and a \SHACL document \vShapeDocument , an
  \emph{assignment} $\sigma$ for $G$ and $\vShapeDocument$ is a function mapping
  nodes in  \nodesof{G}{\vShapeDocument}, to subsets of shape literals in
  $\shapesof{\vShapeDocument} \cup  \{\neg \vshape | \vshape \in
  \shapesof{\vShapeDocument}\}$, such that for all nodes \vn\ and shape names
  \vshape, $\sigma(\vn)$ does not contain both \vshape\ and $\neg \vshape$.
\end{definition}

Notice that given a document and a graph, an assignment does not have to
associate all graph nodes to all document shapes or their negations.
In fact, there might exist node-shape pairs $(\vn, \vshape)$ for which neither
$\vshape \in \sigma(\vn)$ nor $\neg \vshape \in \sigma(\vn)$.
This is the reason why sometimes assignments are called \emph{partial
assignments}, as opposed to \emph{total assignments} which have to associate all
nodes with all shape names or their negation.

\begin{definition}
  An assignment $\sigma$ is \emph{total} \wrt a graph $G$ and a \SHACL document
  $\vShapeDocument$ if, for all nodes $n$ in $\nodesof{G}{\vShapeDocument}$ and
  shapes $\trip{\vshape}{\vshapet}{\vshapec}$ in $\vShapeDocument$, either
  $\vshape \in \sigma(n)$ or $\neg \vshape \in \sigma(n)$.
\end{definition}

For any graph $G$ and \SHACL document $\vShapeDocument$, we denote with
$\assignmentsPartial{G}{\vShapeDocument}$ and
$\assignmentsTotal{G}{\vShapeDocument}$, respectively, the set of assignments,
and the set of total assignments for $G$ and $\vShapeDocument$.
Trivially, $\assignmentsTotal{G}{\vShapeDocument} \subseteq
\assignmentsPartial{G}{\vShapeDocument}$ holds.

When trying to determine whether a node \vn\ of a graph $G$ satisfies a
constraint \vshapec\ of a shape, the outcome does not only depend on \vshapec,
\vn, and $G$, but it might also depend, due to shape references, on whether
other nodes satisfy the constraints of other shapes.
This latter fact can be encoded in an assignment $\sigma$. The
authors of~\citep{CRS18}, therefore, define the evaluation or \emph{conformance} of a node
$\vn$ to a constraint $\vshapec$ \wrt a graph $G$ under an assignment $\sigma$
as $\GsigmaModels{G}{\sigma}{\vshapec}{\vn}$.
This expression can take one of the three truth values of Kleene's logic:
\vtrue, \vfalse, or \vundefined.
If $\GsigmaModels{G}{\sigma}{\vshapec}{\vn}$ is \vtrue\ (\resp, \vfalse) we say
that node $n$ \emph{conforms} (\resp, does not conform) to constraint $\vshapec$
\wrt $G$ under $\sigma$.
 Note that if the set of constraints of a shape is empty,
then every node trivially conforms to it, that is, for all
nodes $\vn$, graphs $G$ and assignments $\sigma$, it holds that $\GsigmaModels{G}{\sigma}{\emptyset}{\vn}$ is $\vtrue$.

Intuitively, the evaluation of $\GsigmaModels{G}{\sigma}{\vshapec}{\vn}$ can be
split into two parts:
the first verifies conditions on $G$, such as the existence of certain triples.
The second part examines other node-shape pairs that $d$ itself is listing for
conformance and, instead of triggering subsequent evaluation, checks whether
their conformance is correctly encoded in $\sigma$.
Since -- in general and for arbitrary \SHACL documents that might be recursive
-- $\sigma$ is partial, it might be that
$\GsigmaModels{G}{\sigma}{\vshapec}{\vn}$ is \vundefined.
 Table \ref{tab:constrDefExamples} provides examples of how
$\GsigmaModels{G}{\sigma}{\vshapec}{\vn}$ is defined for certain salient
constraints. For a comprehensive definition of how all \SHACL constraints
are evaluated we refer the reader to~\citep{CRS18}.\footnote{Some \SHACL constraints are defined in the appendix
of the extended version of~\citep{CRS18}.}

\begin{table}[t]
\begin{center}
\begin{tabular}{|p{4.15cm}|p{4.15cm}|p{3.9cm}|}
 \hline
   Description of constraint $\vshapec$ & \SHACL\ triples for applying $d$ to shape \genericiri{s} & Definition of $\GsigmaModels{G}{\sigma}{\vshapec}{\vn}$ \\ \hline \hline

Empty constraint ($\vshapec = \emptyset$) & $\{ \}$ & \vtrue \\  \hline 

Test whether node is an IRI  & $\{\genericiri{s}\ \sh{nodeKind} \; \sh{isIRI} \}$
 &  \vtrue\ if $n$ is an IRI, 
 
 or else \vfalse  \\  \hline

Conformance to shape \genericiri{s1}  & $\{\genericiri{s}\ \sh{node}\ \genericiri{s1} \}$
 &  \vtrue\ if $\genericiri{s1}\in \sigma(n)$,  
 
 \vfalse\ if $\neg \genericiri{s1}\in \sigma(n)$,  
 
 or else \vundefined  \\  \hline

 Existence of an \genericiri{r}-successor & $\{\genericiri{s} \phantom{\ }\sh{path} \genericiri{r},$ $\genericiri{s}\ \sh{minCount} \; 1 \}$
 & \vtrue\ if $G$ contains a triple  with $n$ as the subject and  \genericiri{r} as the predicate,  
 
 or else \vfalse \\  \hline

Conjunction of constraints  $\vshapec'$ and $\vshapec''$ & Union of \SHACL triples for applying $d'$ and $d''$ to \genericiri{s} & $\GsigmaModels{G}{\sigma}{\vshapec'}{\vn} \wedge \GsigmaModels{G}{\sigma}{\vshapec''}{\vn}$ \\  \hline

\end{tabular}
\caption{\label{tab:constrDefExamples} Example definitions of $\GsigmaModels{G}{\sigma}{\vshapec}{\vn}$ for selected \SHACL constraints $d$, given a node $n$, a graph $G$ and an assignment $\sigma$, where \genericiri{s} and \genericiri{s1} are shape names, $\vshapec'$ and $\vshapec''$ are \SHACL constraints, and \genericiri{r} is an IRI. The second column shows which triples, in a \SHACL shape graph, define that \genericiri{s} has constraint $\vshapec$ .}
\end{center}
\end{table}

It should be noted that the evaluation of certain constraints to the truth value of \vundefined\ might not affect the outcome of the graph validation process
(see Section \ref{sec:shacl;sub:fulsem} for an example).
Graph validation depends on the existence of an assignment such that even if it
is \vundefined\ for certain nodes, at least is consistent (as defined below)
and is \vtrue\ for all target nodes on the constraints of the shapes that
describe these nodes as targets. Such assignments are known as \emph{faithful}
assignments~\citep{CRS18}. Note that, as we show in Lemma~\ref{propertyOfNonRecursiveDocs}, for
non-recursive documents there is a unique faithful assignment which is total and
for which \vundefined\ conformance never appears.

\begin{definition}
  \label{def:faithfulAssignment}
  For all graphs $G$ and \SHACL documents $\vShapeDocument$, an assignment
  $\sigma$ is \emph{faithful} \wrt $G$ and $\vShapeDocument$, denoted by
  \isfaithful{G}{\sigma}{\vShapeDocument}, if the following two conditions
  hold true for any shape $\trip{\vshape}{\vshapet}{\vshapec}$ in
  $\shapesof{\vShapeDocument}$ and node $\vn$ in \nodesof{G}{\vShapeDocument}:
  \begin{enumerate}[(1)]
    \item
      $\vshape \in \sigma(n)$ \iff $\GsigmaModels{G}{\sigma}{\vshapec}{\vn}$ is
      \vtrue\ and $\neg \vshape \in \sigma(n)$ \iff
      $\GsigmaModels{G}{\sigma}{\vshapec}{\vn}$ is \vfalse;
    \item
      if $G \models \vshapet(\vn)$ then $\vshape \in \sigma(n)$.
  \end{enumerate}
\end{definition}

\begin{figure*}[t]%
  \noindent\begin{minipage}[t]{0.39\linewidth}%
    \begin{lstlisting}
:studentShape a sh:NodeShape ;
  sh:targetClass :Student ;
  sh:not :disjFacultyShape .
:disjFacultyShape a
    sh:PropertyShape ;
  sh:path (:hasSupervisor
    :hasFaculty);
  sh:disjoint :hasFaculty .
    \end{lstlisting}%
  \end{minipage} \hfill %
  \noindent\begin{minipage}[t]{0.30\linewidth}%
    \begin{lstlisting}
:Alex a :Student ;
  :hasFaculty :CS ;
  :hasSupervisor :Jane .
:Jane :hasFaculty :CS .
    \end{lstlisting}%
  \end{minipage}%
  \noindent\begin{minipage}[t]{0.30\linewidth}%
    \begin{footnotesize}%
    {
    \[
    \begin{aligned}
    &\sigma(\texttt{:Alex}) = \\ & \{  \texttt{:studentShape}, \\ & \;\; \neg
    \texttt{:disjFacultyShape}\},\\
    &\sigma(\texttt{:Jane}) = \sigma(\texttt{:CS}) = \\ & \;
    \sigma(\texttt{:Student}) = \\
    & \{  \neg \texttt{:studentShape}, \\ &   \;\; \texttt{:disjFacultyShape}\}.
    \end{aligned}%
    \]}%
    \end{footnotesize}%
  \end{minipage}%
  \captionof{figure}{\label{fig:MainExample} A \SHACL document (left), a graph
    that validates it (centre), and a faithful assignment for this graph and
    document (right).}
\end{figure*}

Intuitively, condition (1) ensures that the evaluation described by the
assignment is indeed correct; while condition (2) ensures that the assignment agrees with the target definitions.
The existence of a faithful assignment is a necessary and sufficient
condition for validation of non-recursive \SHACL documents~\citep{CRS18}.

\begin{definition}
  \label{validationDefNonRecursive}
  A graph $G$ is valid \wrt a non-recursive \SHACL document $\vShapeDocument$ if
  there exists an assignment $\sigma$ such that
  $\isfaithful{G}{\sigma}{\vShapeDocument}$.
\end{definition}

An example \SHACL document is shown in Figure~\ref{fig:MainExample}.
This example captures the requirement that all students must have at least one
supervisor from the same faculty.
The shape with name \texttt{:studentShape} has class \texttt{:Student} as a
target, meaning that all members of this class must satisfy the constraint of
the shape.
The constraint definition of \texttt{:studentShape} requires the
non-satisfaction of shape \texttt{:disjFacultyShape}, \ie, a node satisfies
\texttt{:studentShape} if it does not satisfy \texttt{:disjFacultyShape}.
The \texttt{:disjFacultyShape} shape states that an entity has no faculty in
common with any of their supervisors. 
 This is expressed using the \sh{path} term, which defines a property chain (\ie, a composition of roles \texttt{:hasSupervisor}
and \texttt{:hasFaculty}), and the \sh{disjoint} term, which defines a constraint over this property chain (\ie, the non-existence
of a node reachable both by this property chain, and directly by the \texttt{:hasFaculty} role).
The \sh{path} term is used to construct constraints over property chains, but it does not,
on its own, impose their existence. 
A graph that is valid with respects to these shapes is provided in
Figure~\ref{fig:MainExample}, along with a faithful assignment for this graph.
The graph can be made invalid by changing the faculty of \texttt{:Jane} in the
last triple to a different value.

As we will see later, the existence of a faithful assignment is also a necessary
condition for all other semantics that allow recursion.
For those cases, however, we will want to consider additional assignments where
the first property of Definition~\ref{def:faithfulAssignment} holds, but not
necessarily the second, \ie, assignments that agree with the constraint definitions,
but not necessarily the target definitions of the shapes.
In order to do this, we will remove the targets from a document and look for
faithful assignments against the new document, since condition (2) of
Definition~\ref{def:faithfulAssignment} is trivially satisfied for \SHACL
documents where all target definitions are empty.
Let $\minustargets{\vShapeDocument}$ denote the \SHACL document obtained by
substituting all target definitions in \SHACL document $\vShapeDocument$ with
the empty set.
Then, the following lemma is immediate:

\begin{lemma}
  \label{lemma:supportedModel}
  For all graphs $G$, \SHACL documents $\vShapeDocument$ and assignments
  $\sigma$, condition (1) from Definition \ref{def:faithfulAssignment} holds for
  any shape $\vshape$ in $\shapesof{\vShapeDocument}$ and node $\vn$ in
  \nodesof{G}{\vShapeDocument} \iff
  $\isfaithful{G}{\sigma}{\minustargets{\vShapeDocument}}$.
\end{lemma}

For non-recursive \SHACL documents, the next lemma states that for any graph,
there exists a unique faithful total assignment for
\minustargets{\vShapeDocument} and, if there is a faithful assignment for
\vShapeDocument, then this must be it.

\begin{lemma}
  \label{propertyOfNonRecursiveDocs}
  For all graphs $G$ and non-recursive \SHACL documents $\vShapeDocument$, there
  exists a unique assignment $\rho$ in $\assignmentsTotal{G}{\vShapeDocument}$
  such that $\isfaithful{G}{\rho}{\minustargets{\vShapeDocument}}$, and for
  every assignment $\sigma$ in $\assignmentsPartial{G}{\vShapeDocument}$ such
  that $\isfaithful{G}{\sigma}{\vShapeDocument}$, then $\sigma = \rho$.
\end{lemma}
\begin{proof}
If $\vShapeDocument$ is non recursive, then there exists a non empty subset
$\vShapeDocument^{\prime}$ of $\vShapeDocument$ that only contains shapes whose
constraints do not use shape references.
Intuitively, the constraints of the shapes in $\vShapeDocument^{\prime}$ can be
evaluated directly on any graph, independently of any assignment.
Shape references are the only part of the evaluation of a constraint that depends on
the assignment $\sigma$, and that could introduce the truth value \vundefined\
under three-valued logic~\citep{CRS18}.
Thus, for all graphs $G$, assignments $\sigma$, nodes $\vn$ and constraints $c$
in of a shape in $\vShapeDocument^{\prime}$, it holds that the evaluation of
$\GsigmaModels{G}{\sigma}{\vshapec}{\vn}$ (1) does not depend on $\sigma$ and
(2) has a Boolean truth value.
It is easy to see that properties (1) and (2) also hold for the document
$\vShapeDocument^{\prime\prime}$ which contains the shapes of $\vShapeDocument$
whose shape references (if any) only reference shapes in
$\vShapeDocument^{\prime}$.
This reasoning can be extended inductively to prove that properties (1) and (2)
hold for all the shapes of $\vShapeDocument$.
Point (1) ensures that there cannot be more than one assignment such that
$\isfaithful{G}{\sigma}{\minustargets{\vShapeDocument}}$, while point (2)
ensures that such an assignment is total.
This assignment $\sigma$ exists and it can be computed iteratively as follows.
Let $\sigma^{\prime}$ be the assignment for $\vShapeDocument^{\prime}$ such that
for any shape $\trip{\vshape}{\vshapet}{\vshapec}$ in $\vShapeDocument^{\prime}$
and node $\vn$, $\vshape \in \sigma^{\prime}(\vn)$, if
$\GsigmaModels{G}{\emptyset}{\vshapec}{\vn}$ is \vtrue, and $\neg \vshape \in
\sigma^{\prime}(\vn)$, otherwise.
Then let $\sigma^{\prime\prime}$ be the assignment for
$\vShapeDocument^{\prime\prime}$ such that for any shape
$\trip{\vshape}{\vshapet}{\vshapec}$ in $\vShapeDocument^{\prime\prime}$ and
node $\vn$, $\vshape \in \sigma^{\prime\prime}(\vn)$, if
$\GsigmaModels{G}{\sigma^{\prime}}{\vshapec}{\vn}$ is \vtrue, and $\neg \vshape
\in \sigma^{\prime\prime}(\vn)$, otherwise.
This process is repeated until the assignment $\sigma$, defined over all of the
shapes of $\vShapeDocument$, is computed.
Notice that for all graphs $G$, \SHACL documents $\vShapeDocument$ and
assignments $\sigma$, fact $\isfaithful{G}{\sigma}{\vShapeDocument}$ implies
$\isfaithful{G}{\sigma}{\minustargets{\vShapeDocument}}$. Thus the existence of an assignment $\rho'$
different than $\rho$ such that $\isfaithful{G}{\rho'}{\vShapeDocument}$,
is in contradiction with the fact that there cannot
be more than one assignment that is faithful for $G$ and $\minustargets{\vShapeDocument}$. \end{proof}

\subsection{Semantics of Full \SHACL}
\label{sec:shacl;sub:fulsem}

As mentioned, the semantics of recursive shape definitions in \SHACL documents has been left
undefined in the original W3C \SHACL specification~\citep{KK17} and this gives
rise to several possible interpretations.
In this work, we consider, and extend upon, previously introduced semantics of
\SHACL that define how to interpret recursive \SHACL documents.
These can be characterised by two dimensions, namely the choice between
\begin{inparaenum}[(1)]
  \item
    \emph{partial} and \emph{total} assignments~\citep{CRS18} and
  \item
    between \emph{brave} and \emph{cautious} validation~\citep{ACORSS20},
\end{inparaenum}
which we will subsequently formally introduce.
Together, these two dimensions result in the four extended semantics studied in
this article, namely \emph{brave-partial}, \emph{brave-total},
\emph{cautious-partial} and \emph{cautious-total}.

Notice that the formulation of the brave and cautious notions originates in the literature of
non-monotonic reasoning and logic programming (see, \eg,~\cite{McD82} and~\cite{EG95}, respectively).
We do not consider the less obvious dimension of \emph{stable-model}
semantics~\cite{GL88}, which also relates to non-monotonic reasoning in logic
programming~\cite{RT88,Sak89,GRS91} and inductive learning~\cite{SI09}.
Our definitions of partial assignments, total assignments, and brave validation
exactly correspond to existing definitions of \cite{CRS18}. For cautious validation, instead, we adopt a more general definition than the one previously considered in SHACL literature \citep{ACORSS20}, where it was only studied under stable-model
semantics. 


The first extended semantics that we consider coincides with
Definition~\ref{validationDefNonRecursive}.
That is, the existence of a faithful assignment can be directly used as a
semantics for recursive documents as well.
Nevertheless, in this case the assignment is not necessarily total, as is in the case of non-recursive documents proven in Lemma~\ref{propertyOfNonRecursiveDocs}. To
stress this (as well as the ``brave'' nature of the semantics discussed later),
we call this the brave-partial semantics.

\begin{definition}
  \label{validationDefPartialBrave}
    A graph $G$ is valid \wrt a \SHACL document $\vShapeDocument$ under
    \emph{brave-partial semantics} if there exists an assignment $\sigma \in
    \assignmentsPartial{G}{\vShapeDocument}$ such that
    $\isfaithful{G}{\sigma}{\vShapeDocument}$.
\end{definition}

The other three extended semantics are defined by adding further conditions to
the one just introduced.
To motivate those, first consider an example of a recursive document and of a
non-total faithful assignment that evaluates the conformance of some nodes against
some constraints to \vundefined.
This happens when recursion makes it impossible for a node $\vn$ to either
conform or not to conform to a shape $\vshape$ but, at the same time, validity
does not depend on whether $\vn$ conforms to shape $\vshape$ or not.
Consider, for instance, the following \SHACL document, containing the single
shape \trip{\vshape^{*}}{\emptyset}{\vshapec^{*}} defined as follows:
\begin{lstlisting}
:InconsistentS a sh:NodeShape ;
  sh:not :InconsistentS .
\end{lstlisting}
This shape is defined as the negation of itself, that is, given a node $\vn$, a
graph $G$ and an assignment $\sigma$, fact
$\GsigmaModels{G}{\sigma}{\vshapec^{*}}{\vn}$ is \vtrue\ \iff $\neg \vshape^{*}
\in \sigma(\vn)$, and \vfalse\ \iff $ \vshape^{*} \in \sigma(\vn)$.
It is easy to see that any assignment that maps a node to either
$\{\vshape^{*}\}$ or $\{\neg \vshape^{*}\}$ is not faithful, as it would violate
condition (1) of Definition~\ref{def:faithfulAssignment}.
However, an assignment that maps every node of a graph to the empty set would be
faithful for that graph and document $\{\vshape^{*}\}$.
Intuitively, this means that nodes in the graph cannot conform nor not conform
to shape $\vshape^{*}$, but this should not be interpreted as a violation of any
constraint, since this shape does not have any target node to validate.
In effect, conformance for all nodes to the constraint of $\{\vshape^{*}\}$ is
left as \vundefined, but the existence of a faithful assignment makes any graph
valid \wrt to $\{\vshape^{*}\}$.

In the W3C \SHACL specification, where recursion semantics was left open to
interpretation, nodes can either conform to, or not conform to a given shape,
and the concept of an ``undefined'' level of conformance is arguably alien to
the specification.
It is natural, therefore, to consider restricting the evaluation of a constraint
to the \vtrue\ and \vfalse\ values of Boolean logic.
This is achieved by restricting assignments to be total.

\begin{definition}
  \label{validationDefTotalBrave}
  A graph $G$ is valid \wrt a \SHACL document $\vShapeDocument$ under
  \emph{brave-total semantics} if there exists a total assignment $\sigma \in
  \assignmentsTotal{G}{\vShapeDocument}$ such that
  $\isfaithful{G}{\sigma}{\vShapeDocument}$.
\end{definition}

Since total assignments are a more specific type of assignments, if a graph $G$
is valid \wrt a \SHACL document $\vShapeDocument$ under brave-total semantics,
than it is also valid \wrt $\vShapeDocument$ under brave-partial semantics.
The converse, instead, is only true for non-recursive \SHACL documents.
In fact, as we show later on, all extended semantics coincide, for non-recursive
\SHACL documents.
Note also, that there is no obvious preferable choice for the semantics of recursive
documents.
For example, while total assignments can be seen as a more natural way of
interpreting the \SHACL specification, they are not without issues of their own.
Going back to our previous example, we can notice that there cannot exist a
total faithful assignment for the \SHACL document containing shape
\uri{InconsistentS}, for any non-empty graph.
This is a trivial consequence of the fact that no node can conform to, nor not
conform to, shape \uri{InconsistentS}.
In this example, however, brave-total semantics conflicts with the \SHACL specification, since the latter
implies that a \SHACL document without target declarations in any of its shapes (such as
the one in our example) should trivially validate any graph.
If there are no target declarations, in fact, there are no target nodes on which
to verify the conformance of certain shapes, and thus no violations of
constraints should be detected.

Another dimension in the choices for extended semantics studied in
literature~\citep{ACORSS20} is the difference between brave and cautious
validation of recursive documents.
When a \SHACL document $\vShapeDocument$ is recursive, there might exist
multiple assignments satisfying property (1) of
Definition~\ref{def:faithfulAssignment}, that is, multiple $\sigma$ for which
$\isfaithful{G}{\sigma}{\minustargets{\vShapeDocument}}$.
Intuitively, these can be seen as equally ``correct'' assignments with respect
to the constraints of the shapes, and brave validation only checks whether at
least one of them is compatible with the target definitions of the shapes.
Cautious validation, instead, represents a stronger form of validation, where
all such assignments must be compatible with the target definitions.

\begin{definition}
  \label{ref:cautiousValidation}
  A graph $G$ is valid \wrt a \SHACL document $\vShapeDocument$ under
  \emph{cautious-partial} (\resp, \emph{cautious-total}) semantics if it is
  \begin{inparaenum}[(1)]
    \item
      valid under \emph{brave-partial} (\resp, \emph{brave-total}) semantics and
    \item
      for all assignments $\sigma$ in $\assignmentsPartial{G}{\vShapeDocument}$
      (\resp, $\assignmentsTotal{G}{\vShapeDocument}$), it is true that if
      $\isfaithful{G}{\sigma}{\minustargets{\vShapeDocument}}$ holds then
      $\isfaithful{G}{\sigma}{\vShapeDocument}$ holds as well.
  \end{inparaenum}
\end{definition}

To exemplify this distinction, consider the following \SHACL document
$\vShapeDocument_1$.

\begin{lstlisting}
:VegDishShape a sh:PropertyShape ;
  sh:targetNode :DailySpecial ;
  sh:path :hasIngredient ;
  sh:minCount 1 ;
  sh:qualifiedMaxCount 0 ;
  sh:qualifiedValueShape [ sh: not :VegIngredientShape ] .

:VegIngredientShape a sh:PropertyShape ;
  sh:path [ sh:inversePath :hasIngredient ] ;
  sh:node :VegDishShape .
\end{lstlisting}
This document requires the daily special of a restaurant, node
\uri{DailySpecial}, to be vegetarian, that is, to conform to shape
\uri{VegDishShape}.
This shape is recursively defined as follows.
Something is a vegetarian dish if it contains an ingredient, and all of its
ingredients are vegetarian, that is, entities conforming to the
\uri{VegIngredientShape}.
A vegetarian ingredient, in turn, is an ingredient of at least one vegetarian
dish. Consider now a graph $G_1$ containing only the following triple.
\begin{lstlisting}
:DailySpecial :hasIngredient :Chicken .
\end{lstlisting}
Due to the recursive definition of \uri{VegDishShape}, there exist two different
assignments $\sigma_1$ and $\sigma_2$, which are both faithful for $G_1$ and
$\minustargets{\vShapeDocument_1}$.
In $\sigma_1$, no node in $G_1$ conforms to any shape, while $\sigma_2$ differs
from $\sigma_1$ in that node \uri{DailySpecial} conforms to \uri{VegDishShape}
and node \uri{Chicken} conforms to \uri{VegIngredientShape}.
Essentially, either both the dish and the ingredient from graph $G_1$ are
vegetarian, or neither is.
Therefore, $\sigma_2$ is faithful for $G_1$ and $\vShapeDocument_1$, while
$\sigma_1$ is not.
The question of whether the daily special is a vegetarian dish or not can be
approached with different levels of ``caution''.
Under brave validation, graph $G_1$ is valid \wrt $\vShapeDocument_1$, since it
is possible that the daily special is vegetarian.
Cautious validation, instead, takes the more conservative approach, and under
its definition $G_1$ is not valid \wrt by $\vShapeDocument_1$, since it is also
possible that the daily special is not vegetarian.

\begin{table}[t]
\begin{center}
\begin{tabular}{ |l | l | l |}
 \hline
  & Brave & Cautious \\ \hline

 Partial
 & $\exists \, \sigma . \; \isfaithful{G}{\sigma}{\vShapeDocument}$
 &  \begin{tabular}{@{}l@{}}$\exists \, \sigma . \;
\isfaithful{G}{\sigma}{\vShapeDocument}$ and
 \\ $\forall \, \sigma . $ if $
\isfaithful{G}{\sigma}{\minustargets{\vShapeDocument}}$,
 \\ $\phantom{\forall \, \sigma .} $ then $
\isfaithful{G}{\sigma}{\vShapeDocument}$ \end{tabular} \\ \hline

 Total
 & $\exists \, \rho . \; \isfaithful{G}{\rho}{\vShapeDocument}$
 & \begin{tabular}{@{}l@{}}$\exists \, \rho . \;
\isfaithful{G}{\rho}{\vShapeDocument}$ and
 \\ $\forall \, \rho . $ if $
\isfaithful{G}{\rho}{\minustargets{\vShapeDocument}}$,
 \\ $\phantom{\forall \, \rho .} $ then $ \isfaithful{G}{\rho}{\vShapeDocument}$
\end{tabular} \\ \hline
\end{tabular}
\caption{\label{tab:recursiveSemanticsTable} Definition of validity (from
Definitions~\ref{validationDefPartialBrave}, \ref{validationDefTotalBrave}
and~\ref{ref:cautiousValidation}) of a graph $G$ under a \SHACL document
$\vShapeDocument$ ($G \models \vShapeDocument$) \wrt the two dimensions of
extended semantics considered in this article, where $\sigma \in
\assignmentsPartial{G}{\vShapeDocument}$ and $\rho \in
\assignmentsTotal{G}{\vShapeDocument}$.}
\end{center}
\end{table}

For each extended semantics, the definition of validity of a graph $G$ \wrt a
\SHACL document $\vShapeDocument$, denoted by $G \models \vShapeDocument$, is
summarised in the following list, and schematised in
Table~\ref{tab:recursiveSemanticsTable}.
\begin{description}
  \item[\emph{brave-partial}]
    there is an assignment that is faithful \wrt $G$ and $\vShapeDocument$;
  \item[\emph{brave-total}]
    there is an assignment that is total and faithful \wrt $G$ and
    $\vShapeDocument$;
  \item[\emph{cautious-partial}]
    there is an assignment that is faithful \wrt $G$ and $\vShapeDocument$, and
    every assignment that is faithful \wrt $G$ and
    $\minustargets{\vShapeDocument}$ is also faithful \wrt $G$ and
    $\vShapeDocument$.
  \item[\emph{cautious-total}]
    there is an assignment that is total and faithful \wrt $G$ and
    $\vShapeDocument$, and every assignment that is total and faithful \wrt $G$
    and $\minustargets{\vShapeDocument}$ is also faithful \wrt $G$ and
    $\vShapeDocument$.
\end{description}

We now prove that, when considering only non-recursive \SHACL documents, these
four semantics are necessarily equivalent to each other, since the semantics of
non-recursive \SHACL documents is uniquely determined.
The formalisation of this equivalence given in the next theorem is essentially a
consequence of Lemma~\ref{propertyOfNonRecursiveDocs}.

\begin{theorem}
\label{thm:non-rec-semantics}
  For any graph $G$, non-recursive \SHACL document $\vShapeDocument$, and
  extended semantics $\alpha$ and $\beta$, it holds that $G \!\models\!
  \vShapeDocument\!$ under $\alpha$ \iff $G \!\models\! \vShapeDocument\!$ under
  $\beta$.
\end{theorem}
\begin{proof}
  Since $\assignmentsTotal{G}{\vShapeDocument} \subseteq
  \assignmentsPartial{G}{\vShapeDocument}$, for any graph $G$ and \SHACL
  document $\vShapeDocument$, the definition of validity of
  \emph{cautious-total} trivially subsumes the one of \emph{brave-total} and
  \emph{cautious-partial} which, in turn, subsumes the one of
  \emph{brave-partial}.
  Notice that for all graphs $G$, \SHACL documents $\vShapeDocument$ and
  assignments $\sigma$, if $\sigma \in
  \assignmentsTotal{G}{\minustargets{\vShapeDocument}}$, then $\sigma \in
  \assignmentsTotal{G}{\vShapeDocument}$. From Lemma~\ref{propertyOfNonRecursiveDocs} we
  also know that a faithful assignment for $\vShapeDocument$ and $G$ is necessarily total, and it is the
  same unique assignment that is faithful for $\minustargets{\vShapeDocument}$ and $G$.
  Thus, for non-recursive documents, the definition of validity of
  \emph{brave-partial} subsumes the one of \emph{cautious-total},
  and consequently the four extended semantics are equivalent.
\end{proof}

An expert reader might observe that the above theorem resembles a similar result in the literature of logic programming for query
answering under stratified programs~\cite{Prz89}, where the existence of a
unique perfect model forces the collapse of the two notions of brave and
cautious answers.

Given any notion of validity from Table \ref{tab:recursiveSemanticsTable},
corresponding to one of the four extended semantics, we can define the following
decision problems, which we study in detail in the remaining part of the
article.
\begin{itemize}
  \item
    \textbf{SHACL Satisfiability}:
    A \SHACL document $\vShapeDocument$ is satisfiable \iff there exists a graph
    $G$ such that $G \models \vShapeDocument$.
  \item
    \textbf{SHACL Containment}:
    For all \SHACL documents $\vShapeDocument_1$, $\vShapeDocument_2$, we say
    that $\vShapeDocument_1$ is contained in $\vShapeDocument_2$, denoted
    $\vShapeDocument_1 \subseteq \vShapeDocument_2$, \iff for all graphs $G$, if
    $G \models \vShapeDocument_1$ then $G \models \vShapeDocument_2$.
\end{itemize}

Obviously, the more meaningful satisfiability problem is one on finite graphs.
\begin{itemize}
  \item
    \textbf{SHACL Finite Model Property}:
    A \SHACL document $\vShapeDocument$ enjoys the finite model property if
    whenever it is satisfiable it is so on a finite graph.
\end{itemize}

\subsection{From Partial to Total Assignments}
\label{sub:partot}

In the remainder of the paper we simplify our study of \SHACL\ by only considering recursive semantics based on total assignments.
We focus on this type of assignments because,
as we see later, partial assignment semantics can be seen as a special case
of total.
By showing positive results for extended semantics based on total assignments, we are therefore also
showing the same results for the corresponding semantics based on partial assignments.
It should be noted, however, that this does not hold for negative results.
This means that the decidability results that we show in Section \ref{sec:sclsat} apply
to both total and partial assignments, but undecidability, instead, does not carry on to partial assignments; this remains an open question.

We prove a reduction from partial to total assignments by showing that any \SHACL document $\vShapeDocument$ can be linearly transformed into
another document $\vShapeDocument^{*}$ such that a graph $G$ is valid \wrt
$\vShapeDocument$ under brave-partial, or cautious-partial, \iff $G$ is valid
\wrt $\vShapeDocument^{*}$ under brave-total or cautious-total, respectively.
Intuitively, this is achieved by splitting each shape $\vshape$ into two shapes
$\vshape^{+}$ and $\vshape^{-}$, evaluated under total assignments semantics,
such that the constraints of $\vshape^{+}$ and $\vshape^{-}$ model the
evaluation to \vtrue\ and \vfalse, respectively, of the constraints of
$\vshape$, and such that the evaluation to \vundefined\ of the constraints of
$\vshape$ correspond to the negation of the constraints of both $\vshape^{+}$
and $\vshape^{-}$.

Note that the aforementioned reduction has strong similarities with the notion of completion for programs with stratified negation in logic
programming~\cite{Min88} (see, also~\cite{MS92} and \cite{SSI20}).

In the following, we formalise the just discussed transformation by means of a
function $\Gamma$ over \SHACL documents.
With a slight abuse of notation, we use $\neg$ and $\wedge$ to denote,
respectively, the negated form of a \SHACL constraint, and the conjunction of
two \SHACL constraints.
We also denote $\vshape(x)$ the constraint requiring node $x$ to conform to shape $\vshape$.
We use $\vshape^{+}$ and $\vshape^{-}$ to denote two unique fresh shape names,
which are a function of $\vshape$.

\begin{definition}
  \label{gammatransformation}
  Given a \SHACL document $\vShapeDocument$, document $\Gamma(\vShapeDocument)$
  contains shapes \trip{\vshape^{+}}{\vshapet}{\gamma(\vshapec)} and
  \trip{\vshape^{-}}{\vshapet}{\gamma(\neg \vshapec)} for every shape
  \trip{\vshape}{\vshapet}{\vshapec} in $\vShapeDocument$, such that, for every
  constraint $\vshapec$, the corresponding constraint $\gamma(\vshapec)$ is
  constructed by replacing, for every shape $\vshape$, every occurrence of the negated atom ``$\neg
  \vshape(x)$'' in $\vshapec$ with ``$\neg \vshape^{+}(x) \wedge
  \vshape^{-}(x)$'' and every occurrence of the non-negated atom
  ``$\vshape(x)$'' in $\vshapec$ with ``$\vshape^{+}(x) \wedge \neg
  \vshape^{-}(x)$''.
\end{definition}

\begin{definition}
  Given an assignment $\sigma$, let $\sigma^{\gamma}$ be the assignment such
  that for every node $\vn$ the following holds: $\sigma^{\gamma}(\vn) = \{
  \vshape^{+}, \neg \vshape^{-} | \vshape \in \sigma(\vn) \} \cup \{ \neg
  \vshape^{+}, \vshape^{-} | \neg \vshape \in \sigma(\vn) \} \cup \{ \neg
  \vshape^{+}, \neg \vshape^{-} | \vshape , \neg \vshape \not \in \sigma(\vn)
  \}$.
\end{definition}

We can observe that for any \SHACL document $\vShapeDocument$, graph $G$ and
assignment $\sigma$ for $\vShapeDocument$ and $G$, assignment $\sigma^{\gamma}$
is a total assignment for $\Gamma(\vShapeDocument)$ and $G$.
Also, it is easy to see that the complexity of the transformation
$\Gamma(\vShapeDocument)$ is linear in the size of the original document
$\vShapeDocument$.

\begin{lemma}
  \label{gammatransformationProperties}
  Given a \SHACL document $\vShapeDocument$, a graph $G$, an assignment
  $\sigma$, and a node $\vn$, the following hold:
  \begin{itemize}
    \item
      $\GsigmaModelsThree{G}{\sigma}{\vshapec}{\vn}$ is \vtrue\ \iff
      $\GsigmaModels{G}{\sigma^{\gamma}}{\gamma(\vshapec)}{\vn}$ is \vtrue;
    \item
      $\GsigmaModelsThree{G}{\sigma}{\vshapec}{\vn}$ is \vfalse\ \iff
      $\GsigmaModels{G}{\sigma^{\gamma}}{\gamma(\neg \vshapec)}{\vn}$ is \vtrue;
    \item
      $\GsigmaModelsThree{G}{\sigma}{\vshapec}{\vn}$ is \vundefined\ \iff both
      $\GsigmaModels{G}{\sigma^{\gamma}}{\gamma(\vshapec)}{\vn}$ and
      $\GsigmaModels{G}{\sigma^{\gamma}}{\gamma(\neg \vshapec)}{\vn}$ are
      \vfalse.
  \end{itemize}
\end{lemma}
\begin{proof}
  Negation in \SHACL is defined in the standard way, and therefore
  $\GsigmaModelsThree{G}{\sigma}{\vshapec}{\vn}$ is \vtrue\ \iff
  $\GsigmaModelsThree{G}{\sigma}{\neg \vshapec}{\vn}$ is \vfalse.
  Since $\GsigmaModelsThree{G}{\sigma}{\vshapec}{\vn}$ is \vfalse\ \iff
  $\GsigmaModelsThree{G}{\sigma}{\neg \vshapec}{\vn}$ is \vtrue , proof of the
  first statement of the lemma is also proof of the second.
  We can also notice that the third statement of the lemma necessarily follows
  from the first two.
  Thus the entire lemma can be proved by proving just the first statement.
  To prove the first item, we show the following two implications, separately:
  \begin{itemize}
    \item[$(\Rightarrow)$:]
      if $\GsigmaModelsThree{G}{\sigma}{\vshapec}{\vn}$ is \vtrue, then
      $\GsigmaModels{G}{\sigma^{\gamma}}{\gamma(\vshapec)}{\vn}$ is \vtrue;
    \item[$(\Leftarrow)$:]
      if $\GsigmaModels{G}{\sigma^{\gamma}}{\gamma(\vshapec)}{\vn}$ is \vtrue,
      then $\GsigmaModelsThree{G}{\sigma}{\vshapec}{\vn}$ is \vtrue.
  \end{itemize}

  In Kleene's 3-valued logic, the evaluation of a sentence into \vtrue\ or
  \vfalse\ implies that this evaluation does not depend on any of its
  sub-sentences that are evaluated to \vundefined\ (\ie, changing the truth
  value of one such sub-sentence would not affect the truth value of the whole
  sentence).
  Notice also that the only atoms that can be evaluated as \vundefined\ are
  shape references $s(x)$~\citep{CRS18}.
  This means that if the 3-valued evaluation of a constraint $\vshapec$ over a
  node, a graph and an assignment is \vtrue\ (\resp, \vfalse), then this
  evaluation would still be \vtrue\ (\resp, \vfalse), if every shape atom $s(x)$
  that evaluates to \vundefined\ evaluates to \vfalse\ instead.

  $(\Rightarrow)$
  If $\GsigmaModels{G}{\sigma}{\vshapec}{\vn}$ evaluates to \vtrue, then
  $\GsigmaModels{G}{\sigma^{\gamma}}{\gamma(\vshapec)}{\vn}$ must also evaluate
  to \vtrue, since in the transformation from $\vshapec$ to $\gamma(\vshapec)$
  \begin{inparaenum}[(1)]
    \item
      every constraint that is not a shape reference remains unchanged, and
    \item
      every shape reference (in $\vshapec$) is transformed into a conjunction
      of shape references (in $\gamma(\vshapec)$) that still evaluates to the
      same truth value of the original expression, unless this truth value is
      \vundefined.
  \end{inparaenum}
  However, by our previous observation, changing an \vundefined\ truth value
  cannot affect the truth value of
  $\GsigmaModels{G}{\sigma^{\gamma}}{\gamma(\vshapec)}{\vn}$ since
  $\GsigmaModels{G}{\sigma}{\vshapec}{\vn}$ evaluates to \vtrue.
  Thus implication $(\Rightarrow)$ holds.

  $(\Leftarrow)$
  Similarly, if $\GsigmaModels{G}{\sigma^{\gamma}}{\gamma(\vshapec)}{\vn}$
  evaluates to \vtrue, then $\GsigmaModels{G}{\sigma}{\vshapec}{\vn}$ must also
  evaluate to \vtrue, since, in the inverse transformation from $\gamma(\vshapec)$ to $\vshapec$:
  \begin{inparaenum}[(1)]
    \item
      every constraint that is not a shape reference remains unchanged, and
    \item
      every pair of shape references ``$\vshape^{+}(x) \wedge \neg
      \vshape^{-}(x)$'' or ``$\neg \vshape^{+}(x) \wedge \vshape^{-}(x)$'' is
      transformed into a single shape reference which either
      \begin{inparaenum}[(a)]
          \item
            evaluates to the same truth value, or
          \item
            evaluates to the truth value of \vundefined\ when the original
            constraint evaluates to \vfalse.
      \end{inparaenum}
  \end{inparaenum}
  Notice that in \SHACL, the constraints of a shape are considered in
  conjunction, and negation only appears in front of shape references.
  Since $\GsigmaModels{G}{\sigma^{\gamma}}{\gamma(\vshapec)}{\vn}$ evaluates to
  \vtrue, a pair of shape references ``$\vshape^{+}(x) \wedge \neg
  \vshape^{-}(x)$'' or ``$\neg \vshape^{+}(x) \wedge \vshape^{-}(x)$'' that
  evaluates to \vfalse\ \wrt $\vn$, $G$ and $\sigma^{\gamma}$ can only appear in
  a disjunction in $\gamma(\vshapec)$ of which at least one disjunct evaluates
  to \vtrue\ \wrt $\vn$, $G$ and $\sigma^{\gamma}$, since this disjunction
  cannot be within the scope of negation.
  Pairs of shape references ``$\vshape^{+}(x) \wedge \neg \vshape^{-}(x)$'' or
  ``$\neg \vshape^{+}(x) \wedge \vshape^{-}(x)$'' that evaluate to \vfalse\
  \wrt $\vn$, $G$ and $\sigma^{\gamma}$, therefore, do not affect the truth
  value of $\GsigmaModels{G}{\sigma^{\gamma}}{\gamma(\vshapec)}{\vn}$.
  Thus implication $(\Leftarrow)$ holds as well.
\end{proof}

From Definition~\ref{gammatransformation} and
Lemma~\ref{gammatransformationProperties} the main theorem of this subsection easily follows.

\begin{theorem}
  \label{theorem:partialToTotal}
  Given a \SHACL document $\vShapeDocument$ and a graph $G$, it holds that $G$
  is valid \wrt $\vShapeDocument$ under brave-partial (\resp,
  cautious-partial) semantics \iff $G$ is valid \wrt $\Gamma(\vShapeDocument)$
  under brave-total (\resp, cautious-total) semantics.
\end{theorem}

Thus, in the rest of the article we will only focus on total assignments and we shall use the term \emph{brave semantics} to refer to
brave-total and \emph{cautious semantics} to refer to cautious-total.




\section{Shapes Constraint Logic: \SCL}
\label{sec:scl}

In this section we provide a precise formalisation of \SHACL semantics and
related decision problems in a formal logical system.
For the sake of simplicity of presentation, we first focus on the brave
semantics only, and then show how to adapt our system to model cautious
semantics (recall that, as shown in Section~\ref{sub:partot}, partial
assignments semantics is, model-theoretically, a special case of total
assignments semantics).
The main component of this logical system is the \SCL language, a novel fragment
of first-order logic extended with counting quantifiers and the transitive
closure operator, that precisely models \SHACL documents.
We will later show the equivalidity of \SHACL and \SCL, by demonstrating how,
for any graph, the latter can be used to model total faithful assignments.

Our decision problems, instead, are modelled using \MSCL, a fragment of monadic
second-order logic defined on top of \SCL, by extending the latter with
second-order quantifications on monadic relations.
Intuitively, \MSCL allows us to define conditions over the space of all possible
assignments, something that cannot be expressed in \SCL.
Nevertheless, as we will see later, several formulations of our decision
problems are fully reducible to the first-order logic satisfiability problem.

\subsection{A First-Order Logic for \SHACL}
\label{sec:scl;sub:firord}

In the presentation of our logical system and in the analysis of its decision
problems, we consider arbitrary first-order relational models with equality as
the only built-in relation.
When we deal with the \SHACL encoding, instead, we assume the first-order models
to have the set of RDF terms as the domain of discourse, plus a set of
interpreted relations for the \SHACL filters.

Assignments are modelled by means of a set of monadic relation names
\hasshapePredicate, called \emph{shape relations}.
In particular, each shape $\vshape$ is associated with a unique shape relation
$\hasshapePredicateS{\vshape}$.
If $\hasshapePredicateS{\vshape}$ is a shape relation associated with shape
$\vshape$, then fact $\hasshapePredicateS{\vshape}(\varElm)$ (resp. $\neg
\hasshapePredicateS{\vshape}(\varElm)$) describes an assignment $\sigma$ such
that $\vshape \in \sigma(\varElm)$ (resp. $ \neg \vshape \in \sigma(\varElm)$).
Since our logical system uses standard Boolean logic, for any element of the
domain $\conElm$ and shape relation \hasshapePredicate, it holds that
$\hasshapePredicate(\conElm) \vee \neg \hasshapePredicate(\conElm)$ holds, by
the law of excluded middle.
Thus any Boolean interpretation of shape relations defines a total assignment.

\begin{table}[t]
  \begin{center}
    \small
    \tabtargets
  \end{center}
  \vspace{-1em}
  \caption{\label{tab:targets} Translation of a \SHACL shape with name $\vshape$
    and target declaration $\vshapet$, into an \SCL target axiom.%
    }
\end{table}

Sentences and formulae in the \SCL language follow the grammar reported in
Definition~\ref{def:scl}, whose main syntactic components are described later
on.
In the rest of the article, we will focus on this logic to study the
decidability and complexity of our \SHACL decision problems.
In particular, we are going to reserve the symbols $\tau$ and $\tau^{-}$ to
denote the translations from \SHACL documents into \SCL sentences and
\viceversa and refer the reader to the appendix for the full details about these
translations.
Bold capital letters in square brackets on the right of some of the grammar
production rules are pure meta-annotations for naming \SCL features and,
obviously, not an integral part of the syntax.

\begin{definition}
  \label{def:scl}
  The \emph{Shape Constraint Logic} (\SCL, for short) is the set of first-order
  \emph{sentences} $\varphiSnt$ built according to the following context-free
  grammar, where $\conElm$ is a constant from the domain of RDF terms,
  $\hasShapeElm$ is a shape-relation name, $\FSym$ is a filter-relation name,
  $\RRel$ is a binary-relation name, Kleene's star symbol $^{\star}$ indicates
  the transitive closure of the binary relation induced by $\pi(\varElm,
  \yvarElm)$, the superscript $\pm$ stands for a relation or its inverse, and $n
  \in \SetN$:
  \begin{align*}
    {\varphiSnt}
  & \;{\seteq}\;
    {\top \mid \varphiSnt \wedge \varphiSnt} \\
  & \;{\;\:\mid\;}\;
    {\hasShapeElm(\conElm) \mid \forall \varElm \ldotp \isASym(\varElm, \conElm)
    \rightarrow \hasShapeElm(\varElm) \mid \forall \varElm, \yvarElm \ldotp
    \RRel[][\pm](\varElm, \yvarElm) \rightarrow \hasShapeElm(\varElm)} \\
  & \;{\;\:\mid\;}\;
    {\forall \varElm \ldotp \hasShapeElm(\varElm) \leftrightarrow
    \psiFrm(\varElm)};
  \\
    {\psiFrm(\varElm)}
  & \;{\seteq}\;
    {\top \mid \neg \psiFrm(\varElm) \mid \psiFrm(\varElm) \wedge
    \psiFrm(\varElm) \mid \varElm = \conElm \mid \FSym(\varElm) \mid
    \hasShapeElm(\varElm) \mid \exists \yvarElm \ldotp \piFrm(\varElm, \yvarElm)
    \wedge \psiFrm(\yvarElm)}
  & {\textsf{\textbf{[$\varnothing$]}}} \\
  & \;{\;\:\mid\;}\;
    {\neg \exists \yvarElm \ldotp \piFrm(\varElm, \yvarElm) \wedge
    \RRel(\varElm, \yvarElm)}
  & {\textsf{\textbf{[$\D$]}}} \\
  & \;{\;\:\mid\;}\;
    {\forall \yvarElm \ldotp \piFrm(\varElm, \yvarElm) \leftrightarrow
    \RRel(\varElm, \yvarElm)}
  & {\textsf{\textbf{[$\E$]}}} \\
  & \;{\;\:\mid\;}\;
    {\forall \yvarElm, \zvarElm \ldotp \piFrm(\varElm, \yvarElm) \wedge
    \RRel(\varElm, \zvarElm) \rightarrow \varsigmaFrm(\yvarElm, \zvarElm)}
  & {\textsf{\textbf{[$\O$]}}} \\
  & \;{\;\:\mid\;}\;
    {\exists^{\geq n} \yvarElm \ldotp \piFrm(\varElm, \yvarElm) \wedge
    \psiFrm(\yvarElm)};
  & {\textsf{\textbf{[$\C$]}}}
  \\
    {\piFrm(\varElm, \yvarElm)}
  & \;{\seteq}\;
    {\RRel[][\pm](\varElm, \yvarElm)} \\
  & \;{\;\:\mid\;}\;
    {\exists \zvarElm \ldotp \piFrm(\varElm, \zvarElm) \wedge \piFrm(\zvarElm,
    \yvarElm)}
  & {\textsf{\textbf{[$\S$]}}} \\
  & \;{\;\:\mid\;}\;
    {\varElm = \yvarElm \vee \piFrm(\varElm,  \yvarElm)}
  & {\textsf{\textbf{[$\Z$]}}} \\
  & \;{\;\:\mid\;}\;
    {\piFrm(\varElm, \yvarElm) \vee \piFrm(\varElm, \yvarElm)}
  & {\textsf{\textbf{[$\A$]}}} \\
  & \;{\;\:\mid\;}\;
    {(\piFrm(\varElm, \yvarElm))^{\star}};
  & {\textsf{\textbf{[$\T$]}}}
  \\
    {\varsigmaFrm(\varElm, \yvarElm)}
  & \;{\seteq}\;
    {\varElm <^{\pm} \yvarElm \mid \varElm \leq^{\pm} \yvarElm}.
  \end{align*}
\end{definition}

Intuitively, sentences obtained through grammar rule $\varphiSnt$ correspond to
\SHACL documents.
These could be empty ($\top$), a conjunction of documents, a \emph{target axiom}
(production rules 3, 4, and 5 of rule $\varphiSnt$) or a \emph{constraint axiom}
(production rule 6 of rule $\varphiSnt$).
Target axioms take one of three forms, based on the type of target declarations
in the shapes of a \SHACL document.
There are four types of target declarations in \SHACL, namely
\begin{inparaenum}[(1)]
  \item
    a particular constant \texttt{c} (node target),
  \item
    instances of class \texttt{c} (class target), or
  \item
    \!-\!
  \item
    subjects/objects of a triple with predicate \texttt{R}
    (subject-of/object-of target).
\end{inparaenum}
The full correspondence of \SHACL target declarations to \SCL target axioms is
summarised in Table~\ref{tab:targets}.
The correspondence of a target definition containing multiple target
declarations, is simply the conjunction of the corresponding target axioms.

The non terminal symbol $\psiFrm(\varElm)$ corresponds to the subgrammar of the
\SHACL constraints components.
Within this subgrammar, the true symbol $\top$ identifies an empty constraint,
$\varElm = \conElm$ a constant equivalence constraint and $\FSym$ a monadic
filter relation (\eg, $\FSym[][\isIRI](\varElm)$, true \iff $\varElm$ is an
IRI).
By \emph{filters} we refer to the \SHACL constraints about ordering, node-type,
datatype, language tag, regular expressions, and string length
\citep{KK17}. Filters are captured by the $\FSym(\varElm)$
production rule and the \O component.
The \C component captures qualified value shape cardinality constraints.
The \E, \D\ and \O components capture the equality, disjointedness  and order
property pair components.

The $\piFrm(\varElm, \yvarElm)$ subgrammar models \SHACL property paths.
Within this subgrammar \S denotes sequence paths, \A denotes alternate paths,
\Z denotes a zero-or-one path, and, finally, \T denotes a zero-or-more path.

As usual, to enhance readability, we define the following syntactic shortcuts:
\begin{itemize}
  \item
    $\psiFrm[1](\varElm) \vee \psiFrm[2](\varElm) \seteq \neg (\neg
    \psiFrm[1](\varElm) \wedge \neg \psiFrm[2](\varElm))$;
  \item
    $\piFrm(\varElm, \conElm) \seteq \exists \yvarElm . \piFrm(\varElm,
    \yvarElm) \wedge \yvarElm = \conElm$;
  \item
    $\forall \yvarElm \,.\, \piFrm(\varElm, \yvarElm) \rightarrow \psiFrm(y)
    \seteq \neg \exists \yvarElm \,.\, \piFrm(\varElm, \yvarElm) \wedge \neg
    \psiFrm(\yvarElm)$.
\end{itemize}

The above mentioned translations $\tau$ and $\tau^{-}$ between \SHACL and \SCL
are polynomial in the size of the input and computable in polynomial time.
Intuitively, as we show later in
Theorem~\ref{centralSCL2SHACLCorrespondenceTheorem}, a \SHACL document
$\vShapeDocument$ validates a graph $G$ \iff a first-order structure
representing the latter satisfies the \SCL sentence $\tau(\vShapeDocument)$.
\Viceversa, every \SCL sentence $\varphiSnt$ is satisfied by a first-order
structure representing graph $G$ \iff the \SHACL document $\tau^{-}(\varphiSnt)$
validates $G$.

Another important property of these translations is that they preserve the
notion of \emph{\SHACL recursion}, that is, a \SHACL document $\vShapeDocument$
is recursive \iff the \SHACL document $\tau^{-}(\tau(\vShapeDocument))$ is
recursive.
We will call an \SCL sentence $\phi$ \emph{recursive} if $\tau^{-}(\phi)$ is
recursive.

Given a \SHACL document $\vShapeDocument$, the \SCL sentence
$\tau(\vShapeDocument)$ contains a shape relation $\hasshapePredicateS{\vshape}$
for each shape $\vshape$ in $\vShapeDocument$.
Sentence $\tau(\vShapeDocument)$ can be split into constraint axioms and
target axioms.
Intuitively, these are used to verify the first and second condition of
Definition~\ref{def:faithfulAssignment}, respectively.
The constraint axioms of $\tau(\vShapeDocument)$ correspond to the sentence
$\tau(\minustargets{\vShapeDocument})$, \ie, to the translation of the document
ignoring targets, while the target axioms of $\tau(\vShapeDocument)$ correspond
to taking targets into account, \ie, to a sentence $\phiSnt$, where $\phiSnt
\wedge \tau(\minustargets{\vShapeDocument})$ is $\tau(\vShapeDocument)$.

Note that our translation $\tau$ results in a particular structure of \SCL
sentences, that we will call \emph{well-formed}, and thus we restrict the
inverse translation $\tau^{-}$ and define it only on well-formed \SCL sentences.
An \SCL sentence $\varphiSnt$ is well-formed if, for every shape relation
$\Sigma$, sentence $\varphiSnt$ contains exactly one constraint axiom with
relation $\Sigma$ on the left-hand side of the implication.
Intuitively, this condition ensures that every shape relation is ``defined'' by
a corresponding constraint axiom.
Figure~\ref{fig:TranslationExample} shows the translation of the document from
Figure~\ref{fig:MainExample} into a well-formed \SCL sentence.

\begin{figure}[t]
\begin{align*}
&\quad \Big( \forall \vx . \;  \isA(\vx, \texttt{:Student}) \rightarrow \hasshape{\vx}{\texttt{:studentShape}} \Big) \\
&\quad \wedge \Big( \forall \vx . \; \hasshape{\vx}{\texttt{:studentShape}} \leftrightarrow  \neg \hasshape{\vx}{\texttt{:disjFacultyShape}} \Big) \\
&\quad \wedge \Big( \forall \vx . \; \hasshape{\vx}{ \texttt{:disjFacultyShape}} \leftrightarrow  \\
&\quad \quad \qquad \neg \exists \vy . \; ( \; \exists \vz . \;   R_{\texttt{:hasSupervisor}}(\vx, \vz) \; \wedge \\
&\quad\quad \qquad \hphantom{\neg \exists \vy . ( \; \exists \vz . \;\;} R_{\texttt{:hasFaculty}}(\vz, \vy) \; \wedge \\
&\quad \quad \qquad \hphantom{\neg \exists \vy . \big( \; \exists \vz . \;\;} R_{\texttt{:hasFaculty}}(\vx, \vy) \; \big) \; \Big)
\end{align*}
  \captionof{figure}{\label{fig:TranslationExample} Translation of the \SHACL
    document from Figure~\ref{fig:MainExample} into an \SCL sentence.}
\end{figure}

Before defining the semantic correspondence between \SHACL and \SCL we introduce
the translations of graphs and assignments into first-order structures.

\begin{definition}
  Given a graph $G$, the first-order structure \induced{$G$} contains a fact
  $\RRel(\sElm, \oElm)$, \ie, $\RRel(\sElm, \oElm)$ holds true in \induced{$G$},
  if $\tuple {\sSym} {\RSym} {\oSym} \in G$.
\end{definition}

\begin{definition}
  Given a total assignment $\sigma$, the first-order structure
  \induced{$\sigma$} contains fact $\hasshape{\vn}{\sconst}$, \ie,
  $\hasshape{\vn}{\sconst}$ holds true in \induced{$\sigma$}, for every node
  $\vn$, if $\sconst \in \sigma(\vn)$.
\end{definition}

\begin{definition}
  Given a graph $G$ and a total assignment $\sigma$, the first-order structure
  $I$ \emph{induced} by $G$ and $\sigma$ is the disjoint union of structures
  \induced{$G$} and \induced{$\sigma$}.
  Given a first-order structure $I$:
  \begin{inparaenum}[(1)]
    \item
      the graph $G$ \emph{induced} by $I$ is the graph that contains triple
      $\tuple {\sSym} {\RSym} {\oSym}$ if $I \models \RRel(\sElm, \oElm)$ and
    \item
      the assignment $\sigma$ induced by $I$ is the assignment such that, for
      all elements of the domain $\vn$ and shape relations
      $\hasshapePredicateS{\vshape}$, fact $\vshape \in \sigma(\vn)$ is true if
      $I \models \hasshape{\vn}{\sconst}$ and $\neg \vshape \in \sigma(\vn)$
      is true if $I \not\models \hasshape{\vn}{\sconst}$.
  \end{inparaenum}
\end{definition}

The semantic correspondence between \SHACL and \SCL is captured by the following
theorem.

\begin{theorem}
  \label{centralSCL2SHACLCorrespondenceTheorem}
  For all graphs $G$, total assignments $\sigma$ and \SHACL documents
  $\vShapeDocument$, it is true that $\isfaithful{G}{\sigma}{\vShapeDocument}$
  \iff $I \models \tau(\vShapeDocument)$, where $I$ is the first-order
  structure induced by $G$ and $\sigma$.
  For any first-order structure $I$ and \SCL\ sentence $\phi$, it is true $I
  \models \phi$ \iff $\isfaithful{G}{\sigma}{\tau^{-}(\phi)}$, where $G$ and
  $\sigma$ are, respectively, the graph and assignment induced by $I$.
\end{theorem}
This theorem can be proved by a tedious but straightforward structural
induction over the document syntax, with an operator-by-operator analysis of
the translation we provide in the appendix.

Sentences in \SCL have a direct correspondence to the sentences of the grammar
presented in~\citep{pareti2020}.
For each non-recursive \SHACL document, the differences between the sentences
obtained by translating this document are purely syntactic and the two sentences
are equisatisfiable.
In particular, the binary relation \texttt{hasShape} of~\citep{pareti2020} is
now represented instead as a set of monadic relations.
For recursive \SHACL documents, the grammar of Definition~\ref{def:scl}
introduces a one-to-one correspondence between \SHACL target
declarations/constraints, and target/constraint axioms respectively.

The sub-grammar $\psiFrm(\varElm)$ in Definition~\ref{def:scl} corresponds to
the grammar of \SHACL constraints from~\citep{CRS18}, with the addition of
filters.
The grammar from~\citep{CRS18} omits filters by assuming that their evaluation
is not more computationally complex than evaluating equality.
This assumption is true for validation, the main decision problem addressed
in~\citep{CRS18}, but it does not hold for satisfiability and containment, as we further discuss in Section~\ref{sec:fltaxm}.

\begin{table}[t]
  \begin{center}
    \footnotesize
    \tabcomponents
  \end{center}
  \vspace{-1em}
  \caption{\label{tab:components} Correspondence between prominent \SHACL
    components and \SCL expressions.}
\end{table}

To distinguish different fragments of \SCL, Table~\ref{tab:components} lists a
number of \emph{prominent} \SHACL components.
The language defined without any of these constructs is our \emph{base}
language, denoted \X.
When using an abbreviation of a prominent feature, we refer to the fragment of
our logic that includes the base language together with that feature enabled.
For example, \S\A identifies the fragment that only allows the base language,
sequence paths, and alternate paths.

The \SHACL specification presents an unusual asymmetry in the fact that
equality, disjointedness and order components (corresponding to \E, \D, and \O
in \SCL) force one of their two path expressions to be an atomic relation.
This can result in situations where order constraints can be defined in just one
direction, since only the less-than and less-than-or-equal property pair
constraints are defined in \SHACL.
Our \O fragment models a more natural order comparison that includes the $>$
and $\geq$ components, by using the inverse of $<$ and $\leq$.
We instead denote by \O' the fragment where the order relations in the
$\sigmagrammar(\vx, \vy)$ subgrammar cannot be inverted.
In our formal analysis of Section~\ref{sec:sclsat} we will consider both \O and
\O'.

\subsection{A Second-Order Logic for \SHACL Decision Problems}
\label{sec:scl;sub:secord}

In order to model \SHACL decision problems, we introduce the \emph{Monadic
Shape Constraint Logic} (\MSCL, for short) built on top of a \emph{second-order
interpretation} of \SCL sentences.
A second-order interpretation of an \SCL sentence $\phi$ is the second-order
formula obtained by interpreting shape relations as free monadic second order
variables.
Obviously, shape relations that are under the scope of the same quantifier
describe the same assignment.
While \SCL can be used to describe the faithfulness of a single assignment,
\MSCL can express properties that must be true for all possible assignments.
This is necessary to model all extended semantics.
As usual, disjunction and implication symbols in \MSCL sentences are just
syntactic shortcuts.

\begin{definition}
  \label{def:mscl}
  The \emph{Monadic Shape Constraint Logic} (\MSCL, for short)  is the set of
  second-order sentences built according to the following context-free grammar
  $\PhiSnt$, where $\varphiSnt$ is an \SCL sentence and $\hasshapePredicate$ is
  the second-order variable corresponding to a shape relation.
  \begin{align*}
    {\PhiSnt}
  & \;{\seteq}\;
    {\varphiSnt \mid \neg \PhiSnt \mid \PhiSnt \wedge \PhiSnt \mid \exists
    \hasShapeElm \ldotp \PhiSnt \mid \forall \hasShapeElm \ldotp \PhiSnt};
  &
    {\varphiSnt}
  & \;{\seteq}\;
    {\SCL}.
  \end{align*}
  The \ESCL (\resp, \USCL) fragment of \MSCL is the set of sentences obtained by
  the above grammar deprived of the negation and universal (\resp, existential)
  quantifier rules.
\end{definition}

Relying on the standard semantics for second-order logic, we define the
satisfiability and containment for \MSCL\ sentences, as well as the closely
related finite-model property, in the natural way.
\begin{description}
  \item[\MSCL Sentence Satisfiability]
    An \MSCL sentence $\PhiSnt$ is satisfiable if there exists a relational
    structure $\Omega$ such that $\Omega \models \PhiSnt$.
  \item[\MSCL Finite-model Property]
    An \MSCL sentence $\PhiSnt$ enjoys the finite-model property if, whenever
    $\PhiSnt$ is satisfiable, it is so on a relational structure.
\end{description}

In Section~\ref{sec:shacl2scl} we discuss the correspondence between the
\SHACL and \MSCL decision problems.
In this respect, we assume that filters are interpreted relations.
In particular, we prove equivalence of \SHACL and \MSCL, for the purpose of
validity, on models that we call \emph{canonical}; that is, models having the
following properties:
\begin{inparaenum}[(1)]
  \item
    the domain of the model is the set of RDF terms,
  \item
    constant symbols are interpreted as themselves (as in a standard Herbrand
    model~\citep{EF95}),
  \item
    such a model contains built-in interpreted relations for filters, and
  \item
    ordering relations $<$ and $\leq$ are the disjoint union of the total orders
    of the different comparison types allowed in SPARQL.
\end{inparaenum}
To enforce the fact that different RDF terms are not equivalent to each other
we adopt the unique name assumption for the constants of our language.
For the purpose of our decision problems, it is sufficient to axiomatise the
inequality of all the known constants.

Finally, we state a trivial result used later on to show how to solve some of
the mentioned decision problem by looking at the ``simpler'' \SCL satisfiability
and validity decision problems.\footnote{The term \emph{valid} here refers to
the notion of validity in mathematical logic and model theory, not to be
confused with \SHACL validation.}
\begin{proposition}\label{proposition:ESCLtoSCL}
  An \ESCL (\resp, \USCL) sentence $\PhiSnt \defeq \exists \hasShapeElm[1]
  \ldots \exists \hasShapeElm[m] \ldotp \varphiSnt$ (\resp, $\PhiSnt \defeq
  \forall \hasShapeElm[1] \ldots \allowbreak \forall \hasShapeElm[m] \ldotp
  \varphiSnt$) is satisfiable (\resp, valid) \iff the subformula $\varphiSnt$
  interpreted as an \SCL sentence is satisfiable (\resp, valid).
\end{proposition}




\section{From \SHACL Decision Problems to \MSCL Satisfiability}
\label{sec:shacl2scl}

The rich expressiveness of the \MSCL language, defined in the previous section, allows us to formally define several
decision problems. We first use this language to define the main such problems studied in this article,
namely \SHACL validation, satisfiability and containment. We then show how \MSCL can also capture a number
of related decision problems that have been proposed in the literature.

\subsection{Principal Decision Problems}
\label{sec:shacl2scl;sub:prn}

In this section we describe the equivalidity of \MSCL\ and \SHACL, and provide a reduction of our decision
problems into \MSCL\ satisfiability. Notably, we also show how some of them can be further reduced into \ESCL.
As we will see later, this last reduction can be easily translated to a reduction into first-order
logic, from which we derive several decidability results.

We again focus only on total assignment semantics.
Given a second-order formula $\phi$, second-order interpretation of an \SCL\ sentence, we denote with $\qexists{\phi}$,
respectively $\qforall{\phi}$, the \MSCL\ sentence obtained by existentially, respectively universally, quantifying
all of the shape relations of $\phi$. Recall that, by construction, the assignments induced by models of an \MSCL\ sentence
are total, and that the second-order variables under the scope of the same quantifier represent a single assignment.

The following corollaries, which rely on the standard notion of modelling of a sentence by a structure, easily follow from Theorem \ref{centralSCL2SHACLCorrespondenceTheorem} and
the definitions of validity from Table \ref{tab:recursiveSemanticsTable}. 
The first two corollaries define the correspondence between \SHACL\ and \MSCL\ validation.
The last four corollaries express our formalisation of the \SHACL
satisfiability and containment decision problems in the case of brave validation and in the case of cautious validation.
Recall also that $\induced{$G$}$ denotes the first-order structure induced by a graph $G$, and $\minustargets{\vShapeDocument}$
denotes the \SHACL document obtained by removing all target declarations from \SHACL document
$\vShapeDocument$, which we use to test first condition of Def. \ref{def:faithfulAssignment} in isolation
from the second.

\begin{corollary}[Brave-Total Validation]
\label{cor:brave-total-val}
A graph $G$ is valid \wrt a \SHACL document \vShapeDocument\ under brave-total semantics if \induced{$G$}
$\models \qexists{\tau(\vShapeDocument)}$.
\end{corollary}

\begin{corollary}[Cautious-Total Validation]
A graph $G$ is valid \wrt a \SHACL document \vShapeDocument\ under cautious-total semantics if \induced{$G$}
$\models \qexists{\tau(\vShapeDocument)} \wedge \qforall{\tau(\minustargets{\vShapeDocument})
\rightarrow \tau(\vShapeDocument)}$.
\end{corollary}

\begin{corollary}[Brave-Total Satisfiability]
For any \SHACL document $\vShapeDocument$, document $\vShapeDocument$ is (finitely) satisfiable under brave-total semantics if
$\qexists{\tau(\vShapeDocument)}$ is (finitely) satisfiable.
\end{corollary}

\begin{corollary}[Cautious-Total Satisfiability]
For any \SHACL document $\vShapeDocument$, document $\vShapeDocument$ is (finitely) satisfiable under cautious-total semantics
if $\qexists{\tau(\vShapeDocument)} \wedge \qforall{\tau(\minustargets{\vShapeDocument})
\rightarrow \tau(\vShapeDocument)}$ is (finitely) satisfiable.
\end{corollary}

\begin{corollary}[Brave-Total Containment] \label{brave-total-corollary}
For any pair of \SHACL documents $\vShapeDocument_1$ and $\vShapeDocument_2$, document $\vShapeDocument_1$
is contained in $\vShapeDocument_2$ under brave-total semantics iff $\qexists{\tau(\vShapeDocument_1)} \rightarrow
\qexists{\tau(\vShapeDocument_2)}$ is valid, that is, iff $\qexists{\tau(\vShapeDocument_1)} \wedge \neg
\qexists{\tau(\vShapeDocument_2)}$ is unsatisfiable.
\end{corollary}

\begin{corollary}[Cautious-Total Containment]
\label{cor:cautious-total-cont}
For any pair of \SHACL documents $\vShapeDocument_1$ and $\vShapeDocument_2$,
document $\vShapeDocument_1$ is contained in $\vShapeDocument_2$ under cautious-total semantics if

$\left( \qexists{\tau(\vShapeDocument_1)} \wedge \qforall{\tau(\minustargets{\vShapeDocument_1}) \rightarrow
\tau(\vShapeDocument_1)} \right)
\rightarrow
\left( \qexists{\tau(\vShapeDocument_2)} \wedge \qforall{\tau(\minustargets{\vShapeDocument_2}) \rightarrow \tau(\vShapeDocument_2)} \right) $

is valid, that is, iff

$\left( \qexists{\tau(\vShapeDocument_1)} \wedge \qforall{\tau(\minustargets{\vShapeDocument_1}) \rightarrow
\tau(\vShapeDocument_1)} \right)
\wedge \neg
\left( \qexists{\tau(\vShapeDocument_2)} \wedge \qforall{\tau(\minustargets{\vShapeDocument_2}) \rightarrow \tau(\vShapeDocument_2)} \right) $

is unsatisfiable.

\end{corollary}

We now provide a simplified definition of containment for non-recursive \SHACL documents by exploiting the
properties of Lemma \ref{propertyOfNonRecursiveDocs}, and the fact that all extended semantics are equivalent for
non-recursive SHACL.

\begin{lemma} \label{lemma:nonrecursive_containment_to_sat}
For any pair of non-recursive \SHACL documents $\vShapeDocument_1$ and $\vShapeDocument_2$ document $\vShapeDocument_1$
is contained in $\vShapeDocument_2$ iff $\qexists{\tau(\vShapeDocument_1)} \wedge
\qexists{\tau(\minustargets{\vShapeDocument_2}) \wedge \neg \tau(\vShapeDocument_2)}$ is not satisfiable.
\end{lemma}
\begin{proof}
For non-recursive \SHACL documents all semantics are equivalent, thus containment of two non-recursvie \SHACL documents can
be expressed as containment under brave-total semantics (Corollary \ref{brave-total-corollary}), namely the
unsatisfiability of $\qexists{\tau(\vShapeDocument_1)} \wedge  \qforall{\neg \tau(\vShapeDocument_2)}$.
Notice that for all assignment $\sigma$ and graphs $G$, if $\isnotfaithful{G}{\sigma}{\minustargets{\vShapeDocument}}$
then trivially $\isnotfaithful{G}{\sigma}{\vShapeDocument}$, thus we can rewrite containment as the unsatisfiability
of the following sentence:

$\qexists{\tau(\vShapeDocument_1)} \wedge  \qforall{\neg \tau(\minustargets{\vShapeDocument_2}) \vee  \neg \tau(\vShapeDocument_2) }$,

which is trivially equivalent to the following:

$\qexists{\tau(\vShapeDocument_1)} \wedge  \qforall{  \tau(\minustargets{\vShapeDocument_2}) \rightarrow  \neg \tau(\vShapeDocument_2) }$ is unsatisfiable.

From Lemma \ref{propertyOfNonRecursiveDocs} we know that, for any graph $G$, there exists an assignment $\sigma$ such that $ \isfaithful{G}{\sigma}{\minustargets{\vShapeDocument}}$. By Theorem \ref{centralSCL2SHACLCorrespondenceTheorem}, the structure $\induced{$G$}$ induced by any $G$ models $\qexists{  \tau(\minustargets{\vShapeDocument_2})}$, and thus $\qexists{  \tau(\minustargets{\vShapeDocument_2})}$ is true for any model. We can therefore rewrite the containment criterion as the unsatisfiability of the following sentence:

$\qexists{\tau(\vShapeDocument_1)} \wedge  \qexists{  \tau(\minustargets{\vShapeDocument_2})} \wedge \qforall{  \tau(\minustargets{\vShapeDocument_2}) \rightarrow  \neg \tau(\vShapeDocument_2) }$,

which is trivially equivalent to:

$\qexists{\tau(\vShapeDocument_1)} \wedge  \qexists{  \tau(\minustargets{\vShapeDocument_2}) \wedge \neg \tau(\vShapeDocument_2)} \wedge \qforall{  \tau(\minustargets{\vShapeDocument_2}) \rightarrow  \neg \tau(\vShapeDocument_2) }$.

From Lemma \ref{propertyOfNonRecursiveDocs} we also know that there is only one assignment $\sigma$ such that $ \isfaithful{G}{\sigma}{\minustargets{\vShapeDocument}}$, thus the conjunct in the for all quantification can be removed. \end{proof}

From the above results we can notice that several decision problems are
reducible to the satisfiability of \ESCL sentences, which, as defined in
Proposition \ref{proposition:ESCLtoSCL}, can be further reduced to the
satisfiability of \SCL.
In Section \ref{sec:sclsat} we will study the properties of \SCL\ to provide
decidability and complexity results for our decision problems that can be
reduced to \ESCL\ satisfiability, namely the satisfiability and containment of
non-recursive \SHACL documents, and satisfiability of (recursive) \SHACL
documents under brave-total (and thus also brave-partial) semantics.
The remaining decision problems, namely containment for recursive \SHACL
documents (under any extended semantics), and satisfiability for recursive
\SHACL documents under cautious validation, require the expressiveness of
second-order logic, and are likely undecidable even for very restrictive
fragments of SHACL.

It is important to notice that the undecidability results of
Section~\ref{sec:sclsat} only consider the arbitrary unrestricted (non-finite)
satisfiability problem.
It is not immediately clear whether these can be extended to the finite problem
too, but we conjecture that a Trakhtenbrot-like undecidability
proof~\cite{Tra50,Lib04} can be used for the \SCL fragments containing at least
the \O construct.

\subsection{Additional Decision Problems}
\label{sub:add}

Our logical framework allows us to express a number of additional decision problems that shift the focus on more fine-grained objects, such as shapes and constraints. While these additional decision problems are not the focus of this article, we discuss them for the sake of completeness.
To better model these additional problems, we will use $\vshapet_{\vn}$ to denote a constraint definition that targets the single node $\vn$.

Given a \SHACL document $\vShapeDocument$, and two shapes $\vshape$ and $\vshape'$ in $\vShapeDocument$, the decision problem of \emph{shape containment}~\citep{martin2020shapecontainment} determines whether $\vshape$ is \emph{contained} in $\vshape'$. Intuitively, this means that whenever $\vShapeDocument$ is used for validation, nodes conforming to $\vshape$ necessarily conform to $\vshape'$. The definition of shape containment, adapted to the notation of our article, is the following.

\begin{definition}
Given a \SHACL document $\vShapeDocument$, and two shapes $\trip{\vshape}{\vshapet}{\vshapec}$ and $\trip{\vshape'}{\vshapet'}{\vshapec'}$ in $\vShapeDocument$, $\vshape$ is \emph{shape contained} in $\vshape'$ under brave-partial (resp.\ brave-total) semantics if,
for all graphs $G$, nodes $\vn$ in $\nodesof{G}{\emptyset}$ and assignments $\sigma$ in
$\assignmentsPartial{G}{\vShapeDocument}$ (resp.\ $\assignmentsTotal{G}{\vShapeDocument}$)
such that $\isfaithful{G}{\sigma}{\vShapeDocument}$, if $\vshape \in \sigma(\vn)$ then
$\vshape' \in \sigma(\vn)$. \end{definition}
While the original definition only considered brave-total semantics,
our formulation is more general, as it also includes brave-partial.
It is important to notice that, if a \SHACL document is unsatisfiable, any pair of shapes
within that document trivially contain each other. In other words, the containment of a shape into
another is not necessarily caused by any particular property of those shapes.

We should also note that the fragment studied in~\citep{martin2020shapecontainment}
for which shape containment is decidable is the \SHACL fragment corresponding to the \SCL\ sub-fragment
of \C\ (the base language plus counting quantifiers) where filters are not allowed.
This is in agreement with our decidability results, that we present in Sec. \ref{sec:sclsat}, where we demonstrate decidability of the similar \SHACL satisfiability problem for even more general fragments of \C .

The shape containment problem can be expressed as the existence of a node $\vn$ such that document $\vShapeDocument \cup \{\trip{\vshape^{*}}{\vshapet_{\vn}}{\vshapec^{*}}\}$ is unsatisfiable under brave-partial (resp.\ brave-total) semantics, where $\vshape^{*}$ is a fresh shape name, $\vshapet_{\vn}$ is a target declaration that targets only node $\vn$, and $\vshapec^{*}$ is the constraint obtained by conjuncting $\vshapec'$ and the negation of $\vshapec$.

\begin{theorem}
Given a \SHACL document $\vShapeDocument$, and two shapes $\trip{\vshape}{\vshapet}{\vshapec}$ and $\trip{\vshape'}{\vshapet'}{\vshapec'}$ in $\vShapeDocument$, $\vshape$ is not \emph{shape contained} in $\vshape'$ under brave-partial (resp.\ brave-total) semantics iff there exist a node $\vn$ such that document $\vShapeDocument \cup \{\trip{\vshape^{*}}{\vshapet_{\vn}}{\vshapec^{*}}\}$ is satisfiable under brave-partial (resp.\ brave-total) semantics, where $\vshape^{*}$ is a fresh shape name, $\vshapet_{\vn}$ is a target declaration that targets only node $\vn$, and $\vshapec^{*}$ is the constraint obtained by conjuncting $\vshapec$ and the negation of $\vshapec'$.
\end{theorem}
\begin{proof}
Given a node $\vn$ let $\vShapeDocument' = \vShapeDocument \cup \{\trip{\vshape^{*}}{\vshapet_{\vn}}{\vshapec^{*}}\}$.

($\Rightarrow$) If $\vShapeDocument'$ is satisfiable, let $G$ be a graph that
is valid \wrt it. If $\vn \in \nodesof{G}$
it is easy to see that the following properties are true for graph $G$: (1) it is valid \wrt $\vShapeDocument$
(since  $\vShapeDocument$ is a subset of  $\vShapeDocument'$),
(2) there exists an assignment $\sigma$ that is faithful
(resp.\ faithful and total) for $\vShapeDocument$ and $G$, and such that $\vshape \in \sigma(\vn)$ and
$\neg \vshape' \in \sigma(\vn)$ (since $n$ satisfies constraints $\vshapec$, but not $\vshapec'$).
One such assignment $\sigma$ can be obtained by taking an assignment $\sigma'$, faithful for $G$ and $M'$, and by removing elements $\vshape^{*}$ and $\neg \vshape^{*}$ from all the sets in the codomain of the $\sigma'$ function. Thus, shape $\vshape$ is not contained in $\vshape'$ w.r.t.\ $\vShapeDocument$. Instead, if $\vn \not \in \nodesof{G}$, then there exists another graph $G'$ such that $G'$ is valid \wrt $\vShapeDocument'$ and $\vn \in \nodesof{G'}$. One such graph $G'$ is $G \cup \{\tripsquare{\vn^{*}}{r^{*}}{\vn}\}$, where $\vn^{*}$ and $r^{*}$ are, respectively, a fresh constant and a fresh relation name. This is because the shapes of a \SHACL document can only target nodes mentioned in the document, or those that are reachable by the relations mentioned in the document. Moreover, the evaluation of any \SHACL constraints on a node is unaffected by that node being the object of a triple with an unknown predicate. Since $G'$ satisfies the same properties as $G$, we can apply the same reasoning as above (as for case $\vn \in \nodesof{G}$) to prove that shape $\vshape$ is not contained in $\vshape'$ w.r.t.\ $\vShapeDocument$.

($\Leftarrow$) If shape $\vshape$ is not contained in $\vshape'$ w.r.t.\
$\vShapeDocument$ then there exists a graph $G$, an assignment $\sigma$ faithful
(resp.\ faithful and total) for $G$ and $\vShapeDocument$, and a node $\vn$ such
that $\vshape \in \sigma(\vn)$ and $\neg \vshape' \in \sigma(\vn)$. Therefore,
$\GsigmaModels{G}{\sigma}{\vshapec^{*}}{\vn}$ must be true. Let $\sigma^{*}$ be
the extension of the $\sigma$ assignment that accounts for the $\vshape^{*}$
shape, namely $\sigma^{*}(\vj) = \sigma(\vj) \cup \{\vshape^{*} |
\GsigmaModels{G}{\sigma}{\vshapec^{*}}{\vj} = \top \} \cup \{\neg \vshape^{*} |
\neg \GsigmaModels{G}{\sigma}{\vshapec^{*}}{\vj} = \top\}$, for any node $\vj$
in $\nodesof{G}{\vShapeDocument}$. It is easy to see that assignment
$\sigma^{*}$ is faithful (resp.\ faithful and total) for $\vShapeDocument'$ and
$G$, and thus $\vShapeDocument'$ is satisfiable.
\end{proof}

The above mentioned theorem introduces the following auxiliary decision problem.
\begin{definition}\label{def:teplatesat}
Given a \SHACL document $\vShapeDocument$, a shape name $\vshape$ not in $\vShapeDocument$ and a constraint $\vshapec$  that only references shapes in $\vShapeDocument \cup \{\vshape\}$, \emph{\templatesat} under brave-partial (resp.\ brave-total) semantics is the problem of deciding whether there exists a node $\vn$ such that document $\vShapeDocument \cup \{\trip{\vshape}{\vshapet_{\vn}}{\vshapec}\}$ is satisfiable under brave-partial (resp.\ brave-total) semantics.
\end{definition}

Two additional decision problems, \emph{constraint satisfiability} and \emph{constraint containment}, are defined in~\citep{pareti2020} to study the properties of non-recursive \SHACL constraints. Intuitively, a constraint $\vshapec$ is satisfiable if there exists a node that conforms to $\vshapec$, and a constraint $\vshapec$ is contained in $\vshapec'$ if every node that conforms to $\vshapec$ also conforms to $\vshapec'$.
We provide here a generalisation of these problems by introducing a \SHACL document as an additional input. The primary purpose of this additional document is to study constraints under recursion, that is, constraints that reference recursive shapes. However, it can also be used to study constraint satisfiability and containment subject to a particular document being valid. When this document is empty the following decision problems correspond to the ones defined in~\citep{pareti2020}, namely constraint satisfiability and containment without recursion.

\begin{definition}
Given a \SHACL constraint $\vshapec$ and a \SHACL document $\vShapeDocument$, such that $\vshapec$ does not reference shapes not included in $\vShapeDocument$, constraint $\vshapec$ is \emph{satisfiable} under extended semantics $\alpha$ if there exists a node $\vn$ such that \SHACL document $\vShapeDocument \cup \{\trip{\vshape}{\vshapet_{\vn}}{\vshapec}\}$ is satisfiable under $\alpha$, where $\vshape$ is a fresh shape name.
\end{definition}

\begin{definition}
Given two \SHACL constraints $\vshapec$ and $\vshapec'$ and a \SHACL documents $\vShapeDocument$ such that $\vshapec$ and $\vshapec'$ do not reference shapes not included in $\vShapeDocument$, constraint $\vshapec$ is \emph{contained} in $\vshapec'$ under extended semantics $\alpha$
if for all nodes $\vn$, document $\vShapeDocument \cup \{\trip{\vshape}{\vshapet_{\vn}}{\vshapec}\}$ is contained in $\vShapeDocument \cup \{\trip{\vshape'}{\vshapet_{\vn}}{\vshapec'}\}$ under $\alpha$, where $\vshape$ and $\vshape'$ are fresh shape names.
\end{definition}

The problem of constraint satisfiability under brave-partial and brave-total semantics are, by definition, sub-problems of \SHACL \templatesat\ for the respective semantics. Constraint containment for non-recursive \SHACL documents is also a sub-problem of \SHACL \templatesat . This is a consequence of the fact that containment of two non-recursive \SHACL documents can be decided by deciding the satisfiability of an \ESCL\ sentence (Lemma \ref{lemma:nonrecursive_containment_to_sat}). As we will prove later in Section \ref{sec:fltaxm}, the problem of \templatesat\ can be expressed as \ESCL\ sentence satisfiability. Therefore, our positive results that will be presented in Section \ref{sec:sclsat} also provide decidability and upper bound complexity results for the decision problems expressible as \templatesat , namely (1) shape containment, (2) constraint satisfiability under brave-partial and brave-total semantics and (3) constraint containment for non-recursive \SHACL documents.




\section{From Interpreted To Uninterpreted Models via Filter Axiomatisation}
\label{sec:fltaxm}
In this section we discuss explicit axiomatizations of the semantics of a set of
filters, inspired by the relational axiomatisation of the \LTL path formulae in
the conjunctive-binding fragment of Strategy Logic~\cite{ABM19}.
The main goal
of these axiomatisations is to account for filter semantics without requiring filters to be interpreted
relations. For any \MSCL\ sentence $\Phi$ we construct axiomatisations $\alpha$ such that $\Phi$ is
satisfiable on a canonical model if and only if $\Phi \wedge \alpha$ is satisfiable on an  \emph{uninterpreted}
models, that is, models whose domain is the set of RDF terms, but where filters and ordering relations are simple
relations instead of interpreted ones. This reduction to standard first-order logic (\FOL) allows us to prove decidability of the
satisfiability and containment problems for several \SCL\ fragments in the face of filters.

We first present a simplified but expensive formulation of this axiomatisation, that is exponential on size
of the original sentence. We then provide an alternative axiomatisation, polynomial on size of the
original sentence, that however requires counting quantifiers to express certain filters. We exclude from
our axiomatisation the \sh{lessThanOrEquals} or \sh{lessThan} constraints (the \O\ and \O' components of
our grammar) that are binary relations, and which do not belong to any decidable fragment we have so far identified,
as shown in the next section.
We also exclude the \sh{pattern} constraint, which tests whether the string representation of
a node follows an extended version of regular expressions,\footnote{Corresponding to SPARQL REGEX functions \cite{BUILARANDA20131}.} from our polynomial axiomatisation.
However, in our simplified axiomatisation we allow a restricted version of the \sh{pattern} constraints precisely corresponding to standard notion of regular expressions (\ie, regular
expressions that can be converted into a finite state machine).
All features defined as filters in
Sec.~\ref{sec:shacl2scl}, with the exception of \O\ and \O' components, are represented by monadic relations
$F(\vx)$ of the \SCL\ grammar. While equality remains an interpreted relation, for which we do not provide
an axiomatisation, we will also consider equality to a constant $\conc$ as a monadic filter relation (which we
call equality-to-a-constant) whose interpretation is the singleton set containing $\conc$.

\subsection{Na\"ive Axiomatisation}
\label{sec:fltaxm;sub:elmenc}

The semantics of each monadic filter relation is a predetermined interpretation over the domain. For example, the
interpretation of filter relation $\FSym[][\isIRI]$ is the set of all IRIs, since $\FSym[][\isIRI](\varElm)$ is
true \iff $\varElm$ is an
IRI. Notice also that filters are the only components of \MSCL\ whose interpretation is predetermined. Thus, we
can axiomatise the semantics of filters w.r.t.\ deciding satisfiability by capturing which conjunctions of
filters are unsatisfiable, and which conjunctions of filters are satisfiable only by a finite set of elements. For
example, the number of elements of the Boolean datatype is two, the number of elements that are literals is infinite,
and there are four elements of integer datatype that are both greater than 0 and lesser than 5. Let a
\emph{filter combination} $\mathds{F}(\vx)$ denote a conjunction of atoms of the form $\vx = \conc$, $\vx \neq \conc$,
$F(\vx)$ or $\neg F(\vx)$, where \conc\ is a constant and $F$ is a filter predicate.
Given a filter combination, it is possible to compute the set of elements of the domain that can satisfy it.
Let $\gamma$ be the function from filter combinations to subsets of the domain that returns this set.
The computation of $\gamma(\mathds{F}(\vx))$ for the monadic filters we consider is tedious but trivial as
it boils down to determining: \begin{inparaenum}[(1)] \item the lexical space of datatypes; \item the cardinality
of intervals defined by order or string-length constraints; \item the number of elements accepted by a regular
expression; \item well-known RDF-specific restrictions, e.g., the fact that each RDF term has exactly one node type,
and at most one datatype and one language tag\end{inparaenum}.
Combinations of the previous four points are similarly computable.
Let $\mathds{F}^{\Phi}$ be the set of filter combinations that can be constructed with the filters predicates
and constants occurring in an \MSCL\ sentence $\Phi$.
The \emph{na\"ive} filter axiomatization $\axiomatisation{\Phi}$ of a sentence $\Phi$ is
the following conjunction, where \hasshapePredicateS{f} is a fresh shape name.

\begin{align*}
\axiomatisation{\phi} =  \bigwedge_{\mathds{F}(\vx) \in \mathds{F}^{\phi}, |\gamma(\mathds{F}(\vx)) \not = \infty|}  &
\left( \forall \vx . \; \hasshape{\vx}{f} \leftrightarrow \mathds{F}(\vx) \right)
\\ & \wedge \left( \forall \vx . \; \hasshape{\vx}{f} \leftrightarrow
\begin{cases}
\bot, & |\gamma(\mathds{F}(\vx))| = 0 \\
\bigvee_{\conc \in \gamma(\mathds{F}(\vx))} \vx = \conc, & \text{otherwise} \end{cases} \right)
\end{align*}

To better illustrate this axiomatisation, consider the following \MSCL\ sentence $\phi^{*}$.
\begin{align*}
\phi^{*} =  \hasshapenosubscript{\conce}{\sconst} \wedge  \forall \vx . \; \big(  \hasshapenosubscript{\vx}{\sconst}  \leftrightarrow  & \exists^{4} \vy . \; R(\vx, \vy) \wedge \mathds{F}^{>0}(\vy) \wedge \mathds{F}^{\leq 5}(\vy) \wedge F^{\hasdatatype = xsd:int}(\vy) \\
& \wedge \vy \not = 2 \wedge \vy \not = 3 \big)
\end{align*}
Intuitively, this sentence is satisfiable if a constant $\conce$ can be in the $R$ relation with four different integers that (a) are greater than 0, (b) that are less than or equal than 5, and (c), that are not equal to 2 or 3. Since there are only three integers that satisfy the conditions (a), (b) and (c) simultaneously, this sentence is not satisfiable on a canonical model.
This sentence contains the filters $\mathds{F}^{>0}(\vx)$, $\mathds{F}^{\leq 5}(\vx)$ and $F^{\hasdatatype = xsd:int}(\vx)$, that denote, respectively, the fact that $\vx$ is greater than the number 0, the fact that $\vx$ is less or equal than the number 5, and the fact that $\vx$ belongs to the XSD integer datatype\footnote{The \url{https://www.w3.org/TR/xmlschema11-2/\#integer} datatype is supported by SPARQL 1.1, and thus it has a predetermined lexical space.}. The set of known constants of $\phi^{*}$ is $\{2,3,\conce\}$. We will assume that $\conce$ is an IRI and that all other known constants are literals of the XSD integer datatype.

The na\"ive filter axiomatisation $\axiomatisation{\phi^{*}}$ contains, among others, the following conjuncts,  where \hasshapePredicateS{f} is a fresh shape name.
\begin{align*}
& \left( \forall \vx . \; \hasshape{\vx}{f} \leftrightarrow \mathds{F}^{>0}(\vx) \wedge \mathds{F}^{\leq 5}(\vx) \wedge F^{\hasdatatype = xsd:int}(\vx) \wedge \vx \not = 2 \wedge \vx \not = 3 \right)
\\ & \wedge \left( \forall \vx . \; \hasshape{\vx}{f} \leftrightarrow \vx = 1 \vee \vx = 4 \vee \vx = 5 \right)
\end{align*}

This axiomatisation states that only three constants satisfy the main filter combination of $\phi^{*}$, and thus $\phi^{*}\wedge \axiomatisation{\phi^{*}}$ is unsatisfiable on an uninterpreted model.

\begin{theorem} \label{thm:filtersNaive}
Given an \MSCL\ sentence $\phi$ and its na\"ive filter axiomatisation $\axiomatisation{\phi}$, sentence $\phi$ is satisfiable on a canonical model iff $\phi \wedge \axiomatisation{\phi}$ is satisfiable on an uninterpreted model.
Containment $\phi_1 \subseteq \phi_2$ of two \MSCL\ sentences on all canonical models holds iff $\phi_1 \wedge \axiomatisation{\phi_1 \wedge \phi_2} \subseteq \phi_2$ holds on all uninterpreted models.
\end{theorem}
\begin{proof}[Proof]
We focus on satisfiability, since the proof for containment is similar. Let $\conc$ be any element of the domain and $\mathds{F}(\vx)$ be any filter combination that can be constructed with the constants and filter relations in $\phi$. Since the semantics of filter relations has a universal interpretation, $\mathds{F}(\conc)$ is either true on all canonical models, or false on all canonical models. Notice that, by construction of our axiomatisation, the truth value of $\mathds{F}(\conc)$ on all canonical models corresponds to the truth value of $\mathds{F}(\conc)$ on all uninterpreted models of $\axiomatisation{\phi}$. Let $I'$ be an uninterpreted model of $\phi \wedge \axiomatisation{\phi}$, we can construct $I$, canonical model of $\phi$, by (1) changing all the uninterpreted filter relations in $I'$ for their corresponding interpreted ones in $I$ and (2) dropping from $I'$ the interpretation of all the shape relations that occur in $\axiomatisation{\phi}$. Let $I$ be a canonical model of $\phi$, we can construct $I'$, uninterpreted model of $\phi \wedge \axiomatisation{\phi}$, by (1) changing all the interpreted filter relations in $I$ for their corresponding uninterpreted ones in $I'$ and (2) by adding the following interpretation of each shape relation $\hasshape{\vx}{f}$ occurring in $\axiomatisation{\phi}$ to $I$: let $\mathds{F}(\vx)$ be the filter combination such that $\forall \vx . \; \hasshape{\vx}{f} \leftrightarrow \mathds{F}(\vx)$ is one of the conjuncts of $\axiomatisation{\phi}$ (notice that one such conjunct exists for any shape relation), relation $\hasshapePredicateS{f}$ contains all the elements of the domain which satisfy the filter combination $\mathds{F}(\vx)$ on canonical models.
\end{proof}

\subsection{Bounded Axiomatisation}
\label{sec:fltaxm;sub:sucenc}

The main exponential factor in the axiomatisations above is the set of all possible filter combinations. However, we
can limit an axiomatisation to filter combinations having a number of atoms smaller or equal to a constant number,
thus making our axiomatisation polynomial w.r.t.\ an \MSCL\ sentence $\Phi$.
 This new axiomatisation is applicable to all filters considered before, with the exception of \sh{pattern}.
Intuitively, this can be achieved because $\mathds{F}^{\Phi}$ contains several redundant filter combinations.
To illustrate this point, consider \emph{datatype filters} atoms $F^{\hasdatatype = \conc}(\vx)$, derived from the \sh{datatype}
constraint component, that are true if $\vx$ is a literal with datatype $\conc$.\footnote{According to the SPARQL
standard literals with different datatype or language tags are different RDF terms (e.g. literal ``10'' of datatype integer
is not equal to literal ``10'' of datatype float).}
Let $\Phi$ be an \MSCL\ sentence and $\mathds{F}(\vx)$ be a filter combination
$F^{\hasdatatype = \conc}(\vx) \wedge F^{\hasdatatype = \conc'}(\vx)$ of $\mathds{F}^{\Phi}$, where $\conc \not = \conc'$. Since no RDF term can have two different datatypes, the truth
value of $\mathds{F}(\vx)$ is always false (i.e.\ $|\gamma(\mathds{F}(\vx))|=0$). Trivially, any filter
combination in $\mathds{F}^{\Phi}$ whose conjuncts are a proper superset of $\mathds{F}(\vx)$ is also false,
and thus its axiomatisation is not necessary.

In order to limit the size of the filter combinations to a constant number, we reason about each filter type to
determine the maximum number of conjuncts of that type to consider in any filter combination. We call this
number the \emph{maximum non-redundant capacity} (MNRC) of that filter type. Any filter combination that contains
more conjuncts of that type than its MNRC, is necessarely \emph{redundant}.
\begin{definition}
A filter combination $\mathds{F}(\vx)$ is redundant if there exists a filter combination $\mathds{F}'(\vx)$ such
that $\gamma(\mathds{F}(\vx)) = \gamma(\mathds{F}'(\vx))$ and
$\mathds{F}'(\vx)$ is a proper subset of $\mathds{F}(\vx)$.
\end{definition}

We will now define the MNRC for all the monadic SHACL filter types. In the following proofs we will assume that
all conjuncts of a filter combination are syntactically different from each other as any filter combination that
contains multiple copies of the same conjunct is trivially redundant. The MNRC of datatype filters is two.

\begin{lemma} \label{filterlemma:datatype}
Any filter combination $\mathds{F}(\vx)$ that contains more than two datatype filter conjuncts is redundant.
\end{lemma}
\begin{proof}
Since no RDF term can have two datatypes, if $\mathds{F}(\vx)$ contains two positive datatype filter conjuncts,
then $\mathds{F}(\vx)$ is unsatisfiable. Thus $\mathds{F}(\vx)$ cannot contain more than two positive datatype
filter conjuncts without being redundant. Since RDF literals do not need to be annotated with a datatype, any
negation $\neg F^{\hasdatatype = \conc}(\vx)$ of a datatype filter does not affect the truth value of a filter
combination, unless the datatype filter also contains conjunct $F^{\hasdatatype = \conc}(\vx)$, in
which case the filter combination is trivially unsatisfiable. Thus, if $\mathds{F}(\vx)$ is not redundant,
either it does not contain negated datatype filters, or it contains the two filters $F^{\hasdatatype = \conc}(\vx)$
and $\neg F^{\hasdatatype = \conc}(\vx)$ for a constant $\conc$. In this last case, the occurrence of any further
datatype filter in $\mathds{F}(\vx)$ would make the filter combination redundant.
\end{proof}

We represent \emph{language tag} filters, derived from the \sh{languageIn} and \sh{uniqueLang}, with
the $F^{\text{languageTag = \conc}}(\vx)$ filter relation, which is true if $\vx$ is string literal with
language tag $\conc$. Since not all string literals have a language tag, but no string literal has more
than one such tag, this type of filter behaves analogously to the datatype filter. The proof of the following
lemma, which states that the MNRC of language tag filters is two, can be derived from the one above.
\begin{lemma} \label{filterlemma:langtag}
Any filter combination $\mathds{F}(\vx)$ that contains more than two language tag filter conjuncts is redundant.
\end{lemma}

The \emph{order comparison} filters, which are expressible in \SHACL with the \sh{minExclusive},
\sh{maxExclusive}, \sh{minInclusive} and \sh{maxInclusive} constraint components, denote the $x>c$, $x<c$,
$x\ge c$ and $x\le c$ operators, respectively. Order comparison filters have an MNRC of two.
\begin{lemma} \label{filterlemma:order}
Any filter combination $\mathds{F}(\vx)$ that contains more than two order comparison filter conjuncts is redundant.
\end{lemma}
\begin{proof}
If two order comparison filters in $\mathds{F}(\vx)$ are defined over incompatible comparison types
(e.g.\ strings and dates) then $\mathds{F}(\vx)$ is unsatisfiable, and all the other comparison filters
in $\mathds{F}(\vx)$ are redundant. In a set of filters, we define as the \emph{most restrictive} the
one with the smallest number of elements satisfying it, or any such filter if there is more than one. If
all the comparison filters in $\mathds{F}(\vx)$ are defined over the same comparison type, let
$\alpha$ be the most restrictive conjunct in $\mathds{F}(\vx)$ of type $x>c$, $\neg x<c$, $x \ge c$ and
$\neg x\le c$ (or $\top$ if none such conjunct exists), and $\omega$ be the most restrictive conjunct
in $\mathds{F}(\vx)$ of type $\neg x>c$, $ x<c$, $\neg x \ge c$ and $ x\le c$. Trivially, $\mathds{F}(\vx)$
is semantically equivalent to $\mathds{F}'(\vx)$, which is constructed by removing from $\mathds{F}(\vx)$
all comparison filters that are not $\alpha$ or $\omega$.
\end{proof}
String length comparison filters are expressed in \SHACL with the constraint components \sh{minLength} and
\sh{maxLength}, and they behave analogously to the order comparison filters. The proof of the following
lemma, which states that the MNRC of string length comparison filters is two, can be derived from the one above.
\begin{lemma} \label{filterlemma:stringlen}
Any filter combination $\mathds{F}(\vx)$ that contains more than two string length comparison filter conjuncts is redundant.
\end{lemma}

Node kind filters can be represented by three filter relations $\text{F}^{\, \isIRI}(\vx)$,
$\text{F}^{\, \isLiteral}(\vx)$ and $\text{F}^{\, \isBlank}(\vx)$ that are true if $\vx$ is, respectively,
an IRI, a literal or a blank node. Node kind filters have an MNRC of three.
\begin{lemma} \label{filterlemma:nodekind}
Any filter combination $\mathds{F}(\vx)$ that contains more than three node kind filter conjuncts is redundant.
\end{lemma}
\begin{proof}
This lemma can be proven in the same manner as Lemma \ref{filterlemma:datatype}, with the exception that,
since all RDF terms belong to exactly one of the tree node kinds, filter combination
$\neg \text{F}^{\, \isIRI}(\vx) \wedge \neg \text{F}^{\, \isLiteral}(\vx) \wedge \neg \text{F}^{\, \isBlank}(\vx)$
is unsatisfiable and it is not redundant.
\end{proof}

We can establish an MNRC of $1$ for the equality-to-a-constant operator (expressed in SHACL with the
\sh{hasValue} and \sh{in} constraints), by noticing that any variable $\vx$, by the law of excluded middle,
is either interpreted as one of the known constants, or as none of them. In \SCL\ we can express with
$\hasshape{\vx}{\nu}$ the fact that \vx\ is none of the known constants $C$, where $\nu$ is a unique shape
name defined as $\hasshape{\vx}{\nu} \leftrightarrow \bigwedge_{c \in C} \neg \vx = c$.
Intuitively, we consider all possible interactions of the equality operator with filter combinations by
considering whether an element \vx\ is one of the known constants, or whether it conforms to shape
$\hasshape{\vx}{\nu}$. In order to use this new shape $\nu$ in our axiomatisation, we redefine a
\emph{filter combination} $\mathds{F}(\vx)$ as a conjunction of atoms of the form $\vx = \conc$,
$\vx \neg = \conc$, $\hasshape{\vx}{\nu}$, $F(\vx)$ and $\neg F(\vx)$.

\begin{lemma} \label{filterlemma:equality}
Any filter combination $\mathds{F}(\vx)$ that contains more than one equality-to-a-constant
conjuncts is redundant.
\end{lemma}
\begin{proof}
Any filter combination $\mathds{F}(\vx)$ that contains more than one equality-to-a-constant
operator, of which at least one is in positive form, is redundant. In fact, a filter combination is
made redundant by: (a) any two positive equality-to-a-constant operators
$\vx = \conc \wedge \vx = \conc'$, with $\conc \, \not= \conc'$ (recall that we are using the unique name assumption),
which is unsatisfiable by the standard interpretation of the equality operator, and (b) any pair of a
positive and a negative equality-to-a-constant operators $\vx = \conc \wedge \vx \not= \conc'$
because (b.1) if $\conc$ and $\conc'$ are the same constant, then the pair of conjuncts is unsatisfiable
by the standard interpretation of the equality operator and (b.2) if $\conc$ is not the same constant as
$\conc'$ then conjunct $\neg \vx = \conc'$ is redundant.

Moreover, any filter combination $\mathds{F}(\vx)$ that contains equality-to-a-constant operators, but all
negated, is also redundant. Let $D$ be the domain of discourse, $C$ be the set of known constants in the
sentence $\Phi$ from which the filter combinations have been created, and $C^{-}$ the set of constants that
are in the negated equality-to-a-constant operators of $\mathds{F}(\vx)$. The equality-to-a-constant
operators in $\mathds{F}(\vx)$ restricts the domain to elements $D \setminus C^{-}$. Let $\mathds{F}^{*}(\vx)$
be the subset of $\mathds{F}(\vx)$ without equality-to-a-constant conjuncts. We can rewrite $\mathds{F}(\vx)$
into an equivalent set of filter combinations $\bar{\mathds{F}}$ that contain at most one equality-to-a-constant
operator by noticing that we can rewrite $D \setminus C^{-}$ as $( D \setminus C ) \cup ( C \setminus C^{-})$,
and that the left-hand side of this last union of sets corresponds to the elements in the interpretation of
$\hasshape{\vx}{\nu}$, while the right-hand side is a finite set of known constants. The set of filter
combinations $\bar{\mathds{F}}$ that makes $\mathds{F}(\vx)$ redundant is defined as follows:
$\bar{\mathds{F}} = \{\mathds{F}^{*}(\vx) \wedge \hasshape{\vx}{\nu}\} \cup \{\mathds{F}^{*}(\vx) \wedge \vx = \conc | \conc \in C \setminus C^{-}\}$.
Since every element of the domain either belongs to $\hasshape{\vx}{\nu}$ or it is one of the known constants,
the restrictions imposed by $\mathds{F}(\vx)$ and by the set $\bar{\mathds{F}}$ are equivalent.

\end{proof}

The only filter constraint that does not have a maximum non-redundant capacity is \sh{pattern},
since any number of regular expressions can be combined together to generate novel and non-redundant regular expressions.

We define the set of \emph{bounded filter combinations} $\mathds{F}^{'\phi}$ of an \MSCL\ sentence $\phi$
the set of all conjunctions such that (1) the conjuncts are atoms of the form $\vx = \conc$,
$\hasshape{\vx}{\nu}$, $F(\vx)$ or $\neg F(\vx)$, where \conc\ is a constant occurring in $\phi$ and $F$ is
a filter predicate occurring in $\phi$; (2) the number of conjuncts of each filter type, and of equality,
does not exceed its maximum non-redundant capacity.

Notice that in the previous axiomatisation the size of each conjunct depends on the size of the finite
sets computed by the $\gamma$ function. While certain filter constraints, such as \sh{nodeKind},
are either satisfiable by an infinite number of elements, or are unsatisfiable, other constraints
can be satisfied by an arbitrarily large number of elements. We can reduce the size of each conjunct
to a logarithmic factor (with a binary numeric representation) by using counting quantifiers.
This allows us to express the maximum number of elements that can satisfy a filter combination
without explicitly enumerating them.

Given an \MSCL\ sentence $\phi$ and the set $C$ of all known constants in $\phi$, the \emph{bounded}
axiomatisation \boundedaxiomatisation{\phi} of $\phi$ is defined as follows.
\begin{align*}
\boundedaxiomatisation{\phi} =  \left( \hasshape{\vx}{\nu} \leftrightarrow \bigwedge_{c \in C} \neg \vx = c \right)
 \wedge  \bigwedge_{\mathds{F}(\vx) \in \mathds{F}^{'\phi}, |\gamma(\mathds{F}(\vx)) \not = \infty|} \exists^{ \leq \gamma(\mathds{F}(\vx))} \vx . \; \mathds{F}(\vx)
 \end{align*}

By \cref{filterlemma:datatype,filterlemma:langtag,filterlemma:order,filterlemma:stringlen,filterlemma:nodekind,filterlemma:equality},
if $\phi$ does not contain any filter of the \sh{pattern} type, the bounded axiomatisation
only includes filter combinations of up to 12 conjuncts. Thus, the size of the bounded axiomatisation
is polynomial w.r.t.\ $\phi$.

To better explain this second axiomatisation, let us consider again the example of the \MSCL\ sentence
$\phi^{*}$ defined before. The bounded axiomatisation \boundedaxiomatisation{\phi^{*}} of $\phi^{*}$ contains,
among others, the following conjuncts:
\begin{align*}
& \left( \hasshape{\vx}{\nu} \leftrightarrow \vx \not = 2 \wedge  \vx \not = 3 \wedge  \vx \not = \conce \right)
\\ & \wedge \left(  \exists^{ \leq 5} \vx . \; \mathds{F}^{>0}(\vx) \wedge \mathds{F}^{\leq 5}(\vx) \wedge F^{\hasdatatype = xsd:int}(\vx) \right)
\\ & \wedge \left(  \exists^{ \leq 3} \vx . \; \mathds{F}^{>0}(\vx) \wedge \mathds{F}^{\leq 5}(\vx) \wedge F^{\hasdatatype = xsd:int}(\vx) \wedge \hasshape{\vx}{\nu} \right)
\\ & \wedge \left(  \exists^{ \leq 0} \vx . \; \mathds{F}^{>0}(\vx) \wedge \mathds{F}^{\leq 5}(\vx) \wedge F^{\hasdatatype = xsd:int}(\vx) \wedge \vx = \conce \right)
\end{align*}

Of the four elements required by the existentially bounded sub-formula of $\phi^{*}$ to satisfy a filter
combination, only three can belong to $\hasshape{\vx}{\nu}$ (by the third line of the axiomatisation).
The remaining one must satisfy both
$(\vx = 2 \vee  \vx = 3 \vee \vx = \conce)$ and $\vx \not = 2 \wedge \vx \not = 3$, and thus cannot be a
constant other than $\conce$. However, $\conce$ is not compatible with the filter combination
(by the last line of the axiomatisation). Therefore, $\phi^{*}\wedge \boundedaxiomatisation{\phi^{*}}$
is unsatisfiable on an uninterpreted model.

It should be noted that the bounded axiomatisation does not follow the \MSCL\ grammar, while the na\"ive
filter axiomatisation does, albeit not resulting in well-formed sentences. The differences between
our axiomatisations and well-formed \MSCL\ sentences, however, do not affect our decidability and
complexity results presented in the following section since (a), the positive results are applicable to
fragments of first-order logic that are general enough to express our axiomatisations and (b), the negative
results are applicable to SHACL sentences without filters, which therefore do not require an axiomatisation.
For the purposes of the decidability and complexity analysis presented in the following section, the
na\"ive filter axiomatisation is compatible with all of the language fragments, while the bounded filter
axiomatisation is compatible with the fragments that include counting quantifiers.

\begin{theorem} \label{thm:filters}
Given an \MSCL\ sentence $\phi$ and its bounded filter axiomatisation $\boundedaxiomatisation{\phi}$,
sentence $\phi$ is satisfiable on a canonical model iff $\phi \wedge \boundedaxiomatisation{\phi}$
is satisfiable on an uninterpreted model.
Containment $\phi_1 \subseteq \phi_2$ of two \MSCL\ sentences on all canonical models holds iff
$\phi_1 \wedge \boundedaxiomatisation{\phi_1 \wedge \phi_2} \subseteq \phi_2$ holds on all uninterpreted models.
\end{theorem}
\begin{proof}[Proof]
  We focus on satisfiability, since the proof for containment is similar.
  First notice that every canonical model $I$ of $\Phi$ is necessarily a
  model of $\phi \wedge \countingaxiomatisation{\phi}$.
  Indeed, by definition of the function $\gamma$, given a filter combination
  $\mathds{F}(\vx)$, there cannot be more than $|\gamma(\mathds{F}(\vx))|$
  elements satisfying $\mathds{F}(\vx)$, independently of the underlying
  canonical model.
  Thus, $I$ satisfies $\countingaxiomatisation{\phi}$.
  Consider now a model $I$ of $\phi \wedge \countingaxiomatisation{\phi}$ and let
  $I^{\star}$ be the structure obtained from $I$ by replacing the
  interpretations of the monadic filter relations with their canonical ones.
  Obviously, for any filter combination $\mathds{F}(\vx)$, there are exactly
  $|\gamma(\mathds{F}(\vx))|$ elements in $I^{\star}$ satisfying
  $\mathds{F}(\vx)$, since $I^{\star}$ is canonical.
  As a consequence, there exists a injection $\iota$ between the elements
  satisfying $\mathds{F}(\vx)$ in $I$ and those satisfying $\mathds{F}(\vx)$ in
  $I^{\star}$.
  At this point, one can prove that $I^{\star}$ satisfies $\Phi$.
  Indeed, every time a value $x$, satisfying $\mathds{F}(\vx)$ in $I$, is used
  to verify a subformula $\psi$ of $\Phi$ in $I$, one can use the value
  $\iota(x)$ to verify the same subformula $\psi$ in $I^{\star}$.
\end{proof}

\subsection{From Template Satisfiability to \MSCL\ Satisfiability}

As anticipated in the previous section, the problem of \templatesat\ (Def.\ \ref{def:teplatesat})
can be reduced into an \ESCL\ satisfiability problem. In particular, achieving this reduction in the
face of filters requires the additional machinery of the bounded filter axiomatisation. The correspondence
between SHACL \templatesat\ and \ESCL\ sentence satisfiability is given by the following theorem.
The intuition behind this theorem is that, in an uninterpreted model, unknown constant symbols are interchangeable.
Therefore, on an uninterpreted model, considering \templatesat\ for one unknown constant symbol amounts to
considering this problem for all possible constants. Let $Constant(\phi)$ denote the set of constants in $\phi$.

\begin{theorem}\label{thrm:reduction_of_subproblems_to_sat}
The answer to
the \templatesat\ problem for $\vShapeDocument$, $\vshape$ and $\vshapec$ under  brave-total
semantics is \vtrue\ iff there exists a constant symbol $\vf \in Constant(\phi) \cup \{ \conc \}$,
with $\conc$ a fresh constant symbol, such that
$\phi \wedge \boundedaxiomatisation{\phi} \wedge \hasshape{\vf}{\vshape}$ is satisfiable on an uninterpreted model,
where $\phi = \tau(\vShapeDocument \cup \{\trip{\vshape}{\emptyset}{\vshapec}\})$.
\end{theorem}
\begin{proof}
Recall that, by Theorem \ref{thm:filters}, there exists a canonical model $I'$ such that
$I \models \phi \wedge \hasshape{\vn}{\vshape}$ iff there exists an uninterpreted model
$J$ such that $J \models \phi \wedge \hasshape{\vn}{\vshape} \wedge \boundedaxiomatisation{\phi \wedge \hasshape{\vn}{\vshape}}$.

($\Rightarrow$)
Assume that the answer to the \templatesat\ problem for $\vShapeDocument$, $\vshape$ and $\vshapec$ under brave-total
semantics is true. Per Def.~\ref{def:teplatesat} this means that there exists an RDF graph $G$ and a node $n$
such that $G$ is valid w.r.t.\ $\vShapeDocument \cup \{\trip{\vshape}{\vshapet_{\vn}}{\vshapec}\}$. From the translation of target declarations
in Table \ref{tab:targets} it follows that
$\tau(\vShapeDocument \cup \{\trip{\vshape}{\vshapet_{\vn}}{\vshapec}\})$
can be written as
$\phi \wedge \hasshape{\vn}{\vshape}$, where $\phi = \tau(\vShapeDocument \cup \{\trip{\vshape}{\emptyset}{\vshapec}\})$.
Moreover, by Theorem \ref{centralSCL2SHACLCorrespondenceTheorem}, there exists a canonical structure $I$ such that
$I \models \tau(\vShapeDocument \cup \{\trip{\vshape}{\vshapet_{\vn}}{\vshapec}\})$, which means that
$I \models \phi \wedge \hasshape{\vn}{\vshape}$, thanks to our previous observation. Consider the following cases:
(1) $n \in Constant(\phi)$ and (2) $n \not\in Constant(\phi)$

In the first case, let $f$ be $n$. Then there exists an uninterpreted model $J$ such that
$J \models \phi \wedge \hasshape{\vf}{\vshape} \wedge \boundedaxiomatisation{\phi \wedge \hasshape{\vf}{\vshape}}$.
Notice also that the bounded filter axiomatisation of an \MSCL\ sentence $\rho$ depends only on the set of filter relations and the set of constants in $\rho$.
Therefore, if $n \in Constant(\phi)$ then
$\boundedaxiomatisation{\phi \wedge \hasshape{\vf}{\vshape}} = \boundedaxiomatisation{\phi}$. Thus the thesis follows.

In the second case there exists an uninterpreted model $J$ and a constant $n$ such that
$J \models \phi \wedge \hasshape{\vn}{\vshape} \wedge \boundedaxiomatisation{\phi \wedge \hasshape{\vn}{\vshape}}$.
Notice that $\boundedaxiomatisation{\phi \wedge \hasshape{\vn}{\vshape}}$ implies $\boundedaxiomatisation{\phi}$,
since sentence $\phi \wedge \hasshape{\vn}{\vshape}$ contains the same filter relations as $\phi$, and all the
constants of $\phi$ plus one additional constant. The additional constant in $\phi \wedge \hasshape{\vn}{\vshape}$
only results in a stronger axiomatisation that considers more cases. Thus $J \models \phi \wedge \boundedaxiomatisation{\phi}$
and $\hasshapePredicateS{\vshape}$ is not empty in $J$. Let $J^{*}$ be the extension of the uninterpreted model $J$
where constant symbol $f$ is mapped to $\vn$, then
$J^{*} \models \phi \wedge \boundedaxiomatisation{\phi} \wedge \hasshape{f}{\vshape}$ as required by the theorem statement.

($\Leftarrow$) Assume that there exists an uninterpreted model $J$ such that
$J \models \phi \wedge \boundedaxiomatisation{\phi} \wedge \hasshape{f}{\vshape}$.
We distinguish two cases similar to the cases discussed before:
(1) $f \in Constant(\phi)$ and (2) $f \not \in Constant(\phi)$.

In the first case, the thesis can be proven by following the reverse proof of the first case of the previous directionality.
More specifically, $\boundedaxiomatisation{\phi} = \boundedaxiomatisation{\phi \wedge \hasshape{\vf}{\vshape}}$
and thus $J \models \phi \wedge \boundedaxiomatisation{\phi \wedge \hasshape{f}{\vshape}} \wedge \hasshape{f}{\vshape}$.
By Theorem \ref{thm:filters} there exists a canonical model $I$ such that $I \models \phi \wedge \hasshape{f}{\vshape}$.

In case (2), we prove that $J \models \phi \wedge \boundedaxiomatisation{\phi} \wedge \hasshape{f}{\vshape}$ implies
the existence of a value $v$ in the domain of constants such that the uninterpreted model $J[f\mapsto v]$ (obtained by mapping constant symbol $f$ to $v$ in $J$) models
$\phi \wedge \hasshape{f}{\vshape} \wedge \boundedaxiomatisation{\phi \wedge \hasshape{f}{\vshape}}$.
If no such value $v$ exists, then it must follow that
there exist a non-empty filter combination $\mathds{F}$, without equality operators, such that $J \models \mathds{F}(f)$,
but such that $\boundedaxiomatisation{\phi \wedge \hasshape{f}{\vshape}} \rightarrow \forall \vx . \neg \mathds{F}(\vx)$.
Since $\mathds{F}$ does not contain equality operators, and since $\phi \wedge \hasshape{f}{\vshape}$
and $\phi$ contain the same shape relations, it follows that
$\boundedaxiomatisation{\phi} \rightarrow \forall \vx . \neg \mathds{F}(\vx)$, which is in contradiction to the premises.
Intuitively, this is due to the fact that the interpretation of filters is universal, so if a filter combination $\mathds{F}$
is unsatisfiable, it is unsatisfiable in all axiomatisations whose filter relations can express $\mathds{F}$.
Having proven the existence of uninterpreted model $J[f\mapsto v]$, such that
$J[f\mapsto v] \models \phi \wedge \hasshape{f}{\vshape} \wedge \boundedaxiomatisation{\phi \wedge \hasshape{f}{\vshape}}$
the existence of a canonical model $I$ such that $I \models \phi \wedge \hasshape{v}{\vshape}$
easily follows, and thus the thesis is proven. \end{proof}

By this theorem, the positive decidability results that we will present in Sect. \ref{sec:sclsat}
are also applicable to SHACL \templatesat , and the complexity of the corresponding decision procedures
can be considered an upper bound for the complexity of SHACL \templatesat\ in the same fragment, when it
is at least polynomial. This, in turn, allows us to extend our positive results to many of the additional decision
problems discussed in Section \ref{sub:add}.




\section{\SCL Satisfiability}
\label{sec:sclsat}

\begin{figure}[tbh]
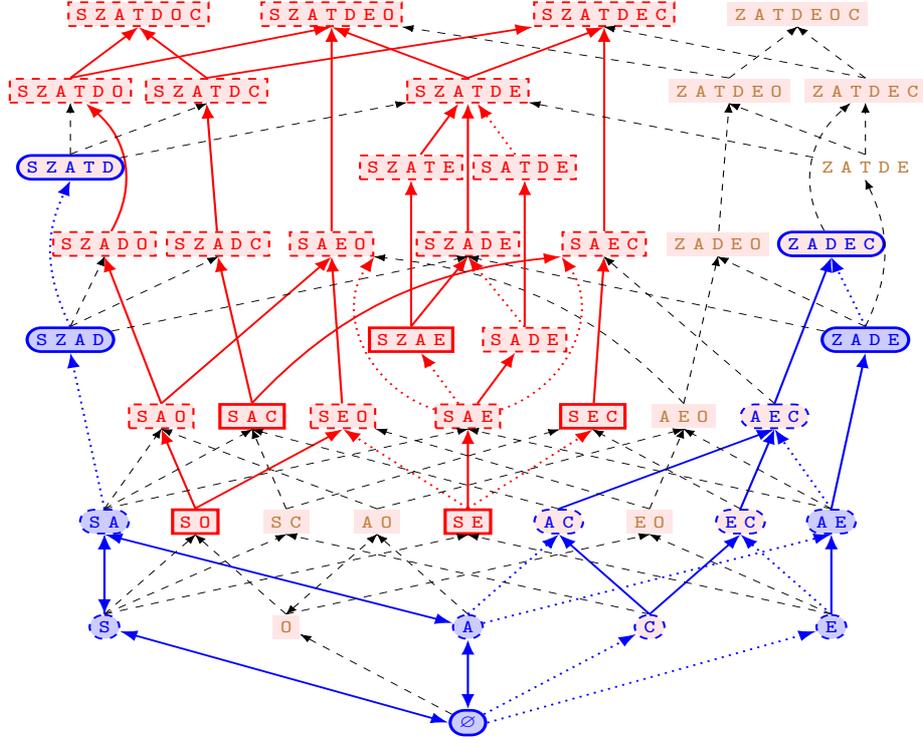

  \begin{center}
    \footnotesize
    \scalebox{1}[0.85]{\figfrg}
  \end{center}
  \caption{\label{fig:frg} Decidability and complexity map of \SCL sentence
    fragments.
    Round (blue) and square (red) nodes denote decidable and undecidable
    fragments, respectively.
    Solid borders on nodes correspond to theorems in this paper, while dashed
    ones are implied results.
    Directed edges indicate inclusion of fragments, while bidirectional ones
    denote polynomial-time reducibility.
    Solid edges are preferred derivations to obtain tight results from a
    complexity viewpoint, while dotted ones leads to worst upper-bounds or just
    model-theoretic properties.
    Finally, a light blue background indicates that the fragment enjoys the
    finite-model property, while those with a light red background do not
    satisfy such a property.
    Nothing is known for the $\Z\A\T\D\E$ fragment reported in the figure.
    Letters denote the components of each given fragment: \emph{sequence} (\S),
    \emph{zero-or-one} (\Z), \emph{alternative} (\A) and \emph{transitive} (\T)
    \emph{paths}; \emph{disjointness} (\D), \emph{order} (\O) and
    \emph{equality} (\E) \emph{of property pairs}; \emph{cardinality
    constraints} (\C).}
\end{figure}

We finally embark on a detailed analysis of the satisfiability problem for
different fragments of \SCL.
Some of the proven and derived results for sentences are visualised in
Figure~\ref{fig:frg}.

The decidability results are proved via embedding in known decidable (extensions
of) fragments of first-order logic, while the undecidability ones are obtained
through reductions from either the \emph{classic domino problem}~\citep{Wan61}
or the \emph{subsumption problem} of constructs called \emph{role-value
maps}~\cite{Sch89} in Description Logic~\cite{BCMNP03}, that are undecidable
even in very restricted forms~\cite{HS04}.

Since we are not considering filters explicitly, but through axiomatisation, the
only interpreted relations are the standard equality and the orders between
elements.

For the sake of clarity and readability, the map depicted in the figure is not
complete \wrt two aspects.
First, it misses few fragments whose decidability can be immediately derived via
inclusion into a more expressive decidable fragment, \eg, \Z\A\D\E\C or
\S\Z\A\T\D.
Second, the remaining missing cases have an open decidability problem.
In particular, while there are several decidable fragments containing the \T
feature, we do not know any decidable fragment with the \O or \O' features.
Notice that the undecidability results exploiting the last two features are only
applicable in the case of generalised RDF.

The letters denoting all \SCL fragments directly correspond to \SHACL\
constraint components, as specified in Table~\ref{tab:components}.
The (un)decidability of an \SCL fragment $\alpha$ translates to the
(un)decidability of our decision problems for the corresponding \SHACL\
fragment, that is, the fragment that excludes the constraint components
identified by the letters not included in $\alpha$.
The results reported in Figure~\ref{fig:frg} show that the decidability of our
decision problems for \SHACL fragments is achieved by the exclusion of either
complex constraint components or complex path expressions.
This is exemplified by the two largest decidable fragments.
The \S\Z\A\T\D\ fragment (shown on the left of the figure), contains all of the
\SHACL path expressions, but it excludes three constraint components, namely
cardinality constraints and both the property pair equality and order
constraints.
The \Z\A\D\E\C\ fragment (shown on the right of the figure), in contrast,
contains all \SHACL\ constraints, with the exception of the property pair
order one, but significantly restricts the path expressions.
In \Z\A\D\E\C, predicate paths (\ie, single relations) can only be combined with
the zero-or-one path expression.
It is also worth noting that the following \SHACL\ features are included in the
base grammar \X, which is decidable: logical constraint components
(\eg, conjunction, disjunction, and negation of constraints); filter constraints;
shape references (potentially recursive); a limited form of the cardinality
constraint that can only express cardinality $\geq 1$.

\subsection{Decidability Results}
\label{sec:sclsat;sub:decrsl}

As a preliminary result, we show that the base language \X is already powerful
enough to express properties writable by combining the \S, \Z, and \A features.
In particular, the last one does not increase in expressive power when the \D
and \O features are also taken in consideration.

\begin{theorem}
  \label{thm:eqv}
  There are
  \begin{inparaenum}[(a)]
    \item\label{thm:eqv(sem)}
      semantic-preserving and
    \item\label{thm:eqv(sat)}
      polynomial-time finite-model-invariant satisfiability-preserving
  \end{inparaenum}
  translations among the following \SCL fragments:
  \begin{inparaenum}
    \item\label{thm:eqv(sza)}
      $\varnothing \equiv \S \equiv \Z \equiv \A \equiv \S\,\Z \equiv \S\A
      \equiv \Z\A \equiv \S\Z\A$;
    \item\label{thm:eqv(ad)}
      $\D \equiv \A\D$;
    \item\label{thm:eqv(ao)}
      $\O \equiv \A\O$;
    \item\label{thm:eqv(ado)}
      $\D\O \equiv \A\D\O$.
  \end{inparaenum}
\end{theorem}
\begin{proof}
  To show the equivalences among the fourteen \SCL fragments mentioned in the
  statement, we consider the following first-order formula equivalences that
  represent few distributive properties enjoyed by the \S, \Z, and \A features
  \wrt some of the other language constructs.
  The verification of their correctness only requires the application of
  standard properties of Boolean connectives and first-order quantifiers.
  \begin{itemize}
    \item
      \textsf{\textbf{[$\S$]}}
      The sequence combination of two path formulae $\piFrm[1]$ and $\piFrm[2]$
      in the body of an existential quantification is removed by nesting two
      plain quantifications, one for each $\piFrm[i]$:
      {\small\[
        \exists \yvarElm \ldotp (\exists \zvarElm \ldotp \piFrm[1](\varElm,
        \zvarElm) \wedge \piFrm[2](\zvarElm, \yvarElm)) \wedge \psiFrm(\yvarElm)
      \equiv
        \exists \zvarElm \ldotp \piFrm[1](\varElm, \zvarElm) \wedge (\exists
        \yvarElm \ldotp \piFrm[2](\zvarElm, \yvarElm) \wedge \psiFrm(\yvarElm)).
      \]}
    \item
      \textsf{\textbf{[$\Z$]}}
      The \Z path construct can be removed from the body of an existential
      quantification on a free variable $\varElm$ by verifying whether the
      formula $\psiFrm$ in its scope is already satisfied by the value bound to
      $\varElm$ itself:
      {\small\[
        \exists \yvarElm \ldotp (\varElm = \yvarElm \vee \piFrm(\varElm,
        \yvarElm)) \wedge \psiFrm(\yvarElm)
      \equiv
        \psiFrm(\varElm) \vee \exists \yvarElm \ldotp \piFrm(\varElm, \yvarElm)
        \wedge \psiFrm(\yvarElm).
      \]}
    \item
      \textsf{\textbf{[$\A$]}}
      The removal of the \A path construct from the body of an existential
      quantifier or of the \D and \O constructs can be done by exploiting the
      following equivalences:
      \[
      \small
      \hspace{-2.5em}
      \begin{aligned}
        {\exists \yvarElm \ldotp (\piFrm[1](\varElm, \yvarElm) \vee
        \piFrm[2](\varElm, \yvarElm)) \wedge \psiFrm(\yvarElm)}
      & \equiv
        {(\exists \yvarElm \ldotp \piFrm[1](\varElm, \yvarElm) \wedge
        \psiFrm(\yvarElm))} \,\vee \\
      & \,\vee
        {(\exists \yvarElm \ldotp \piFrm[2](\varElm, \yvarElm) \wedge
        \psiFrm(\yvarElm))};
      \\
        {\neg \exists \yvarElm \ldotp (\piFrm[1](\varElm, \yvarElm) \vee
        \piFrm[2](\varElm, \yvarElm)) \wedge \RRel(\varElm, \yvarElm)}
      & \equiv
        {(\neg \exists \yvarElm \ldotp \piFrm[1](\varElm, \yvarElm) \wedge
        \RRel(\varElm, \yvarElm))} \,\wedge \\
      & \,\wedge
        {(\neg \exists \yvarElm \ldotp \piFrm[2](\varElm, \yvarElm) \wedge
        \RRel(\varElm, \yvarElm))};
      \\
        {\forall \yvarElm, \zvarElm \ldotp (\piFrm[1](\varElm, \yvarElm) \vee
        \piFrm[2](\varElm, \yvarElm)) \wedge \RRel(\varElm, \zvarElm)
        \rightarrow \sigmaFrm(\yvarElm, \zvarElm)}
      & \equiv
        {(\forall \yvarElm, \zvarElm \ldotp \piFrm[1](\varElm, \yvarElm)
        \wedge \RRel(\varElm, \zvarElm) \rightarrow \sigmaFrm(\yvarElm,
        \zvarElm))} \,\wedge \\
      & \,\wedge
        {(\forall \yvarElm, \zvarElm \ldotp \piFrm[2](\varElm, \yvarElm)
        \wedge \RRel(\varElm, \zvarElm) \rightarrow \sigmaFrm(\yvarElm,
        \zvarElm))}
      \end{aligned}
      \]
  \end{itemize}
  At this point, the equivalences between the fragments naturally follow by an
  iterative application of the reported equivalences used as rewriting rules.
  This clearly concludes the proof of Item~\ref{thm:eqv(sem)}.

  The removal of the \Z and \A constructs from an existential quantification
  might lead, however, to an exponential blow-up in the size of the formula due
  to the duplication of the body $\psiFrm$ of the quantification.
  Therefore, to prove Item~\ref{thm:eqv(sat)}, \ie, to obtain polynomial-time
  finite-model-invariant satisfiability-preserving translations, we first
  construct from the given sentence $\varphiSnt$ a finite-model-invariant
  equisatisfiable sentence $\varphiSnt^{\star}$.
  The latter has size linear in the original one and all the bodies of its
  quantifications are just plain relations.
  Then, we apply the above described semantic-preserving translations to
  $\varphiSnt^{\star}$ that, in the worst case, only leads to a doubling of the
  size.
  The sentence $\varphiSnt^{\star}$ is obtained by iteratively applying to
  $\varphiSnt$ the following two rewriting operations, until no complex formula
  appears in the scope of an existential quantification.
  Let $\psiFrm'(\varElm) = \exists \yvarElm \ldotp \piFrm(\varElm, \yvarElm)
  \wedge \psiFrm(\yvarElm)$ be a subformula, where $\psiFrm(\yvarElm)$ does not
  contain quantifiers other than possibly those of the \S, \D, and \O features.
  Then:
  \begin{inparaenum}[(i)]
    \item
      replace $\psiFrm'(\varElm)$ with $\exists \yvarElm \ldotp \piFrm(\varElm,
      \yvarElm) \wedge \hasShapeElm(\yvarElm)$, where $\hasShapeElm$ is a fresh
      monadic relation;
    \item
      conjoin the resulting sentence with $\forall \varElm \ldotp
      \hasShapeElm(\varElm) \leftrightarrow \psiFrm(\varElm)$.
  \end{inparaenum}
  The two rewriting operations in isolation only lead to a constant increase of
  the size and are applied only a linear number of times.
\end{proof}

It turns out that the base language \X resembles the description logic \ALC
extended with universal roles, inverse roles, and nominals~\citep{BCMNP03}.
This resemblance is effectively exploited as a key observation at the core of
the following result.

\begin{theorem}
  \label{thm:sat(sza)}
  All \SCL subfragments of $\S\Z\A$ enjoy the finite-model property and an
  \ExpTimeC satisfiability problem.
\end{theorem}
\begin{proof}
  The finite-model property follows from the fact that the subsuming \S\Z\A\D
  fragment enjoys the same property, as shown later on in
  Theorem~\ref{thm:sat(szatd)}.

  As far as the satisfiability problem is concerned, thanks to
  Item~\ref{thm:eqv(sza)} of Theorem~\ref{thm:eqv}, we can focus on the base
  fragment \X.

  On the one hand, on the hardness side, one can observe that the description
  logic \ALC extended with inverse roles and nominals (\ALCOI)~\citep{BCMNP03}
  and the fragment \X deprived of the universal quantifications at the level of
  sentences (\ie, the \X subfragment generated by grammar rule $\varphiSnt
  \seteq \top \mid \varphiSnt \wedge \varphiSnt \mid \hasShapeElm(\conElm)$) are
  linearly interreducible.
  Indeed, every existential modality $\exists \RRel.\CElm$ (\emph{resp.},
  $\exists \RRel[][-].\CElm$) can be translated back-and-forth to the \SCL
  construct $\exists \yvarElm \ldotp \RRel(\varElm, \yvarElm) \wedge
  \psiFrm[\CElm](\yvarElm)$  (\emph{resp.}, $\exists \yvarElm \ldotp
  \RRel[][-](\varElm, \yvarElm) \wedge \psiFrm[\CElm](\yvarElm)$), where
  $\psiFrm[\CElm]$ represents the recursive translation of the concept $\CElm$.
  Moreover, every nominal $n$ corresponds to the equality construct $\varElm =
  \conElm[n]$, where a natural bijection between nominals and constant symbols
  is considered.
  At this point, since the aforementioned description logic has an \ExpTimeC
  satisfiability problem~\citep{SCh91,DM00}, it holds that the same problem for
  all subfragments of \S\Z\A is \ExpTimeH.

  On the other hand, completeness follows by observing that the universal
  quantifications at the level of sentences can be encoded in the further
  extension of \ALC with the universal role $\URel$~\citep{SCh91,KSH12,RKH08},
  which has an \ExpTimeC satisfiability problem~\citep{SV01}.
  Indeed, the universal sentences of the form
  \begin{inparaenum}[(a)]
    \item
      $\forall \varElm \ldotp \isASym(\varElm, \conElm) \rightarrow
      \hasShapeElm(\varElm)$,
    \item
      $\forall \varElm, \yvarElm \ldotp \RRel[][\pm](\varElm, \yvarElm)
      \rightarrow \hasShapeElm(\varElm)$,
    \item
      and $\forall \varElm \ldotp \hasShapeElm(\varElm) \leftrightarrow
      \psiFrm(\varElm)$
  \end{inparaenum}
   can be translated, respectively, as follows:
  \begin{inparaenum}[(a)]
    \item
      $n_{\conElm} \wedge \forall \isASym[][-]. \hasShapeElm$, where
      $n_{\conElm}$ is the nominal for the constant $\conElm$;
    \item
      $\forall \URel. \forall \RRel[][\mp]. \hasShapeElm$;
    \item
      $\forall \URel. (\hasShapeElm \leftrightarrow \CElm[\psi])$, where
      $\CElm[\psi]$ is the concept obtained by translating the \X-formula $\psi$
      into \ALCOI.
  \end{inparaenum}
\end{proof}

To derive properties of the \Z\A\D\E fragment, together with its sub-fragments
(two of those -- \E\ and \A\,\E\ -- are included in Figure~\ref{fig:frg}), we
leverage on the syntactic embedding in the \emph{two-variable fragment} of
first-order logic~\citep{Mor75}.

\begin{theorem}
  \label{thm:sat(zade)}
  The $\Z\A\D\E$ fragment of \SCL enjoys the finite-model property and a
  \NExpTime satisfiability problem.
\end{theorem}
\begin{proof}
  Via a syntactic inspection of the \SCL grammar one can observe that, by
  avoiding the \S and \O features of the language, it is only possible to write
  formulae with at most two free variables.
  For this reason, every \Z\A\D\E-formula belongs to the two-variable fragment
  of first-order logic~\citep{Mor75} which is known to enjoy both the
  exponentially-bounded finite-model property and a \NExpTimeC satisfiability
  problem~\citep{GKV97}.
\end{proof}

The embedding in the two-variable fragment used in the previous theorem can be
generalised when the \C feature is added to the picture.
However, the gained additional expressive power does not come without a price,
since the finite-model property is not preserved.

\begin{theorem}
  \label{thm:sat(zadec)}
  The non-recursive $\C$ fragment of \SCL does not enjoy the finite-model
  property (on both sentences and formulae) and has a \NExpTimeH satisfiability
  problem.
  Nevertheless, the finite and unrestricted satisfiability problems for the
  $\Z\A\D\E\C$ fragment are \NExpTime-\Complete.
\end{theorem}
\begin{proof}
  As for the proof of Theorem~\ref{thm:sat(zade)}, one can observe that every
  \Z\A\D\E\C-formula belongs to the two-variable fragment of first-order logic
  extended with counting quantifiers.
  Such a logic does not enjoy the finite-model property~\citep{GOR97}, since it
  syntactically contains a sentence that encodes the existence of an injective
  non-surjective function from the domain of the model to itself.
  The non-recursive \C fragment of \SCL allows us to express a similar property
  via the following sentence $\varphiSnt$, thus proving the first part of the
  statement:
  \[
  \begin{aligned}
    {\varphiSnt}
  & \;{\defeq}\;
    {\isASym(0, \conElm) \wedge \hasShapeElm(0) \wedge \forall \varElm \ldotp
    \hasShapeElm(\varElm) \leftrightarrow \psiFrm[1](\varElm) \wedge \forall
    \varElm \ldotp \isASym(\varElm, \conElm) \rightarrow \psiFrm[2](\varElm)};
  \\
    {\psiFrm[1](\varElm)}
  & \;{\defeq}\;
    {\neg \exists \yvarElm \ldotp \RRel[][-](\varElm, \yvarElm)};
  \\
    {\psiFrm[2](\varElm)}
  & \;{\defeq}\;
    {\exists^{= 1} \yvarElm \ldotp (\RRel(\varElm, \yvarElm) \wedge
    \isASym(\yvarElm, \conElm)) \wedge \neg \exists^{\geq 2} \yvarElm \ldotp
    \RRel[][-](\varElm, \yvarElm)}.
  \end{aligned}
  \]
  Intuitively, the first three conjuncts of $\varphiSnt$ force every model of
  the sentence to contain a distinguished element $0$ that
  \begin{inparaenum}[(i)]
    \item
      does not have any $\RRel$-predecessor and
    \item
      is related to an arbitrary but fixed constant $\conElm$ \wrt $\isASym$.
  \end{inparaenum}
  In other words, $0$ is contained in the domain of the relation $\isASym$, but
  is not contained in the image of the relation $\RRel$.
  Then, the final conjunct of $\varphiSnt$ ensures that every element related
  to $\conElm$ \wrt $\isASym$ has exactly one $\RRel$-successor, also related
  to $\conElm$ in the same way, and at most one $\RRel$-predecessor.
  Thus, a model of $\varphiSnt$ must contain an infinite chain of elements
  pairwise connected by the functional relation $\RRel$.

  It is interesting to observe that the ability to model an infinity axiom is
  already present at the level of constraints, as witnessed by the following
  \C-formula, where the constant $0$ is replaced by the existentially quantified
  variable $\varElm$, where $\psiFrm[1](\varElm)$ and $\psiFrm[2](\varElm)$ are
  the previously introduced formulae with one free variable:
  \[
    \trn{\psiFrm}(\zvarElm) \defeq\, (\zvarElm = \conElm) \wedge \exists \varElm
    \ldotp (\isASym[][-](\zvarElm, \varElm) \wedge \psiFrm[1](\varElm)) \wedge
    \forall \varElm \ldotp \isASym[][-](\zvarElm, \varElm) \rightarrow
    \psiFrm[2](\varElm).
  \]

  By generalising the proof of Theorem~\ref{thm:sat(sza)}, one can notice that
  the \C fragment of \SCL semantically subsumes the description logic \ALC
  extended with inverse roles, nominals, and cardinality
  restrictions (\ALCOIQ)~\citep{BCMNP03}.
  Indeed, every qualified cardinality restriction $(\geq n\, \RRel.\CElm)$
  (\emph{resp.}, $(\leq n\, \RRel.\CElm)$) precisely corresponds to the \SCL
  construct $\exists^{\geq n} \yvarElm \ldotp \RRel(\varElm, \yvarElm) \wedge
  \psiFrm[\CElm](\yvarElm)$ (\emph{resp.}, $\neg \exists^{\geq n + 1} \yvarElm
  \ldotp \RRel(\varElm, \yvarElm) \wedge \psiFrm[\CElm](\yvarElm)$), where
  $\psiFrm[\CElm]$ represents the recursive translation of the concept $\CElm$.
  Thus, the hardness result for \C follows by recalling that the specific
  \ALC language has a \NExpTimeH satisfiability problem~\citep{Tob00,Lut04}.

  On the positive side, however, the extension of the two-variable fragment of
  first-order logic with counting quantifiers has decidable finite and
  unrestricted satisfiability problems.
  Specifically, both can be solved in \NExpTime, even in the case of binary
  encoding of the cardinality constants~\citep{Pra05,Pra10}.
  Hence, the second part of the statement follows as well.
\end{proof}

For the \S\Z\A\D fragment, we obtain model-theoretic and complexity results via
an embedding in the \emph{unary-negation fragment} of first-order
logic~\citep{CS11}.
When the \T feature is considered, the same embedding can be adapted to rewrite
\S\Z\A\T\D into the extension of the mentioned first-order fragment with regular
path expressions~\citep{JLMS18}.
Unfortunately, as for the addition of the \C feature to \Z\A\D\E, we need to pay
the price of losing the finite-model property.

\begin{theorem}
  \label{thm:sat(szatd)}
  The $\S\Z\A\D$ fragment of \SCL enjoys the finite-model property, while the
  non-recursive $\S\T\D$ fragment does not (on both sentences and formulae).
  Nevertheless, the finite and unrestricted satisfiability problems for the
  $\S\Z\A\T\D$ fragment are solvable in 2\ExpTime.
\end{theorem}
\begin{proof}
  By inspecting the \SCL grammar, one can notice that every formula that does
  not make use of the \T, \E, \O, and \C constructs can be translated into the
  standard first-order logic syntax, with conjunctions and disjunctions as
  unique binary Boolean connectives, where negation is only applied to formulae
  with at most one free variable.
  For this reason, every \S\Z\A\D-formula semantically belongs to the
  unary-negation fragment of first-order logic, which is known to enjoy the
  finite-model property~\citep{CS11,CS13}.

  \Mutatismutandis, every \S\Z\A\T\D-formula belongs to the unary-negation
  fragment of first-order logic extended with regular path
  expressions~\citep{JLMS18}.
  Indeed, the grammar rule $\piFrm(\varElm, \yvarElm)$ of \SCL, precisely
  resembles the way the regular path expressions are constructed in the
  considered logic, when one avoids the test construct.
  Unfortunately, as for the two-variable fragment with counting quantifiers,
  this logic also fails to satisfy the finite-model property since it is able to
  encode the existence of a non-terminating path without cycles.
  The non-recursive \S\T\D fragment of \SCL allows us to express the same
  property, as described in the following.
  First of all, consider the \S\T-path-formula $\piFrm(\varElm, \yvarElm) \defeq
  \exists \zvarElm \ldotp (\RRel[][-](\varElm, \zvarElm) \wedge
  (\RRel[][-](\zvarElm, \yvarElm))^{\star})$.
  Obviously, $\piFrm(\varElm, \yvarElm)$ holds between two elements $\varElm$
  and $\yvarElm$ of a model \iff there exists a non-trivial $\RRel$-path (of
  arbitrary positive length) that, starting in $\yvarElm$, leads to $\varElm$.
  Now, by writing the \S\T\D-formula $\psiFrm(\varElm) \defeq \neg \exists
  \yvarElm \ldotp (\piFrm(\varElm, \yvarElm) \wedge \RRel(\varElm, \yvarElm))$,
  we express the fact that an element $\varElm$ does not belong to any
  $\RRel$-cycle since, otherwise, there would be an $\RRel$-successor $\yvarElm$
  able to reach $\varElm$ itself.
  Thus, by ensuring that every element in the model has an $\RRel$-successor,
  but does not belong to any $\RRel$-cycle, we can enforce the existence of an
  infinite $\RRel$-path.
  The non-recursive \S\T\D sentence $\varphiSnt$ expresses exactly this
  property, where $\conElm$ is an arbitrary but fixed constant:
  \[
    \varphiSnt \defeq \isASym(0, \conElm) \wedge \forall \varElm \ldotp
    \isASym(\varElm, \conElm) \rightarrow (\psiFrm(\varElm) \wedge \exists
    \yvarElm \ldotp (\RRel(\varElm, \yvarElm) \wedge \isASym(\yvarElm,
    \conElm))).
  \]

  The same can be stated via the following non-recursive \S\T\D-formula:
  \[
    \trn{\psi}(\zvarElm) \defeq (\zvarElm = \conElm) \wedge \isASym(0, \zvarElm)
    \wedge \forall \varElm \ldotp \isASym[][-1](\zvarElm, \varElm) \rightarrow
    (\psiFrm(\varElm) \wedge \exists \yvarElm \ldotp (\RRel(\varElm, \yvarElm)
    \wedge \isASym(\yvarElm, \conElm))).
  \]

  On the positive side, however, the extension of the unary-negation fragment of
  first-order logic with arbitrary transitive relations or, more generally, with
  regular path expressions has decidable finite and unrestricted satisfiability
  problems.
  Specifically, both can be solved in 2\ExpTime~\citep{ABBV16,JLMS18,DK19}.
\end{proof}

At this point, it is interesting to observe that the \O feature allows us to
express a very weak form of counting restriction which is, however, powerful
enough to describe an infinity axiom.

\begin{theorem}
  \label{thm:fmp(oeop)}
  The non-recursive $\O$ and $\E\O'$ fragments of \SCL do not enjoy the
  finite-model property (on both sentences and formulae).
\end{theorem}
\begin{proof}
  Similarly to the use of the \C construct of \SCL, a simple combination of just
  few instances of the \O feature allows us to write the following sentence
  $\varphiSnt$ encoding the existence of an injective function that is not
  surjective.
  Indeed, a weaker version of the role of the counting quantifier is played here
  by the \O' construct that enforces the functionality of the two relations
  $\RRel$ and $\SRel$.
  Then, by applying both \O' and \O to the inverse of $\RRel$ and $\SRel$, we
  ensure that $\SRel$ is equal to $\RRel[][-]$, which in its turn implies that
  the latter is functional as well.
  Hence, the statement of the theorem immediately follows.
  \[
  \begin{aligned}
    {\varphiSnt}
  & \;{\defeq}\;
    {\isASym(0, \conElm) \wedge \hasShapeElm(0) \wedge \forall \varElm \ldotp
    \hasShapeElm(\varElm) \leftrightarrow \psiFrm[1](\varElm) \wedge \forall
    \varElm \ldotp \isASym(\varElm, \conElm) \rightarrow \psiFrm[2](\varElm)};
  \\
    {\psiFrm[1](\varElm)}
  & \;{\defeq}\;
    {\neg \exists \yvarElm \ldotp \RRel[][-](\varElm, \yvarElm)};
  \\
    {\psiFrm[2](\varElm)}
  & \;{\defeq}\;
    {\exists \yvarElm \ldotp (\RRel(\varElm, \yvarElm) \wedge \isASym(\yvarElm,
    \conElm))} \\
  & \;\wedge\;
    {\forall \yvarElm, \zvarElm \ldotp \RRel(\varElm, \yvarElm) \wedge
    \RRel(\varElm, \zvarElm) \rightarrow \yvarElm \leq \zvarElm \wedge \forall
    \yvarElm, \zvarElm \ldotp \SRel(\varElm, \yvarElm) \wedge \SRel(\varElm,
    \zvarElm) \rightarrow \yvarElm \leq \zvarElm} \\
  & \;\wedge\;
    {\forall \yvarElm, \zvarElm \ldotp \RRel[][-](\varElm, \yvarElm) \wedge
    \SRel(\varElm, \zvarElm) \rightarrow \yvarElm \leq \zvarElm \wedge \forall
    \yvarElm, \zvarElm \ldotp \RRel[][-](\varElm, \yvarElm) \wedge
    \SRel(\varElm, \zvarElm) \rightarrow \yvarElm \geq \zvarElm}.
  \end{aligned}
  \]

  To show that the \E\O' fragment does not enjoy the finite-model property too,
  it is enough to replace the last two applications of the \O' and \O features
  with the \E-formula $\forall \yvarElm \ldotp \RRel[][-](\varElm, \yvarElm)
  \leftrightarrow \SRel(\varElm, \yvarElm)$, which clearly ensures the
  functionality of $\RRel[][-]$, being $\SRel$ functional.

  Notice that also in this case we can express the above property at the level
  of formulae with one free variable, where $\psiFrm[1](\varElm)$ and
  $\psiFrm[2](\varElm)$ are defined as above:
  \[
    \trn{\psiFrm}(\zvarElm) \defeq\, (\zvarElm = \conElm) \wedge \exists \varElm
    \ldotp (\isASym[][-](\zvarElm, \varElm) \wedge \psiFrm[1](\varElm)) \wedge
    \forall \varElm \ldotp \isASym[][-](\zvarElm, \varElm) \rightarrow
    \psiFrm[2](\varElm).
    \qed
  \]
  \renewcommand{\qed}{}
\end{proof}

\subsection{Undecidability Results}
\label{sec:sclsat;sub:undrsl}

In the remaining part of this section, we show the undecidability of the
satisfiability problem for several fragments of \SCL through a semi-conservative
reduction from
\begin{inparaenum}[(1)]
  \item
    the standard \emph{domino problem}~\citep{Wan61,Ber66,Rob71}, whose solution
    is known to be $\Pi_{0}^{1}\text{-complete}$ (see
    Theorems~\ref{thm:sntsat(und)} and~\ref{thm:consat(und)}) and
  \item
    and the \emph{subsumption problem} of \emph{role-value maps}~\cite{Sch89} in
    DL (see Theorem~\ref{thm:se(und)}).
\end{inparaenum}

A $\SetN \times \SetN$ tiling system $\tuple {\TSet} {\HRel} {\VRel}$ is a
structure built on a non-empty set $\TSet$ of domino types, \aka tiles, and two
horizontal and vertical matching relations $\HRel, \VRel \subseteq \TSet \times
\TSet$.
The domino problem asks for a compatible tiling of the first quadrant $\SetN
\times \SetN$ of the discrete plane, \ie, a solution mapping $\tilFun \colon
\SetN \times \SetN \to \TSet$ such that, for all $\varElm, \yvarElm \in \SetN$,
both $(\tilFun(\varElm, \yvarElm), \tilFun(\varElm + 1, \yvarElm)) \in \HRel$
and $(\tilFun(\varElm, \yvarElm), \tilFun(\varElm, \yvarElm + 1)) \in \VRel$
hold true.

\begin{theorem}
  \label{thm:sntsat(und)}

  The sentence satisfiability problems of the non-recursive $\S\O$, $\S\A\C$,
  $\S\E\C$, $\S\E\O'$, and $\S\Z\A\E$ fragments of \SCL are undecidable, even
  with a bounded number ($4$, $3$, $4$, $4$, and $8$, respectively) of binary
  relations.

\end{theorem}
\begin{proof}
  The main idea behind the proof is to embed a tiling system into a model of a
  particular \SCL sentence $\varphiSnt$ that is satisfiable \iff the tiling
  system allows for an admissible tiling.
  The hardest part in the reduction consists in the definition of a satisfiable
  sentence all of whose models homomorphically contain the infinite grid of the
  tiling problem. In other words, this sentence should admit an infinite square
  grid graph as a minor of the model unwinding. Given that, the remaining part
  of the reduction can be carried out in the base language \X.

  Independently of the fragment we choose to prove undecidable, consider the
  following definition:
  \[
    \varphiSnt \defeq \Big( \bigvee_{\tElm \in \TSet} \isASym(0, \tElm) \Big)
    \wedge \Big( \bigwedge_{\tElm \in \TSet} \forall \varElm \ldotp
    \isASym(\varElm, \tElm) \rightarrow (\psiFrm[T][\tElm](\varElm) \wedge
    \psiFrm[G](\varElm)) \Big).
  \]
  Intuitively, the first conjunct ensures the existence of the point $0$, \ie,
  the origin of the grid, labelled by some arbitrary tile in the set $\TSet$.
  Notice that $\TSet$ is lifted to a set of constants in \SCL.
  The second conjunct, then, states that all points $\varElm$, labelled by some
  tile $\tElm$, need to satisfy the properties expressed by the two monadic
  formulae $\psiFrm[T][\tElm](\varElm)$ and $\psiFrm[G](\varElm)$.
  The first one, called \emph{tiling formula}, is used to ensure the
  admissibility of the tiling, while the second one, called \emph{grid formula},
  forces all models of $\varphiSnt$ to necessarily embed a grid.

  \[
  \begin{aligned}
    {\psiFrm[T][\tElm](\varElm) \defeq\;}
      & {\bigwedge_{\tElm' \in \TSet}^{\tElm' \neq \tElm} \neg \isASym(\varElm,
        \tElm')} \\
    \wedge\;
      & {\left( \forall \yvarElm \ldotp \HSym(\varElm, \yvarElm) \rightarrow
        \bigvee_{(\tElm, \tElm') \in \HRel} \isASym(\yvarElm, \tElm') \right)
        \wedge \left( \forall \yvarElm \ldotp \VSym(\varElm, \yvarElm)
        \rightarrow \bigvee_{(\tElm, \tElm') \in \VRel} \isASym(\yvarElm,
        \tElm') \right)}.
  \end{aligned}
  \]
  The first conjunct of the tiling formula $\psiFrm[T][\tElm](\varElm)$ verifies
  that the point associated with the argument $\varElm$ is labelled by no other
  tile than $\tElm$ itself.
  The second part, instead, ensures that the points $\yvarElm$ on the right or
  above of $\varElm$ are labelled by some tile $\tElm'$ which is compatible with
  $\tElm$, \wrt the constraints imposed by the horizontal $\HRel$ and vertical
  $\VRel$ matching relations, respectively.
  Notice here that the relation symbols $\HSym$ and $\VSym$ are the syntactic
  counterpart of $\HRel$ and $\VRel$, respectively.

  At this point, we can focus on the grid formula $\psiFrm[G](\varElm)$ defined
  as follows:
  \[
    \psiFrm[G](\varElm) \defeq \left( \exists \yvarElm \ldotp \HSym(\varElm,
    \yvarElm) \right) \wedge \left( \exists \yvarElm \ldotp \VSym(\varElm,
    \yvarElm) \right) \wedge \gammaFrm(\varElm).
  \]
  The first two conjuncts guarantee the existence of an horizontal and vertical
  adjacent of the point $\varElm$, while the subformula $\gammaFrm(\varElm)$,
  whose definition depends on the considered fragment of \SCL, needs to enforce
  the fact that $\varElm$ is the origin of a square.
  In other words, this means that, going horizontally and then vertically or,
  \viceversa, vertically and then horizontally, the same point is reached.
  To do this, we make use of the two \S-path-formulae
  $\piFrm[\HSym\VSym](\varElm, \yvarElm) \defeq\exists \zvarElm \ldotp
  (\HSym(\varElm, \zvarElm) \wedge \VSym(\zvarElm, \yvarElm))$ and
  $\piFrm[\VSym\HSym](\varElm, \yvarElm) \defeq \exists \zvarElm \ldotp
  (\VSym(\varElm, \zvarElm) \wedge \HSym(\zvarElm, \yvarElm))$.
  In some cases, we also consider the \S\A-path-formula $\piFrm[\DSym](\varElm,
  \yvarElm) \defeq \piFrm[\HSym\VSym](\varElm, \yvarElm) \vee
  \piFrm[\VSym\HSym](\varElm, \yvarElm)$ combining the previous ones, which
  implicitly define a diagonal relation.
  We now proceed by a case analysis on the specific fragment.
  \begin{itemize}
    \item
      \textsf{\textbf{[$\S\O$]}}
      By assuming the existence of a non-empty relation $\DSym$ connecting a
      point with its opposite in the square, \ie, the diagonal point, we can
      express the fact that all points reachable through $\piFrm[\HSym\VSym]$ or
      $\piFrm[\VSym\HSym]$ are, actually, the same unique point:
      \[
      \begin{aligned}
        {\gammaFrm(\varElm) \defeq\;}
          & {\exists \yvarElm \ldotp \DSym(\varElm, \yvarElm)} \\
        \wedge\;
          & {\forall \yvarElm, \zvarElm \ldotp \piFrm[\HSym\VSym](\varElm,
            \yvarElm) \wedge \DSym(\varElm, \zvarElm) \rightarrow \yvarElm \leq
            \zvarElm} \\
        \wedge\;
          & {\forall \yvarElm, \zvarElm \ldotp \piFrm[\HSym\VSym](\varElm,
            \yvarElm) \wedge \DSym(\varElm, \zvarElm) \rightarrow \yvarElm \geq
            \zvarElm} \\
        \wedge\;
          & {\forall \yvarElm, \zvarElm \ldotp \piFrm[\VSym\HSym](\varElm,
            \yvarElm) \wedge \DSym(\varElm, \zvarElm) \rightarrow \yvarElm \leq
            \zvarElm} \\
        \wedge\;
          & {\forall \yvarElm, \zvarElm \ldotp \piFrm[\VSym\HSym](\varElm,
            \yvarElm) \wedge \DSym(\varElm, \zvarElm) \rightarrow \yvarElm \geq
            \zvarElm}.
      \end{aligned}
      \]
      The \S\O-formula $\gammaFrm(\varElm)$ ensures that the relation $\DSym$ is
      both non-empty and functional and that all points reachable via
      $\piFrm[\HSym\VSym]$ or $\piFrm[\VSym\HSym]$ are necessarily the single
      one reachable through $\DSym$.
    \item
      \textsf{\textbf{[$\S\A\C$]}}
      By applying a counting quantifier to the formula $\piFrm[\DSym]$, which
      encodes the union of the points reachable through $\piFrm[\HSym\VSym]$ or
      $\piFrm[\VSym\HSym]$, we can ensure the existence of a single diagonal
      point:
      \[
        \gammaFrm(\varElm) \defeq \neg \exists^{\geq 2} \yvarElm \ldotp
        \piFrm[\DSym](\varElm, \yvarElm).
      \]
    \item
      \textsf{\textbf{[$\S\E\C$]}}
      As for the \S\O fragment, here we use a diagonal relation $\DSym$, which
      needs to contain all and only the points reachable via
      $\piFrm[\HSym\VSym]$ or $\piFrm[\VSym\HSym]$.
      By means of the counting quantifier, we enforce its functionality:
      \[
      \begin{aligned}
        {\gammaFrm(\varElm) \defeq\;}
          & {\neg \exists^{\geq 2} \yvarElm \ldotp \DSym(\varElm, \yvarElm)} \\
        \wedge\;
          & {\forall \yvarElm \ldotp \piFrm[\HSym\VSym](\varElm, \yvarElm)
            \leftrightarrow \DSym(\varElm, \yvarElm)} \\
        \wedge\;
          & {\forall \yvarElm \ldotp \piFrm[\VSym\HSym](\varElm, \yvarElm)
            \leftrightarrow \DSym(\varElm, \yvarElm)}.
      \end{aligned}
      \]
    \item
      \textsf{\textbf{[$\S\E\O'$]}}
      This case is similar to the previous one, where the functionality of
      $\DSym$ is obtained by means of the \O' construct:
      \[
      \begin{aligned}
        {\gammaFrm(\varElm) \defeq\;}
          & {\forall \yvarElm, \zvarElm \ldotp \DSym(\varElm, \yvarElm) \wedge
            \DSym(\varElm, \zvarElm) \rightarrow \yvarElm \leq \zvarElm} \\
        \wedge\;
          & {\forall \yvarElm \ldotp \piFrm[\HSym\VSym](\varElm, \yvarElm)
            \leftrightarrow \DSym(\varElm, \yvarElm)} \\
        \wedge\;
          & {\forall \yvarElm \ldotp \piFrm[\VSym\HSym](\varElm, \yvarElm)
            \leftrightarrow \DSym(\varElm, \yvarElm)}.
      \end{aligned}
      \]
    \item
      \textsf{\textbf{[$\S\Z\A\E$]}}
      The proof for this final case is inspired by the one proposed for the
      undecidability of the guarded fragment extended with transitive closure of
      binary relations~\citep{Gra99a}.
      This time, the functionality of the diagonal relation $\DSym$ is
      indirectly ensured by the conjunction of the four formulae
      $\gammaFrm[1](\varElm)$, $\gammaFrm[2](\varElm)$, $\gammaFrm[3](\varElm)$,
      and $\gammaFrm[4](\varElm)$ that exploit all the features of the fragment:
      \[
      \begin{aligned}
        {\gammaFrm(\varElm) \defeq\;}
          & {\gammaFrm[1](\varElm) \wedge \gammaFrm[2](\varElm) \wedge
            \gammaFrm[3](\varElm) \wedge \gammaFrm[4](\varElm)} \\
        \wedge\;
          & {\forall \yvarElm \ldotp \piFrm[\DSym](\varElm, \yvarElm)
            \leftrightarrow \DSym(\varElm, \yvarElm)},
      \end{aligned}
      \]
      where
      \[
      \begin{aligned}
        {\gammaFrm[1](\varElm) \defeq\;}
          & {\forall \yvarElm \ldotp \left( \bigvee_{i \in \{ 0, 1 \}}
            \DSym[i](\varElm, \yvarElm) \right) \leftrightarrow \DSym(\varElm,
            \yvarElm)}, \\
        {\gammaFrm[2](\varElm) \defeq\;}
          & {\left( \bigvee_{i \in \{ 0, 1 \}} \neg \exists \yvarElm \ldotp
            \DSym[i](\varElm, \yvarElm) \right) \wedge \left( \bigwedge_{i \in
            \{ 0, 1 \}} \forall \yvarElm \ldotp \DSym[i](\varElm, \yvarElm)
            \rightarrow \exists \zvarElm \ldotp \DSym[1 - i](\yvarElm, \zvarElm)
            \right)}, \\
        {\gammaFrm[3](\varElm) \defeq\;}
          & {\bigwedge_{i \in \{ 0, 1 \}} \forall \yvarElm \ldotp \left( \varElm
            = \yvarElm \vee \DSym[i](\varElm, \yvarElm) \vee
            \DSym[i]^{-}(\varElm, \yvarElm) \right) \leftrightarrow
            \ESym[i](\varElm, \yvarElm)}, \text{ and} \\
        {\gammaFrm[4](\varElm) \defeq\;}
          & {\bigwedge_{i \in \{ 0, 1 \}} \forall \yvarElm . (\exists \zvarElm
            \ldotp (\ESym[i](\varElm, \zvarElm) \wedge \ESym[i](\zvarElm,
            \yvarElm))) \leftrightarrow \ESym[i](\varElm, \yvarElm)}.
      \end{aligned}
      \]
      Intuitively, $\gammaFrm[1]$ asserts that $\DSym$ is the union of the two
      accessory relations $\DSym[0]$ and $\DSym[1]$, while $\gammaFrm[2]$
      guarantees that a point can only have adjacents \wrt just one relation
      $\DSym[i]$ and that these adjacents can only appear as first argument of
      the opposite relation $\DSym[1 - i]$.
      In addition, $\gammaFrm[3]$ ensures that the additional relation
      $\ESym[i]$ is the reflexive symmetric closure of $\DSym[i]$ and
      $\gammaFrm[4]$ forces $\ESym[i]$ to be transitive too.

      We can now prove that the relation $\DSym$ is functional.
      Suppose by contradiction that this is not case, \ie, there exist values
      $a$, $b$, and $c$ in the domain of the model of the sentence $\varphiSnt$,
      with $b \neq c$ such that both $\DSym(a, b)$ and $\DSym(a, c)$ hold true.
      By the formula $\gammaFrm[1]$ and the first conjunct of $\gammaFrm[2]$, we
      have that $\DSym[i](a, b)$ and $\DSym[i](a, c)$ hold for exactly one index
      $i \in \{ 0, 1 \}$.
      Thanks to the full $\gammaFrm[2]$, we surely know that $a \neq b$, $a \neq
      c$, and neither $\DSym[i](b, c)$ nor $\DSym[i](c, b)$ can hold.
      Indeed, if $a = b$ then $\DSym[i](a, a)$. This in turn implies $\DSym[1 -
      i](a, d)$ for some value $d$ due to the second conjunct of $\gammaFrm[2]$.
      Hence, there would be pairs with the same first element in both relations,
      trivially violating the first conjunct of $\gammaFrm[2]$.
      Similarly, if $\DSym[i](b, c)$ holds, then $\DSym[1 - i](c, d)$ needs to
      hold as well, for some value $d$, leading again to a contradiction.
      Now, by the formula $\gammaFrm[3]$, both $\ESym[i](b, a)$ and $\ESym[i](a,
      c)$ hold, but $\ESym[i](b, c)$ does not.
      However, this clearly contradicts $\gammaFrm[4]$.
      As a consequence, $\DSym$ is necessarily functional.
  \end{itemize}
  Now, it is not hard to see that the above sentence $\varphiSnt$ (one for each
  fragment) is satisfiable \iff the domino instance on which the reduction is
  based on is solvable.
  Indeed, on the one hand, every compatible tiling $\tilFun \colon \SetN \times
  \SetN \to \TSet$ of a tiling system $\tuple {\TSet} {\HRel} {\VRel}$ induces a
  grid model that trivially satisfies $\varphiSnt$.
  On the other hand, a model of $\varphiSnt$ necessarily embed a grid whose
  points are labelled by tiles satisfying the horizontal and vertical relations.
\end{proof}

\begin{theorem}
  \label{thm:consat(und)}
  The formula satisfiability problems of the non-recursive $\S\T\O$, $\S\A\T\C$,
  $\S\T\E\C$, $\S\T\E\O'$, and $\S\Z\A\T\E$ fragments of \SCL are undecidable,
  even with a bounded (at least $4$, $3$, $4$, $4$, and $8$, respectively)
  number of binary relations.
\end{theorem}
\begin{proof}
  The proof of this theorem builds on top of the one of the previous result, by
  showing that, with the addition of the transitive closure operator, we can
  encode the solution of a domino problem as the existence of a
  constant satisfying the following \SCL formula $\psiFrm(\varElm)$, where the
  relation symbols $\HSym$ and $\VSym$ and the tiling and grid formulae
  $\psiFrm[T][\tElm]$ and $\psiFrm[G]$ are defined as in
  Theorem~\ref{thm:sntsat(und)}:
  \[
  \begin{aligned}
    {\psiFrm(\varElm) \defeq\;}
      & {\left( \bigvee_{\tElm \in \TSet} \isASym(\varElm, \tElm) \right)} \\
    \wedge\;
      & {\forall \yvarElm \ldotp \left( \left( \exists \zvarElm \ldotp
        (\HSym(\varElm, \zvarElm))^{*} \wedge (\VSym(\zvarElm, \yvarElm))^{*}
        \right) \rightarrow \left( \bigwedge_{\tElm \in \TSet} \isASym(\yvarElm,
        \tElm) \rightarrow (\psiFrm[T][\tElm](\yvarElm) \wedge
        \psiFrm[G](\yvarElm)) \right) \right)}.
  \end{aligned}
  \]
  Intuitively, the formula $\psiFrm(\varElm)$ is satisfied by a constant
  $\conElm$ if this element is labelled by a tile in $\TSet$ and every other
  element $\yvarElm$, reachable from $\conElm$ via an arbitrary numbers of
  horizontal steps followed by another arbitrary number of vertical steps,
  satisfies both the tiling and grid formulae.
  Obviously, $\psiFrm(\varElm)$ is satisfied at the root of a grid model induced
  by a compatible tiling $\tilFun \colon \SetN \times \SetN \to \TSet$ of a
  tiling system $\tuple {\TSet} {\HRel} {\VRel}$.
  Indeed, every node in the grid is reachable from the root by following a
  \emph{first-horizontal then-vertical path}.
  Moreover, its labelling is coherent with what is prescribed by the two
  matching relations $\HRel$ and $\VRel$, so, $\psiFrm[T][\tElm](\yvarElm)$
  necessarily holds at every node of the grid.
  \Viceversa, every structure satisfying $\psiFrm(\conElm)$ induces a compatible
  tiling, as the set of elements reachable from $\conElm$ form a grid, due to
  the formula $\psiFrm[G]$, and are suitably labelled thanks to the formula
  $\psiFrm[T][\tElm]$.
\end{proof}

We now prove the undecidability of the non-recursive $\S\E$ fragment of \SCL.
Observe that, even if this statement directly subsumes some of the results
reported in Theorems~\ref{thm:sntsat(und)} and~\ref{thm:consat(und)}, it does so
in a weak way, as the family of sentences defined in the reduction below does
require an unbounded number of binary relations.

\begin{theorem}
  \label{thm:se(und)}
  The satisfiability problem for both sentences and formulae of the
  non-recursive $\S\E$ fragment of \SCL is undecidable.
\end{theorem}
\begin{proof}
  It has been proved that the \emph{subsumption problem} of role-value maps in
  description logic is undecidable~\cite{Sch89}, via a reduction from the
  \emph{word problem} of groups.
  This specific problem can be formalised by means of a
  \emph{constraint-satisfaction problem} as follows, where we consider the $n$
  binary relations $\set{ \RSym[i] }{ i \!\in\! \numcc{1}{n} }$ and the $2(m +
  1)$ binary relations $\set{ \PSym[i], \QSym[i] }{ i \!\in\! \numcc{0}{m} }$ as
  vocabulary.

  \begin{problem}
    Decide whether the set of constraints $\{ \PSym[0] \neq \QSym[0] \} \cup
    \set{ \PSym[i] = \QSym[i] }{ i \in \numcc{1}{m} }$ is satisfiable, under the
    proviso that, for all $i \in \numcc{0}{m}$, it holds that $\PSym[i] =
    \RSym[j_{\PSym[i]}^{1}] \cmp \ldots \cmp \RSym[j_{\PSym[i]}^{k_{\PSym[i]}}]$
    and $\QSym[i] \!=\! \RSym[j_{\QSym[i]}^{1}] \cmp \ldots \cmp
    \RSym[j_{\QSym[i]}^{k_{\QSym[i]}}]$, for some $j_{\PSym[i]}^{1}, \ldots,
    j_{\PSym[i]}^{k_{\PSym[i]}} \!\in\! \numcc{1}{n}$ and $j_{\QSym[i]}^{1},
    \ldots, j_{\QSym[i]}^{k_{\QSym[i]}} \!\in\! \numcc{1}{n}$, with
    $k_{\PSym[i]}, k_{\QSym[i]} \!\in\! \SetN$.
  \end{problem}

  The author of~\cite{Sch89} has shown that the above problem can be encoded in
  \ALC extended with \emph{role composition}, by considering an additional
  binary relation $\RSym$ as technical device.
  By unravelling the \FOL semantics of this encoding, we obtain the reduction of
  the problem to the (un)satisfiability of the formula $\psiFrm(\varElm) \defeq
  \left( \bigwedge_{i = 1}^{n} \deltaFrm[i](\varElm) \right) \wedge \left( \neg
  \psiFrm[0](\varElm) \wedge \bigwedge_{i  = 1}^{m} \psiFrm[i](\varElm) \right)
  \wedge \bigwedge_{i = 0}^{m} \left( \gammaFrm[{\PSym[i]}](\varElm) \!\wedge\!
  \gammaFrm[{\QSym[i]}](\varElm) \right)$, whose components are defined as
  follows, with ${\ZElm \!\in\! \set{ \PSym[i], \QSym[i] \!}{\!i \!\in\!
  \numcc{0}{m} }}$:
  \[
  \begin{aligned}
    {\deltaFrm[i](\varElm) \defeq\;}
      & {\forall \yvarElm \ldotp \left( \left( \exists \zvarElm \ldotp
        \RSym(\varElm, \zvarElm) \wedge \RSym[i](\zvarElm, \yvarElm) \right)
        \leftrightarrow \RSym(\varElm, \yvarElm) \right)}; \\
    {\psiFrm[i](\varElm) \defeq\;}
      & {\forall \yvarElm \ldotp \RSym(\varElm, \yvarElm) \rightarrow \left(
        \forall \zvarElm \ldotp \PSym[i](\yvarElm, \zvarElm) \leftrightarrow
        \QSym[i](\yvarElm, \zvarElm) \right)}; \\
    {\gammaFrm[\ZElm](\varElm) \defeq\;}
      & {\forall \yvarElm \ldotp \RSym(\varElm, \yvarElm) \rightarrow \left(
        \forall \zvarElm \ldotp \gammaFrm[\ZElm](\yvarElm, \zvarElm)
        \leftrightarrow \ZElm(\yvarElm, \zvarElm) \right)}; \\
    {\gammaFrm[\ZElm](\yvarElm, \zvarElm) \defeq\;}
      & {\exists \wvarElm[k_{\ZElm} - 1] \ldotp \left( \ldots \left( \exists
      \wvarElm[1] \ldotp \RSym[j_{\ZElm}^{1}](\yvarElm, \wvarElm[1]) \wedge
      \RSym[j_{\ZElm}^{2}](\wvarElm[1], \wvarElm[2]) \right) \ldots \right)
      \wedge \RSym[j_{\ZElm}^{k_{\ZElm}}](\wvarElm[k_{\ZElm} - 1], \zvarElm)}.
  \end{aligned}
  \]
  Now, it is evident that $\psiFrm[i](\varElm)$ just uses the $\E$ construct,
  while $\deltaFrm[i](\varElm)$ and $\gammaFrm[\ZElm](\varElm)$ exploit both the
  $\S$ and $\E$ constructs.
  No other special construct is applied, thus, the formula $\psiFrm(\varElm)$
  belongs to the non-recursive $\S\E$ fragment of \SCL.

  Intuitively, the conjunct $\left( \neg \psiFrm[0](\varElm) \wedge \bigwedge_{i
  = 1}^{m} \psiFrm[i](\varElm) \right)$ models the set of constants $\{
  \PSym[0] \neq \QSym[0] \} \cup \set{ \PSym[i] = \QSym[i] }{ i \in
  \numcc{1}{m} }$, while $ \bigwedge_{i = 0}^{m} \left(
  \gammaFrm[{\PSym[i]}](\varElm) \!\wedge\! \gammaFrm[{\QSym[i]}](\varElm)
  \right)$ ensures the side conditions $\PSym[i] = \RSym[j_{\PSym[i]}^{1}] \cmp
  \ldots \cmp \RSym[j_{\PSym[i]}^{k_{\PSym[i]}}]$ and $\QSym[i] \!=\!
  \RSym[j_{\QSym[i]}^{1}] \cmp \ldots \cmp \RSym[j_{\QSym[i]}^{k_{\QSym[i]}}]$.
  Finally, the conjunct $\left( \bigwedge_{i = 1}^{n} \deltaFrm[i](\varElm)
  \right)$ is a technical expedient to guarantee the correctness of the
  reduction.

  At this point, Theorem 3.5 and, in particular, Lemma 3.1 of~\cite{Sch89}
  guarantee the undecidability of the class of formulae just described.
  As a consequence, the formula satisfiability problem for the non-recursive
  $\S\E$ fragment is necessarily undecidable.
  The same holds for the sentence satisfiability problem, by considering the
  non-recursive $\S\E$ sentence $\hasShapeElm(\conElm) \wedge \forall \varElm
  \ldotp \hasShapeElm(\varElm) \leftrightarrow \psiFrm(\varElm)$.
\end{proof}




\section{Conclusion}

In this article we have studied the satisfiability and containment problems for
\SHACL documents and shape constraints.
In order to do so, we examined several recursive semantics proposed in the
literature and proved that they all coincide for non-recursive documents.
We also proved that partial assignments semantics reduces to total assignments,
and focused on the latter. 
We then provided a complete translation between:
\begin{inparaenum}[(1)]
  \item
    non-recursive \SHACL and \SCL, a new fragment of first-order logic extended
    with counting quantifiers and transitive closure,
  \item
    recursive \SHACL and \MSCL, an extension of \SCL into a monadic second-order
    logic, where shape names become monadic second-order variables.
\end{inparaenum}
These translations into mathematical logic are effective since, firstly, they
offer a standard framework to model the language, contrary to previous \adhoc
modellings, and, secondly, they allow us to study several formal properties:
from capturing the semantics of filters (that have not been addressed in
literature before), to laying out a detailed map of \SHACL fragments for which
we are able to prove (un)decidability along with complexity results, for our
decision problems.
We also expose semantic properties and asymmetries within \SHACL which might
inform a future update of the W3C language specification.
Although the satisfiability and containment problems are both undecidable for
the full \SHACL, decidability can be achieved by restricting the usage of
certain \SHACL components, such as cardinality restrictions over shape or path
properties.

Nevertheless, the status of some weak fragments of \SHACL, such as \O and \S\C,
as well as the finite-satisfiability problem for the $\S\E$, $\S\O$, $\S\A\C$,
$\S\E\C$, $\S\E\O'$, and $\S\Z\A\E$ fragments remains an open question worthy of
further investigation.




\begin{section}*{Acknowledgements}

G. Konstantinidis was partially supported by The Alan Turing Institute under the EPSRC grant EP/N510129/1, and a
Turing Enhancement Project.
F. Mogavero has been partially supported by the GNCS 2020 project ``Ragionamento
Strategico e Sintesi Automatica di Sistemi Multi-Agente''.
The authors thank the anonymous reviewers for insightful comments that helped
improve the quality of the article, \eg, 
for proposing a proof
approach that directly led to the undecidability result for the \S\E\ fragment.

\end{section}


  \appendix



\section{Translation from SHACL to \SCL}
\label{app:a}

In this section we present our translation $\tau(\vShapeDocument)$ from a SHACL document \vShapeDocument\
(a set of SHACL shape definitions) into sentences of our \SCL\ grammar.
For the sake of completeness, we define our translation $\tau$, for any SHACL document.
However, it should be noted that within SHACL, the same constraint can sometimes be expressed with syntactically different,
but semantically equivalent expressions.
This syntactic detail is not relevant to our analysis of \SHACL, which is focused on semantics. Nevertheless, our formulation of
Theorem \ref{centralSCL2SHACLCorrespondenceTheorem} requires us to work with a ``standardised'' syntactic representation of \SHACL documents.
Thus, we restrict ourselves to \emph{standardised} SHACL documents, that we will define next.
In essence, a standardised SHACL document restricts the usage of these syntactic variations without affecting generality. Given any SHACL document, it is always possible to transform it into a standardised one in linear size and time.
\begin{definition}
A \emph{standardised} \SHACL document is a SHACL document that has the following properties:
\begin{enumerate}
    \item all shape names are identified by IRIs (instead of blank nodes);
    \item it does not contain the following terms: \sh{qualifiedValueShapesDisjoint}, \sh{in}, \sh{class}, \sh{minCount}, \sh{maxCount}, \sh{qualifiedMaxCount}, \sh{and}, \sh{or}, and \sh{xone};
    \item it does not contain triples with a subject that is a property shape, and a predicate that is one of the following: \sh{hasValue}, \sh{datatype}, \sh{nodeKind}, \sh{pattern}, \sh{node}, \sh{property},  \sh{minExclusive}, \sh{minInclusive}, \sh{maxExclusive}, \\ \sh{maxInclusive}, \sh{maxLength}, \sh{minLength} and \sh{not};
    \item triples with \sh{languageIn} as the property contain a list with a single element as the object;
\end{enumerate}
\end{definition}

The translation into \SCL\ grammar of a document \vShapeDocument\ is $\bigwedge_{\vshape \in \vShapeDocument} \tau(\vshape)$, where $\tau(\vshape)$ is the translation of a single SHACL shape $\vshape$ in \vShapeDocument .
Given a shape $\trip{\vshape}{\vshapet}{\vshapec}$, its translation  $\tau(\trip{\vshape}{\vshapet}{\vshapec})$ the following, where $\tau_{\vshapet, \vshape}$ and $\tau_{\vshapec, \vshape}$ are, respectively, the target and constraint axioms of the shape.
\[
\tau(\trip{\vshape}{\vshapet}{\vshapec}) = \tau_{\vshapet, \vshape} \wedge \tau_{\vshapec, \vshape}
\]

The  translation of axiom $\tau_{\vshapet, \vshape}$ is defined in Table \ref{tab:targets} if $\vshapet$ is not empty, or else it is $\top$.
We do not discuss implicit class-based targets, as they just represent a syntactic variant of class targets.
The translation $\tau_{\vshapec, \vshape}$ is the following, where $\tau(\vx, \vshape, \vshapec)$ is the unary formula that models the constraints $\vshapec$ of shape $\vshape$.

\[
\tau_{\vshapec, \vshape}  = \hasshape{\vx}{\sconst}   \leftrightarrow \tau(\vx, \vshape, \vshapec)
\]

In the reminder of this section we define how to compute $\tau(\vx, \vshape, \vshapec)$. The constraints of $\vshapec$ of a shape $\vshape$ in a SHACL document \vShapeDocument , is the set of triples that (1) have $\vshape$ as the subject in the RDF graph representing \vShapeDocument , or (2) define property paths or lists of elements.
As convention, we use \conc\ as an arbitrary constant and \constantList\ as an arbitrary list of constants.
We use \vshape ,  \vshape$^{\prime}$\ and \vshape$^{''}$ as shape names, and \vShapes\  as a list of shape names. Variables are defined as $\vx$, $\vy$ and $\vz$. Arbitrary paths are identified with $\tr$.

The translation of the constraints of a shape $\tau(\vx, \vshape, \vshapec)$ is defined in two cases as follows. The first case deals with the property shapes, which must have exactly one value for the $\sh{path}$ property. The second case deals with node shapes, which cannot have any value for the $\sh{path}$ property. Recall that we use notation \tripsquare{s}{p}{o} to represent an RDF triple with subject $S$, predicate $p$ and object $o$.

\[
\tau(\vx, \vshape, \vshapec) = \top \wedge \begin{cases}
      \bigwedge_{\forall \tripsquare{\vshape}{\vy}{\vz} \in \vshapec} \tau_2(\vx, \tr, \tripsquare{\vshape}{\vy}{\vz}) & \\
      \quad \qquad \text{if} \; \exists r .    \tripsquare{\vshape}{ \sh{path}}{ \tr}) \in \vshapec & \\
      \bigwedge_{\forall \tripsquare{\vshape}{\vy}{\vz} \in \vshapec} \tau_1(\vx, \tripsquare{\vshape}{\vy}{\vz})&\\
      \quad  \qquad \text{otherwise} &
    \end{cases}
\]

Next, we define the translations $\tau_1$ of node shapes triples, $\tau_2$ of property shape triples and $\tau_3$ of property paths.

\subsection{Translation of Node Shape Triples}

The translation of $\tau_1(\vx, \tripsquare{\vshape}{\vy}{\vz})$ is split in the following cases, depending on the predicate of the triple. In case none of those cases are matched $\tau_1(\vx, \tripsquare{\vshape}{\vy}{\vz}) \transeq \top$. The latter ensures that any triple not directly described in the cases below does not alter the truth value of the conjunction in the definition of $\tau(\vx, \vshape)$.

\begin{itemize}
    \item
        $\tau_1(\vx, \tripsquare{\vshape}{ \sh{hasValue}}{ \conc}) \transeq
        \vx = \conc $ .
    \item
        $\tau_1(\vx, \tripsquare{\vshape}{ \sh{in}}{\constantList}) \transeq \bigvee_{\conc \in \constantList}
        \vx = \conc $ .
    \item
        $\tau_1(\vx, \tripsquare{\vshape}{ \sh{class}}{ \conc}) \transeq
        \exists \vy . \isA(\vx, \vy) \wedge \vy = \conc $ .
    \item
        $\tau_1(\vx,\tripsquare{\vshape}{ \sh{datatype}}{ \conc})) \transeq
        \text{F}^{\, \hasdatatype = \conc}(\vx) $ .
    \item
        $\tau_1(\vx, \tripsquare{\vshape}{ \sh{nodeKind}}{ \conc}) \transeq
        \text{F}^{\, \isIRI}(\vx) $ if $c = $\sh{IRI};
        $\text{F}^{\, \isLiteral}(\vx) $ if \\$c = $\sh{Literal};
        $\text{F}^{\, \isBlank}(\vx) $ if $c = $\sh{BlankNode}. The translations for a $\conc$ that equals\\ \sh{BlankNodeOrIRI}, \sh{BlankNodeOrLiteral} or \sh{IRIOrLiteral} are trivially constructed by a conjunction of two of these three filters.
    \item
        $\tau_1(\vx, \tripsquare{\vshape}{ \sh{minExclusive}}{ \conc}) \transeq \vx > \conc
        $ if order is an interpreted relation, else $\text{F}^{\, > \conc}(\vx)$.
    \item
        $\tau_1(\vx, \tripsquare{\vshape}{ \sh{minInclusive}}{ \conc}) \transeq \vx \geq \conc
         $ if order is an interpreted relation, else $\text{F}^{\, \geq \conc}(\vx)$.
    \item
        $\tau_1(\vx, \tripsquare{\vshape}{ \sh{maxExclusive}}{ \conc}) \transeq  \vx < \conc
         $ if order is an interpreted relation, else $\text{F}^{\, < \conc}(\vx)$.
    \item
        $\tau_1(\vx, \tripsquare{\vshape}{ \sh{maxInclusive}}{ \conc}) \transeq \vx \leq \conc
         $ if order is an interpreted relation, else $\text{F}^{\, \leq \conc}(\vx)$.
    \item
        $\tau_1(\vx, \tripsquare{\vshape}{ \sh{maxLength}}{ \conc}) \transeq $  $ \text{F}^{\, \text{maxLength}=\conc}(\vx) $ .
    \item
        $\tau_1(\vx, \tripsquare{\vshape}{ \sh{minLength}}{ \conc}) \transeq $  $ \text{F}^{\, \text{minLength}=\conc}(\vx) $ .
    \item
        $\tau_1(\vx, \tripsquare{\vshape}{ \sh{pattern}}{ \conc}) \transeq $  $
        \text{F}^{\, \text{pattern}=\conc}(\vx) $ .
    \item
        $\tau_1(\vx, \tripsquare{\vshape}{ \sh{languageIn}}{ \constantList}) \transeq $  $  \bigvee_{\conc \in \constantList} F^{\text{languageTag = \conc}}(\vx)$ .
    \item
        $\tau_1(\vx, \tripsquare{\vshape}{ \sh{not}}{ \vshape^{\prime}}) \transeq $  $ \neg \hasshape{\vx}{\vshape^{\prime}}$ .
    \item
        $\tau_1(\vx, \tripsquare{\vshape}{ \sh{and}}{ \vShapes}) \transeq  $  $
        \bigwedge_{\vshape^{\prime} \in \vShapes} \hasshape{\vx}{\vshape^{\prime}}$ .
    \item
        $\tau_1(\vx, \tripsquare{\vshape}{ \sh{or}}{ \vShapes}) \transeq  $  $
        \bigvee_{\vshape^{\prime} \in \vShapes} \hasshape{\vx}{\vshape^{\prime}}$ .
    \item
        $\tau_1(\vx, \tripsquare{\vshape}{ \sh{node}}{ \vshape^{\prime}}) \transeq  $  $
        \hasshape{\vx}{\vshape^{\prime}}$ .
    \item
        $\tau_1(\vx, \tripsquare{\vshape}{ \sh{property}}{ \vshape^{\prime}}) \transeq  $  $
        \hasshape{\vx}{\vshape^{\prime}}$ .
\end{itemize}

\subsection{Translation of Property Shapes}

The translation of $\tau_2(\vx, \tr, \tripsquare{\vshape}{\vy}{\vz})$ is split in the following cases, depending on the predicate of the triple. In case none of those cases are matched $\tau_2(\vx, \tr, \tripsquare{\vshape}{\vy}{\vz}) \transeq \top$.

\begin{itemize}
    \item
         $\tau_2(\vx, \tr, \tripsquare{\vshape}{ \sh{hasValue}}{ \conc}) \transeq $  $
         \exists \vy . r(\vx, \vy) \wedge \tau_1(\vy, \tripsquare{\vshape}{ \sh{hasValue}}{ \conc}) $
    \item
        $\tau_2(\vx, \tr, \tripsquare{\vshape}{ p}{ \conc}) \transeq
        \forall \vy . \tau_3(\vx, \tr,\vy)) \rightarrow \tau_1(\vy, \tripsquare{\vshape}{ p}{ \conc}) $, if $p$ equal to one of the following: \sh{class}, \sh{datatype}, \sh{nodeKind}, \sh{minExclusive},  \sh{minInclusive}, \sh{maxExclusive}, \sh{maxInclusive}, \sh{maxLength}, \\ \sh{minLength}, \sh{pattern}, \sh{not}, \sh{and}, \sh{or}, \sh{xone},  \sh{node}, \\ \sh{property}, \sh{in} .
    \item
        $\tau_2(\vx, \tr, \tripsquare{\vshape}{ \sh{languageIn}}{ \constantList}) \transeq $ \\ $\null\qquad
        \forall \vy . \tau_3(\vx, \tr,\vy)) \rightarrow  \tau_1(\vy, \tripsquare{\vshape}{ \sh{languageIn}}{ \constantList}) $ .
    \item
        $\tau_2(\vx, \tr, \tripsquare{\vshape}{ \sh{uniqueLang}}{ \text{true}}) \transeq  \bigwedge_{\conc \in L} \neg \exists^{\geq 2} \vy . r(\vx, \vy) \wedge F^{\text{lang}=c}(\vy) $ \\  where $L = \{\conc | \conc \in \constantList \wedge \exists \vshape^{'} .  \tripsquare{\vshape^{'}}{\sh{languageIn}}{\constantList} \in \vShapeDocument\} \cup \{c^{\texttt{uniqueL}}\}$ and $c^{\texttt{uniqueL}}$ is a fresh unique constant. This translation is possible because \sh{languageIn} is the only constraint that can force language tags constraints on literals.%
    \item
        $\tau_2(\vx, \tr, \tripsquare{\vshape}{ \sh{minCount}}{\conc}) \transeq
        \exists^{\geq \conc} \vy . \tau_3(\vx, \tr,\vy)) $ .
    \item
        $\tau_2(\vx, \tr, \tripsquare{\vshape}{ \sh{maxCount}}{ \conc}) \transeq \neg
        \exists^{\leq \conc} \vy . \tau_3(\vx, \tr,\vy)) $ .
    \item
        $\tau_2(\vx, \tr, \tripsquare{\vshape}{ \sh{equals}}{ \conc}) \transeq  \forall \vy .  \tau_3(\vx, \tr,\vy) \leftrightarrow \tau_3(\vx, \conc,\vy)$ .
    \item
        $\tau_2(\vx, \tr, \tripsquare{\vshape}{ \sh{disjoint}}{ \conc}) \transeq   \neg \exists \vy .  \tau_3(\vx, \tr,\vy) \wedge \tau_3(\vx, \conc,\vy)$ .
    \item
        $\tau_2(\vx, \tr, \tripsquare{\vshape}{ \sh{lessThan}}{ \conc}) \transeq
        \forall \vy, \vz . \tau_3(\vx, \tr,\vy) \wedge \tau_3(\vx, \conc,\vz) \rightarrow \vy < \vz $ .
    \item
        $\tau_2(\vx, \tr, \tripsquare{\vshape}{ \sh{lessThanOrEquals}}{ \conc}) \transeq
        \forall \vy, \vz \,.\, \tau_3(\vx, \tr,\vy) \wedge \tau_3(\vx, \conc,\vz) \rightarrow \vy \leq \vz $ .
    \item
        $\tau_2(\vx, \tr, \tripsquare{\vshape}{ \sh{qualifiedValueShape}}{ \vshape^{\prime}}) \transeq \alpha \wedge \beta
        $ , where $\alpha$ and $\beta$ are defined as follows.
        Let $S^{'}$ be the set of \emph{sibling shapes} of $\vshape$ if $\vShapeDocument$ contains
         $\tripsquare{\vshape}{  \sh{qualifiedValueShapesDisjoint}}{ \text{true}}$, or the empty set otherwise. \\ Let $\nu(\vx) = \hasshape{\vx}{\vshape^{\prime}} \bigwedge_{\forall \vshape^{''} \in S^{'}} \neg \hasshape{\vx}{\vshape^{''}}$. If triple  $\tripsquare{\vshape}{ \sh{qualifiedMinCount}}{ \conc}$ is contained in $\vShapeDocument$, then $\alpha$ is equal to $\exists^{\geq \conc} \vy . \tau_3(\vx, \tr,\vy) \wedge \nu(\vx)$, otherwise $\alpha$ is equal to $\top$. If $\vShapeDocument$ contains the triple  $\tripsquare{\vshape}{ \sh{qualifiedMaxCount}}{ \conc}$, then $\beta$ is equal to $\neg \exists^{\leq \conc} \vy . \tau_3(\vx, \tr,\vy) \wedge \nu(\vx)$, otherwise $\beta$ is equal to $\top$.
    \item
        $\tau_2(\vx, \tr, \tripsquare{\vshape}{ \sh{closed}}{ \text{true}}) \transeq  \bigwedge_{\forall R \in \Theta} \neg \exists_{} \vy . R(\vx, \vy)$ if $\Theta$ is not empty, or else $\top$, where $\Theta$ is defined as follows. Let $\Theta^{\text{all}}$ be the set of all relation names in $\vShapeDocument$, namely $\Theta^{\text{all}} = \{R | \tripsquare{\vx}{R}{\vy} \in \vShapeDocument \}$. If this \FOL translation is used to compare multiple SHACL documents, such in the case of deciding containment, then $\Theta^{\text{all}}$ must be extended to contain all the relation names in all these SHACL documents. Let $\Theta^{\text{declared}}$ be the set of all the binary property names $\Theta^{\text{declared}} = \{R | \{\tripsquare{\vshape}{ \sh{property}}{ \vx} \wedge \tripsquare{\vx}{\sh{path}}{R)}\} \subseteq \vShapeDocument \}$. Let $R^{\texttt{closed}}$ be a unique fresh relation name, $\Theta^{\text{ignored}}$ be the set of all the binary property names declared as ``ignored'' properties, namely $\Theta^{\text{ignored}} = \{R | R \in \bar{R} \wedge \allowbreak \tripsquare{\vshape}{  \sh{ignoredProperties}}{ \bar{R}} \in \vShapeDocument\}$, where $\bar{R}$ is a list of IRIs. The set $\Theta$ can now be defined as $\Theta = (\Theta^{\text{all}}\cup \{R^{\texttt{closed}}\}) \setminus (\Theta^{\text{declared}} \cup \Theta^{\text{ignored}})$.

\end{itemize}

\subsection{Translation of Property Paths}

The translation $\tau_3(\vx, \tr,\vy))$ of any SHACL path $\tr$ is given by the following cases. For simplicity, we will assume that all property paths have been translated into an equivalent form having only simple IRIs within the scope of the inverse operator. Using SPARQL syntax for brevity, where the inverse operator is identified by the hat symbol $\hat \ $, the sequence path $\hat \ (r_1 / r_2)$ can be simplified into $\hat \ r_2 / \hat r_1$; an alternate path $\hat \ (r_1 \mid r_2)$ can be simplified into $\hat \ r_2 \mid \hat \ r_1$. We can simplify in a similar way zero-or-more, one-or-more and zero-or-one paths $\hat \ (r^{* / + / ?})$ into $(\hat \ r)^{* / + / ?}$.

\begin{itemize}
    \item
         If $\tr$ is an IRI $R$, then $\tau_3(\vx, \tr,\vy)) \transeq R(\vx, \vy)$
    \item
        If $\tr$ is an inverse path, with $\tr = $ ``\texttt{[ sh:inversePath $R$ ]}'', then \\ $\tau_3(\vx, \tr,\vy)) \transeq
        R^{-}(\vx, \vy) $
    \item
        If $\tr$ is a conjunction of paths, with $\tr = $ ``\texttt{( $\tr_1$, $\tr_2$, ..., $\tr_n$ )}'', then $\tau_3(\vx, \tr,\vy)) \transeq
        \exists \vz_1, \vz_2, ..., \vz_{n-1} . \; \tau_3(\vx, \tr_1,\vz_1)) \; \wedge \; \tau_3(\vz_1, \tr_2,\vz_2)) \; \wedge \; ... \\ \wedge \; \tau_3(\vz_{n-1}, \tr_2,\vy)) $
    \item
        If $\tr$ is a disjunction of paths, with $\tr = $ ``\texttt{[ sh:alternativePath ( $\tr_1$, $\tr_2$, ..., $\tr_n$ ) ]}'', then $\tau_3(\vx, \tr,\vy)) \transeq
        \tau_3(\vx, \tr_1,\vy)) \vee \tau_3(\vx, \tr_2,\vy)) \vee ... \vee \tau_3(\vx, \tr_n,\vy)) $
    \item
        If $\tr$ is a zero-or-more path, with $\tr = $ ``\texttt{[ \sh{zeroOrMorePath} $\tr_1$]}'', then $\tau_3(\vx, \tr,\vy)) \\ \transeq
        (\tau_3(\vx, \tr_1,\vy)))^{*}$
    \item
        If $\tr$ is a one-or-more path, with $\tr = $ ``\texttt{[ \sh{oneOrMorePath} $\tr_1$]}'', then $\tau_3(\vx, \tr,\vy)) \transeq \exists \vz . \tau_3(\vx, \tr_1,\vz)) \wedge (\tau_3(\vz, \tr_1,\vy)))^{*}$
    \item
        If $\tr$ is a zero-or-one path, with $\tr = $ ``\texttt{[ \sh{zeroOrOnePath} $\tr_1$]}'', then $\tau_3(\vx, \tr,\vy)) \transeq \vx = \vy \vee \tau_3(\vx, \tr_1,\vy))$
\end{itemize}




\section{Translation from \SCL\ to SHACL}
\label{app:b}

In this section we present the translation $\tau^{-}$, inverse of $\tau$, from sentences in the \SCL\ grammar into SHACL documents. We begin by defining the translation of the property path subgrammar $r(\vx, \vy)$ into SHACL property paths:
\begin{itemize}
    \item $\tau^{-}(R) \transeq R$
    
    \item $\tau^{-}(R^{-}) \transeq$\texttt{ [ \sh{inversePath} $R$ ] }
    
    \item $\tau^{-}(r^{\star}(\vx, \vy)) \transeq$\texttt{ [ \sh{zeroOrMorePath} $\tau^{-}(r(\vx, \vy))$ ] } 
    
    \item $\tau^{-}(\vx = \vy \vee r(\vx, \vy)) \transeq$\texttt{ [ \sh{zeroOrOnePath} $\tau^{-}(r(\vx, \vy))$ ] } %
    
    \item $\tau^{-}(r_1(\vx, \vy) \vee r_2(\vx, \vy)) \transeq$\texttt{ [ \sh{alternativePath} ( $\tau^{-}(r_1(\vx, \vy))$, \\ $\tau^{-}(r_2(\vx, \vy))$ ] ) } 
    
    \item $\tau^{-}(r_1(\vx, \vy) \wedge r_2(\vx, \vy)) \transeq$\texttt{ ( $\tau^{-}(r_1(\vx, \vy))$, $\tau^{-}(r_2(\vx, \vy))$ ) }
\end{itemize}

The translation of the constraint  subgrammar $\psi(\vx)$ is the following. we will use $\tau^{-}(\psi(\vx))$ to denote the SHACL translation of shape $\psi(\vx)$, and $\iota(\tau^{-}(\psi(\vx)))$ to denote the IRI corresponding to its shape name. To improve legibility, we omit set brackets around sets of RDF triples, and we represent them in Turtle syntax. For example, a set of RDF triples such as ``\texttt{$s$ a \sh{NodeShape} ; \sh{hasValue} \conc\ .} '' is to be interpreted as the set $\{\trip{s}{\rdf{type}}{\sh{NodeShape}}, \trip{s}{\sh{hasValue}}{c}\}$. When alternative translations are possible, the one listed first takes precedence. In other words, a translation in the following list is applied to a formula only if no translation listed before it are applicable.

\begin{itemize}
    \item $\tau^{-}(\top) \transeq  $ \\ \texttt{$s$ a \sh{NodeShape} . } 
    
    \item $\tau^{-}(\vx = \conc) \transeq  $ \\ \texttt{$s$ a \sh{NodeShape} ; \\
    \phantom{x} \sh{hasValue} \conc\ . 
    } 
    
    \item $ \tau^{-} \left( \bigwedge_{\conc \in L} \neg \exists^{\geq 2} \vy . r(\vx, \vy) \wedge F^{\text{lang}=c}(\vy) \right)$,  where $c^{\texttt{uniqueL}} \in L$ \\
    \texttt{$s$ a \sh{PropertyShape} ; \\
    \phantom{x} \sh{path} r ; \\
    \phantom{x} \sh{uniqueLang} true .} 
    
    \item $\tau^{-}(F(\vx)) \transeq $ \\ \texttt{$s$ a \sh{NodeShape} ; \\
    \phantom{x} $f$ $C$ . 
    } \\ Predicate $f$ and the RDF term $C$ is the filter function identified by $F$, namely one of the following: \sh{datatype}, \sh{nodeKind}, \sh{minExclusive}, \sh{minInclusive},  \sh{maxExclusive}, \sh{maxInclusive}, \sh{maxLength}, \\ \sh{minLength}, \sh{pattern}, \sh{languageIn}. Depending on the type of the filter, $C$ could be a literal, an IRI, or an RDF list with a single element.
    
    \item $\tau^{-}(\hasshape{\vx}{\vshape^{\prime}}) \transeq  $ \\ \texttt{$s$ a \sh{NodeShape} ; \\
    \phantom{x} \sh{node} $\vshape^{\prime}$ . 
    } \\ if  $\vshape^{\prime}$ is a node shape, else:\\
    \texttt{$s$ a \sh{NodeShape} ; \\
    \phantom{x} \sh{property} $\vshape^{\prime}$ . 
    }
    
    \item $\tau^{-}\left( \bigwedge_{\forall R \in \Theta} \neg \exists_{} \vy . R(\vx, \vy) \right)$ where $\Theta$ is a set of property relation names that includes $R^{\texttt{closed}}$, and $\Theta^{\texttt{list}}$ is the RDF list representation of all the property relation names in the \SCL\ formula that are not in $\Theta$  \\
    \texttt{$s$ a \sh{PropertyShape} ; \\ 
    \phantom{x} \sh{close} true ; \\
    \phantom{x} \sh{ignored} $\Theta^{\texttt{list}}$. } 
    
    \item $\tau^{-}(\neg \psi(\vx)) \transeq  $ \\ \texttt{$s$ a \sh{NodeShape} ; \\
    \phantom{x} \sh{not} $\iota(\tau^{-}(\psi(\vx)))$ . } 
    
    \item $\tau^{-}(\psi_1(\vx) \wedge \psi_2(\vx)) \transeq  $ \\ \texttt{$s$ a \sh{NodeShape} ; \\
    \phantom{x} \sh{and} ($\iota(\tau^{-}(\psi_1(\vx)))$, $\iota(\tau^{-}(\psi_2(\vx)))$) . }
    
    \item $\tau^{-}(\exists^{\geq n} \vy . r(\vx, \vy) \wedge \psi(\vy)) \transeq  $ \\ \texttt{$s$ a \sh{PropertyShape} ; \\
    \phantom{x} \sh{path} $\tau^{-}(r(\vx, \vy))$ ; \\ 
    \phantom{x} \sh{qualifiedValueShape} $\iota(\tau^{-}(\psi(\vy)))$ ; \\
    \phantom{x} \sh{qualifiedMinCount} $n$  .   
    }
    
    \item $\tau^{-}( \forall \vy . r(\vx, \vy) \leftrightarrow R(\vx, \vy) ) \transeq  $ \\ \texttt{$s$ a \sh{PropertyShape} ; \\
    \phantom{x} \sh{path} $\tau^{-}(r(\vx, \vy))$ ; \\ 
    \phantom{x} \sh{equals} $R$  .   
    }
    
    \item $\tau^{-}( \neg \exists \vy . r(\vx, \vy) \wedge R(\vx, \vy) ) \transeq  $ \\ \texttt{$s$ a \sh{PropertyShape} ; \\
    \phantom{x} \sh{path} $\tau^{-}(r(\vx, \vy))$ ; \\ 
    \phantom{x} \sh{disjoint} $R$ .   
    }    
    
    \item $\tau^{-}( \forall \vy, \vz . r(\vx, \vy) \wedge R(\vx, \vz) \rightarrow \vy < \vz ) \transeq  $ \\ \texttt{$s$ a \sh{PropertyShape} ; \\
    \phantom{x} \sh{path} $\tau^{-}(r(\vx, \vy))$ ; \\ 
    \phantom{x} \sh{lessThan} $R$ .   
    }   

    \item $\tau^{-}( \forall \vy, \vz . r(\vx, \vy) \wedge R(\vx, \vz) \rightarrow \vy \leq \vz ) \transeq  $ \\ \texttt{$s$ a \sh{PropertyShape} ; \\
    \phantom{x} \sh{path} $\tau^{-}(r(\vx, \vy))$ ; \\ 
    \phantom{x} \sh{lessThanOrEquals} $R$ .   
    }   
    
\end{itemize}

We now define the translation $\tau^{-}(\varphi)$ of a complete sentence of the $\varphi$-grammar into a SHACL document \vShapeDocument\ as follows.

\begin{itemize}
    \item $\tau^{-}(\varphi_1 \wedge \varphi_2) \transeq  \tau^{-}(\varphi_1) \cup \tau^{-}(\varphi_2)) $ 
    
    \item $\tau^{-}(\hasshape{\conc}{\sconst}) \transeq   $  \texttt{$s$  \sh{targetNode} \conc  . 
    }
    
    \item $\tau^{-}(\forall \vx . \; \isA(\vx, \conc) \rightarrow \hasshape{\vx}{\sconst}) \transeq   $  \texttt{$s$ \sh{targetClass} \conc .
    } 

    \item $\tau^{-}(\forall \vx, \vy . \; R(\vx, \vy) \rightarrow \hasshape{\vx}{\sconst}) \transeq   $  \texttt{$s$  \sh{targetSubjectsOf} $R$  .
    }     

    \item $\tau^{-}(\forall \vx, \vy . \; R^{-}(\vx, \vy) \rightarrow \hasshape{\vx}{\sconst}) \transeq   $  \texttt{$s$ \sh{targetObjectsOf} $R$  .
    }     
    
    \item $\tau^{-} \big(  \hasshape{\vx}{\sconst}  \leftrightarrow \psi(\vx)  \big) \transeq  \tau^{-}(\psi_1(\vx)) \;  $  
    such that $\sconst = \iota(\tau^{-}(\psi_1(\vx)))$ .

\end{itemize}


  \bibliographystyle{elsarticle-harv}
  \bibliography{References}

\begin{thebibliography}{57}
\expandafter\ifx\csname natexlab\endcsname\relax\def\natexlab#1{#1}\fi
\providecommand{\url}[1]{\texttt{#1}}
\providecommand{\href}[2]{#2}
\providecommand{\path}[1]{#1}
\providecommand{\DOIprefix}{doi:}
\providecommand{\ArXivprefix}{arXiv:}
\providecommand{\URLprefix}{URL: }
\providecommand{\Pubmedprefix}{pmid:}
\providecommand{\doi}[1]{\href{http://dx.doi.org/#1}{\path{#1}}}
\providecommand{\Pubmed}[1]{\href{pmid:#1}{\path{#1}}}
\providecommand{\bibinfo}[2]{#2}
\ifx\xfnm\relax \def\xfnm[#1]{\unskip,\space#1}\fi
\bibitem[{Abiteboul et~al.(1995)Abiteboul, Hull and Vianu}]{AHV95}
\bibinfo{author}{Abiteboul, S.}, \bibinfo{author}{Hull, R.},
  \bibinfo{author}{Vianu, V.}, \bibinfo{year}{1995}.
\newblock \bibinfo{title}{{Foundations of Databases.}}
\newblock \bibinfo{publisher}{Addison-Wesley}.
\bibitem[{Acar et~al.(2019)Acar, Benerecetti and Mogavero}]{ABM19}
\bibinfo{author}{Acar, E.}, \bibinfo{author}{Benerecetti, M.},
  \bibinfo{author}{Mogavero, F.}, \bibinfo{year}{2019}.
\newblock \bibinfo{title}{{Satisfiability in Strategy Logic Can Be Easier than
  Model Checking.}}, in: \bibinfo{booktitle}{AAAI Press19},
  \bibinfo{publisher}{AAAI Press}. pp. \bibinfo{pages}{2638--2645}.
\bibitem[{Amarilli et~al.(2016)Amarilli, Benedikt, Bourhis and Boom}]{ABBV16}
\bibinfo{author}{Amarilli, A.}, \bibinfo{author}{Benedikt, M.},
  \bibinfo{author}{Bourhis, P.}, \bibinfo{author}{Boom, M.V.},
  \bibinfo{year}{2016}.
\newblock \bibinfo{title}{{Query Answering with Transitive and Linear-Ordered
  Data.}}, in: \bibinfo{booktitle}{International Joint Conference on Artificial
  Intelligence'16}, \bibinfo{publisher}{International Joint Conference on
  Artificial Intelligence' \& AAAI Press}. pp. \bibinfo{pages}{893--899}.
\bibitem[{Andresel et~al.(2020)Andresel, Corman, Ortiz, Reutter, Savkovi{\'c}
  and Simkus}]{ACORSS20}
\bibinfo{author}{Andresel, M.}, \bibinfo{author}{Corman, J.},
  \bibinfo{author}{Ortiz, M.}, \bibinfo{author}{Reutter, J.},
  \bibinfo{author}{Savkovi{\'c}, O.}, \bibinfo{author}{Simkus, M.},
  \bibinfo{year}{2020}.
\newblock \bibinfo{title}{{Stable Model Semantics for Recursive SHACL.}}, in:
  \bibinfo{booktitle}{WWW'20}, pp. \bibinfo{pages}{1570--1580}.
\bibitem[{Baader et~al.(2003)Baader, Calvanese, McGuinness, Nardim and
  Patel-Scheider}]{BCMNP03}
\bibinfo{author}{Baader, F.}, \bibinfo{author}{Calvanese, D.},
  \bibinfo{author}{McGuinness, D.}, \bibinfo{author}{Nardim, D.},
  \bibinfo{author}{Patel-Scheider, P.}, \bibinfo{year}{2003}.
\newblock \bibinfo{title}{{The Description Logic Handbook: Theory,
  Implementation, and Applications.}}
\newblock \bibinfo{publisher}{Cambridge University Press}.
\bibitem[{Berger(1966)}]{Ber66}
\bibinfo{author}{Berger, R.}, \bibinfo{year}{1966}.
\newblock \bibinfo{title}{{The Undecidability of the Domino Problem.}}
\newblock \bibinfo{journal}{Memoirs of the American Mathematical Society}
  \bibinfo{volume}{66}, \bibinfo{pages}{1--72}.
\bibitem[{Buil-Aranda et~al.(2013)Buil-Aranda, Arenas, Corcho and
  Polleres}]{BUILARANDA20131}
\bibinfo{author}{Buil-Aranda, C.}, \bibinfo{author}{Arenas, M.},
  \bibinfo{author}{Corcho, O.}, \bibinfo{author}{Polleres, A.},
  \bibinfo{year}{2013}.
\newblock \bibinfo{title}{Federating queries in sparql 1.1: Syntax, semantics
  and evaluation}.
\newblock \bibinfo{journal}{Web Semantics: Science, Services and Agents on the
  World Wide Web} \bibinfo{volume}{18}, \bibinfo{pages}{1 -- 17}.
\newblock \bibinfo{note}{Special Section on the Semantic and Social Web}.
\bibitem[{Carothers and Prud'hommeaux(2014)}]{CP14}
\bibinfo{author}{Carothers, G.}, \bibinfo{author}{Prud'hommeaux, E.},
  \bibinfo{year}{2014}.
\newblock \bibinfo{title}{{RDF 1.1 Turtle.}}
\newblock \bibinfo{type}{{W3C Recommendation}}. {W3C}.
\newblock
  \bibinfo{note}{\url{https://www.w3.org/TR/2014/REC-turtle-20140225/}}.
\bibitem[{ten Cate and Segoufin(2011)}]{CS11}
\bibinfo{author}{ten Cate, B.}, \bibinfo{author}{Segoufin, L.},
  \bibinfo{year}{2011}.
\newblock \bibinfo{title}{{Unary Negation.}}, in: \bibinfo{booktitle}{Symposium
  on Theoretical Aspects of Computer Science'11},
  \bibinfo{publisher}{Leibniz-Zentrum fuer Informatik}. pp.
  \bibinfo{pages}{344--355}.
\bibitem[{ten Cate and Segoufin(2013)}]{CS13}
\bibinfo{author}{ten Cate, B.}, \bibinfo{author}{Segoufin, L.},
  \bibinfo{year}{2013}.
\newblock \bibinfo{title}{{Unary Negation.}}
\newblock \bibinfo{journal}{Logical Methods in Computer Science}
  \bibinfo{volume}{9}, \bibinfo{pages}{1--46}.
\bibitem[{Corman et~al.(2018)Corman, Reutter and Savkovi{\'c}}]{CRS18}
\bibinfo{author}{Corman, J.}, \bibinfo{author}{Reutter, J.},
  \bibinfo{author}{Savkovi{\'c}, O.}, \bibinfo{year}{2018}.
\newblock \bibinfo{title}{{Semantics and Validation of Recursive SHACL.}}, in:
  \bibinfo{booktitle}{ISWC'18}, pp. \bibinfo{pages}{318--336}.
\bibitem[{Cyganiak et~al.(2014)Cyganiak, Wood and Lanthaler}]{CWG14}
\bibinfo{author}{Cyganiak, R.}, \bibinfo{author}{Wood, D.},
  \bibinfo{author}{Lanthaler, M.}, \bibinfo{year}{2014}.
\newblock \bibinfo{title}{{RDF 1.1 Concepts and Abstract Syntax.}}
\newblock \bibinfo{type}{{W3C Recommendation}}. {W3C}.
\newblock
  \bibinfo{note}{\url{http://www.w3.org/TR/2014/REC-rdf11-concepts-20140225/}}.
\bibitem[{Danielski and Kieronski(2019)}]{DK19}
\bibinfo{author}{Danielski, D.}, \bibinfo{author}{Kieronski, E.},
  \bibinfo{year}{2019}.
\newblock \bibinfo{title}{{Finite Satisfiability of Unary Negation Fragment
  with Transitivity.}}, in: \bibinfo{booktitle}{Mathematical Foundations of
  Computer Science'19}, \bibinfo{publisher}{Leibniz-Zentrum fuer Informatik}.
  pp. \bibinfo{pages}{17:1--15}.
\bibitem[{Dar et~al.(1996)Dar, Franklin, Jonsson, Srivastava, Tan
  et~al.}]{dar1996semantic}
\bibinfo{author}{Dar, S.}, \bibinfo{author}{Franklin, M.J.},
  \bibinfo{author}{Jonsson, B.T.}, \bibinfo{author}{Srivastava, D.},
  \bibinfo{author}{Tan, M.}, et~al., \bibinfo{year}{1996}.
\newblock \bibinfo{title}{Semantic data caching and replacement}, in:
  \bibinfo{booktitle}{VLDB}, pp. \bibinfo{pages}{330--341}.
\bibitem[{Donini and Massacci(2000)}]{DM00}
\bibinfo{author}{Donini, F.}, \bibinfo{author}{Massacci, F.},
  \bibinfo{year}{2000}.
\newblock \bibinfo{title}{{ExpTime Tableaux for $\mathcal{ALC}$.}}
\newblock \bibinfo{journal}{Artificial Intelligence} \bibinfo{volume}{124},
  \bibinfo{pages}{87--138}.
\bibitem[{Ebbinghaus and Flum(1995)}]{EF95}
\bibinfo{author}{Ebbinghaus, H.D.}, \bibinfo{author}{Flum, J.},
  \bibinfo{year}{1995}.
\newblock \bibinfo{title}{{Finite Model Theory.}}
\newblock {Perspectives in Mathematical Logic.}, \bibinfo{publisher}{Springer}.
\bibitem[{Eiter and Gottlob(1995)}]{EG95}
\bibinfo{author}{Eiter, T.}, \bibinfo{author}{Gottlob, G.},
  \bibinfo{year}{1995}.
\newblock \bibinfo{title}{{On the Computational Cost of Disjunctive Logic
  Programming: Propositional Case.}}
\newblock \bibinfo{journal}{Annals of Mathematics and Artificial Intelligence}
  \bibinfo{volume}{15}, \bibinfo{pages}{289--323}.
\bibitem[{Gelder et~al.(1991)Gelder, Ross and Schlipf}]{GRS91}
\bibinfo{author}{Gelder, A.V.}, \bibinfo{author}{Ross, K.},
  \bibinfo{author}{Schlipf, J.}, \bibinfo{year}{1991}.
\newblock \bibinfo{title}{{The Well-Founded Semantics for General Logic
  Programs.}}
\newblock \bibinfo{journal}{Journal of the ACM} \bibinfo{volume}{38},
  \bibinfo{pages}{619--649}.
\bibitem[{Gelfond and Lifschitz(1988)}]{GL88}
\bibinfo{author}{Gelfond, M.}, \bibinfo{author}{Lifschitz, V.},
  \bibinfo{year}{1988}.
\newblock \bibinfo{title}{{The Stable Model Semantics For Logic Programming.}}
  , \bibinfo{pages}{1070--1080}.
\bibitem[{Gr{\"a}del(1999)}]{Gra99a}
\bibinfo{author}{Gr{\"a}del, E.}, \bibinfo{year}{1999}.
\newblock \bibinfo{title}{{On The Restraining Power of Guards.}}
\newblock \bibinfo{journal}{Journal of Symbolic Logic} \bibinfo{volume}{64},
  \bibinfo{pages}{1719--1742}.
\bibitem[{Gr{\"a}del et~al.(1997a)Gr{\"a}del, Kolaitis and Vardi}]{GKV97}
\bibinfo{author}{Gr{\"a}del, E.}, \bibinfo{author}{Kolaitis, P.},
  \bibinfo{author}{Vardi, M.}, \bibinfo{year}{1997}a.
\newblock \bibinfo{title}{{On the Decision Problem for Two-Variable First-Order
  Logic.}}
\newblock \bibinfo{journal}{Bulletin of Symbolic Logic} \bibinfo{volume}{3},
  \bibinfo{pages}{53--69}.
\bibitem[{Gr{\"a}del et~al.(1997b)Gr{\"a}del, Otto and Rosen}]{GOR97}
\bibinfo{author}{Gr{\"a}del, E.}, \bibinfo{author}{Otto, M.},
  \bibinfo{author}{Rosen, E.}, \bibinfo{year}{1997}b.
\newblock \bibinfo{title}{{Two-Variable Logic with Counting is Decidable.}},
  in: \bibinfo{booktitle}{Logic in Computer Science'97},
  \bibinfo{publisher}{IEEE Computer Society}. pp. \bibinfo{pages}{306--317}.
\bibitem[{Gupta et~al.(1994)Gupta, Sagiv, Ullman and
  Widom}]{gupta1994constraint}
\bibinfo{author}{Gupta, A.}, \bibinfo{author}{Sagiv, Y.},
  \bibinfo{author}{Ullman, J.D.}, \bibinfo{author}{Widom, J.},
  \bibinfo{year}{1994}.
\newblock \bibinfo{title}{Constraint checking with partial information}, in:
  \bibinfo{booktitle}{Proceedings of the thirteenth ACM SIGACT-SIGMOD-SIGART
  symposium on Principles of database systems}, pp. \bibinfo{pages}{45--55}.
\bibitem[{Horrocks and Sattler(2004)}]{HS04}
\bibinfo{author}{Horrocks, I.}, \bibinfo{author}{Sattler, U.},
  \bibinfo{year}{2004}.
\newblock \bibinfo{title}{{Decidability of SHIQ with Complex Role Inclusion
  Axioms.}}
\newblock \bibinfo{journal}{Artificial Intelligence} \bibinfo{volume}{160},
  \bibinfo{pages}{79--104}.
\bibitem[{Jung et~al.(2018)Jung, Lutz, Martel and Schneider}]{JLMS18}
\bibinfo{author}{Jung, J.}, \bibinfo{author}{Lutz, C.},
  \bibinfo{author}{Martel, M.}, \bibinfo{author}{Schneider, T.},
  \bibinfo{year}{2018}.
\newblock \bibinfo{title}{{Querying the Unary Negation Fragment with Regular
  Path Expressions.}}, in: \bibinfo{booktitle}{International Conference on
  Database Theory'18}, \bibinfo{publisher}{OpenProceedings.org}. pp.
  \bibinfo{pages}{15:1--18}.
\bibitem[{Knublauch and Kontokostas(2017)}]{KK17}
\bibinfo{author}{Knublauch, H.}, \bibinfo{author}{Kontokostas, D.},
  \bibinfo{year}{2017}.
\newblock \bibinfo{title}{{Shapes constraint language (SHACL).}}
\newblock \bibinfo{type}{{W3C Recommendation}}. {W3C}.
\newblock \bibinfo{note}{\url{https://www.w3.org/TR/shacl/}}.
\bibitem[{Konstantinidis and Ambite(2011)}]{konstantinidis2011scalable}
\bibinfo{author}{Konstantinidis, G.}, \bibinfo{author}{Ambite, J.L.},
  \bibinfo{year}{2011}.
\newblock \bibinfo{title}{Scalable query rewriting: a graph-based approach},
  in: \bibinfo{booktitle}{Proceedings of the 2011 ACM SIGMOD International
  Conference on Management of data}, pp. \bibinfo{pages}{97--108}.
\bibitem[{Kr{\"o}tzsch et~al.(2012)Kr{\"o}tzsch, Siman{\v{c}}{\'\i}k and
  Horrocks}]{KSH12}
\bibinfo{author}{Kr{\"o}tzsch, M.}, \bibinfo{author}{Siman{\v{c}}{\'\i}k, F.},
  \bibinfo{author}{Horrocks, I.}, \bibinfo{year}{2012}.
\newblock \bibinfo{title}{{A Description Logic Primer.}}
\newblock \bibinfo{type}{Technical Report}. arXiv.
\bibitem[{Leinberger et~al.(2020)Leinberger, Seifer, Rienstra, L{\"a}mmel and
  Staab}]{martin2020shapecontainment}
\bibinfo{author}{Leinberger, M.}, \bibinfo{author}{Seifer, P.},
  \bibinfo{author}{Rienstra, T.}, \bibinfo{author}{L{\"a}mmel, R.},
  \bibinfo{author}{Staab, S.}, \bibinfo{year}{2020}.
\newblock \bibinfo{title}{{Deciding SHACL Shape Containment through Description
  Logics Reasoning}}, in: \bibinfo{booktitle}{The Semantic Web -- ISWC 2020},
  \bibinfo{publisher}{Springer International Publishing}. pp.
  \bibinfo{pages}{366--383}.
\newblock \bibinfo{note}{(\emph{this volume})}.
\bibitem[{Leinberger et~al.(2019)Leinberger, Seifer, Schon, L{\"a}mmel and
  Staab}]{leinberger2019type}
\bibinfo{author}{Leinberger, M.}, \bibinfo{author}{Seifer, P.},
  \bibinfo{author}{Schon, C.}, \bibinfo{author}{L{\"a}mmel, R.},
  \bibinfo{author}{Staab, S.}, \bibinfo{year}{2019}.
\newblock \bibinfo{title}{Type checking program code using shacl}, in:
  \bibinfo{booktitle}{International Semantic Web Conference},
  \bibinfo{organization}{Springer}. pp. \bibinfo{pages}{399--417}.
\bibitem[{Lenzerini(2002)}]{lenzerini2002data}
\bibinfo{author}{Lenzerini, M.}, \bibinfo{year}{2002}.
\newblock \bibinfo{title}{Data integration: A theoretical perspective}, in:
  \bibinfo{booktitle}{Proceedings of the twenty-first ACM SIGMOD-SIGACT-SIGART
  symposium on Principles of database systems}, pp. \bibinfo{pages}{233--246}.
\bibitem[{Levy and Sagiv(1993)}]{levy1993queries}
\bibinfo{author}{Levy, A.Y.}, \bibinfo{author}{Sagiv, Y.},
  \bibinfo{year}{1993}.
\newblock \bibinfo{title}{Queries independent of updates}, in:
  \bibinfo{booktitle}{VLDB}, \bibinfo{organization}{Citeseer}. pp.
  \bibinfo{pages}{171--181}.
\bibitem[{Libkin(2004)}]{Lib04}
\bibinfo{author}{Libkin, L.}, \bibinfo{year}{2004}.
\newblock \bibinfo{title}{{Elements of Finite Model Theory.}}
\newblock {Texts in Theoretical Computer Science.},
  \bibinfo{publisher}{Springer}.
\bibitem[{Lutz(2005)}]{Lut04}
\bibinfo{author}{Lutz, C.}, \bibinfo{year}{2005}.
\newblock \bibinfo{title}{{An Improved {NExpTime}-Hardness Result for
  Description Logic $\mathcal{ALC}$ Extended with Inverse Roles, Nominals, and
  Counting}}.
\newblock \bibinfo{type}{Technical Report} \bibinfo{number}{05-05}. Dresden
  University of Technology, Dresden, Germany.
\bibitem[{Marek and Subrahmanian(1992)}]{MS92}
\bibinfo{author}{Marek, W.}, \bibinfo{author}{Subrahmanian, V.},
  \bibinfo{year}{1992}.
\newblock \bibinfo{title}{{The Relationship Between Stable, Supported, Default
  and Autoepistemic Semantics for General Logic Programs.}}
\newblock \bibinfo{journal}{Theoretical Computer Science}
  \bibinfo{volume}{103}, \bibinfo{pages}{365--386}.
\bibitem[{McDermott(1982)}]{McD82}
\bibinfo{author}{McDermott, D.}, \bibinfo{year}{1982}.
\newblock \bibinfo{title}{{Non-Monotonic Logic II: Nonmonotonic Modal
  Theories.}}
\newblock \bibinfo{journal}{Journal of the ACM} \bibinfo{volume}{29},
  \bibinfo{pages}{33--57}.
\bibitem[{Minker(1988)}]{Min88}
\bibinfo{author}{Minker, J.}, \bibinfo{year}{1988}.
\newblock \bibinfo{title}{{Foundations of Deductive Databases and Logic
  Programming.}}
\newblock \bibinfo{publisher}{Elsevier/Morgan Kaufmann}.
\bibitem[{Mortimer(1975)}]{Mor75}
\bibinfo{author}{Mortimer, M.}, \bibinfo{year}{1975}.
\newblock \bibinfo{title}{{On Languages with Two Variables.}}
\newblock \bibinfo{journal}{Mathematical Logic Quarterly} \bibinfo{volume}{21},
  \bibinfo{pages}{135--140}.
\bibitem[{Pareti and Konstantinidis(2022)}]{pareti2021shaclreview}
\bibinfo{author}{Pareti, P.}, \bibinfo{author}{Konstantinidis, G.},
  \bibinfo{year}{2022}.
\newblock \bibinfo{title}{A {R}eview of {SHACL}: {F}rom {D}ata {V}alidation to
  {S}chema {R}easoning for {RDF} {G}raphs}, in: \bibinfo{editor}{{\v{S}}imkus,
  M.}, \bibinfo{editor}{Varzinczak, I.} (Eds.), \bibinfo{booktitle}{Reasoning
  Web. Declarative Artificial Intelligence}, \bibinfo{publisher}{Springer
  International Publishing}, \bibinfo{address}{Cham}. pp.
  \bibinfo{pages}{115--144}.
\bibitem[{Pareti et~al.(2020)Pareti, Konstantinidis, Mogavero and
  Norman}]{pareti2020}
\bibinfo{author}{Pareti, P.}, \bibinfo{author}{Konstantinidis, G.},
  \bibinfo{author}{Mogavero, F.}, \bibinfo{author}{Norman, T.J.},
  \bibinfo{year}{2020}.
\newblock \bibinfo{title}{Shacl satisfiability and containment}, in:
  \bibinfo{booktitle}{The Semantic Web -- ISWC 2020},
  \bibinfo{publisher}{Springer International Publishing},
  \bibinfo{address}{Cham}. pp. \bibinfo{pages}{474--493}.
\bibitem[{Pareti et~al.(2019)Pareti, Konstantinidis, Norman and
  {\c{S}}ensoy}]{pareti2019shacl}
\bibinfo{author}{Pareti, P.}, \bibinfo{author}{Konstantinidis, G.},
  \bibinfo{author}{Norman, T.J.}, \bibinfo{author}{{\c{S}}ensoy, M.},
  \bibinfo{year}{2019}.
\newblock \bibinfo{title}{Shacl constraints with inference rules}, in:
  \bibinfo{booktitle}{International Semantic Web Conference},
  \bibinfo{organization}{Springer}. pp. \bibinfo{pages}{539--557}.
\bibitem[{Pratt{-}Hartmann(2005)}]{Pra05}
\bibinfo{author}{Pratt{-}Hartmann, I.}, \bibinfo{year}{2005}.
\newblock \bibinfo{title}{{Complexity of the Two-Variable Fragment with
  Counting Quantifiers.}}
\newblock \bibinfo{journal}{Journal of Logic, Language, and Information'}
  \bibinfo{volume}{14}, \bibinfo{pages}{369--395}.
\bibitem[{Pratt{-}Hartmann(2010)}]{Pra10}
\bibinfo{author}{Pratt{-}Hartmann, I.}, \bibinfo{year}{2010}.
\newblock \bibinfo{title}{{The Two-Variable Fragment with Counting
  Revisited.}}, in: \bibinfo{booktitle}{Workshop on Logic, Language,
  Information and Computation'10}, \bibinfo{publisher}{Springer}. pp.
  \bibinfo{pages}{42--54}.
\bibitem[{Przymusinski(1989)}]{Prz89}
\bibinfo{author}{Przymusinski, T.}, \bibinfo{year}{1989}.
\newblock \bibinfo{title}{{On the Declarative and Procedural Semantics of Logic
  Programs.}}
\newblock \bibinfo{journal}{Journal of Automated Reasoning}
  \bibinfo{volume}{5}, \bibinfo{pages}{167--205}.
\bibitem[{Robinson(1971)}]{Rob71}
\bibinfo{author}{Robinson, R.}, \bibinfo{year}{1971}.
\newblock \bibinfo{title}{{Undecidability and Nonperiodicity for Tilings of the
  Plane.}}
\newblock \bibinfo{journal}{Inventiones Mathematicae} \bibinfo{volume}{12},
  \bibinfo{pages}{177--209}.
\bibitem[{Ross and Topor(1988)}]{RT88}
\bibinfo{author}{Ross, K.}, \bibinfo{author}{Topor, R.}, \bibinfo{year}{1988}.
\newblock \bibinfo{title}{{Inferring Negative Information from Disjunctive
  Databases.}}
\newblock \bibinfo{journal}{Journal of Automated Reasoning}
  \bibinfo{volume}{4}, \bibinfo{pages}{397--424}.
\bibitem[{Rudolph et~al.(2008)Rudolph, Kr{\"o}tzsch and Hitzler}]{RKH08}
\bibinfo{author}{Rudolph, S.}, \bibinfo{author}{Kr{\"o}tzsch, M.},
  \bibinfo{author}{Hitzler, P.}, \bibinfo{year}{2008}.
\newblock \bibinfo{title}{{Cheap Boolean Role Constructors for Description
  Logics.}}, in: \bibinfo{booktitle}{European Conference on Logics in
  Artificial Intelligence'08}, \bibinfo{publisher}{Springer}. pp.
  \bibinfo{pages}{362--374}.
\bibitem[{Sakama(1989)}]{Sak89}
\bibinfo{author}{Sakama, C.}, \bibinfo{year}{1989}.
\newblock \bibinfo{title}{{Possible Model Semantics for Disjunctive
  Databases.}}, in: \bibinfo{booktitle}{Deductive and Object-Oriented
  Databases'89}, \bibinfo{publisher}{North-Holland/Elsevier}. pp.
  \bibinfo{pages}{369--383}.
\bibitem[{Sakama and Inoue(2009)}]{SI09}
\bibinfo{author}{Sakama, C.}, \bibinfo{author}{Inoue, K.},
  \bibinfo{year}{2009}.
\newblock \bibinfo{title}{{Brave Induction: A Logical Framework for Learning
  from Incomplete Information.}}
\newblock \bibinfo{journal}{Machine Learning} \bibinfo{volume}{76},
  \bibinfo{pages}{3--35}.
\bibitem[{Sato et~al.(2020)Sato, Sakama and Inoue}]{SSI20}
\bibinfo{author}{Sato, T.}, \bibinfo{author}{Sakama, C.},
  \bibinfo{author}{Inoue, K.}, \bibinfo{year}{2020}.
\newblock \bibinfo{title}{{From 3-valued Semantics to Supported Model
  Computation for Logic Programs in Vector Spaces.}}, in:
  \bibinfo{booktitle}{20}, \bibinfo{publisher}{SciTePress}. pp.
  \bibinfo{pages}{758--765}.
\bibitem[{Sattler and Vardi(2001)}]{SV01}
\bibinfo{author}{Sattler, U.}, \bibinfo{author}{Vardi, M.},
  \bibinfo{year}{2001}.
\newblock \bibinfo{title}{{The Hybrid $\mu$-Calculus.}}, in:
  \bibinfo{booktitle}{International Joint Conference on Automated
  Reasoning'01}, \bibinfo{publisher}{Springer}. pp. \bibinfo{pages}{76--91}.
\bibitem[{Schild(1991)}]{SCh91}
\bibinfo{author}{Schild, K.}, \bibinfo{year}{1991}.
\newblock \bibinfo{title}{{A Correspondence Theory for Terminological Logics:
  Preliminary Report.}}, in: \bibinfo{booktitle}{International Joint Conference
  on Artificial Intelligence'91}, \bibinfo{publisher}{Morgan Kaufmann}. pp.
  \bibinfo{pages}{466--471}.
\bibitem[{Schmidt-Schaub\ss{}(1989)}]{Sch89}
\bibinfo{author}{Schmidt-Schaub\ss{}, M.}, \bibinfo{year}{1989}.
\newblock \bibinfo{title}{{Subsumption in KL-ONE is Undecidable.}}, in:
  \bibinfo{booktitle}{Knowledge Representation and Reasoning'89},
  \bibinfo{publisher}{Morgan Kaufmann}. pp. \bibinfo{pages}{421--431}.
\bibitem[{Tobies(2000)}]{Tob00}
\bibinfo{author}{Tobies, S.}, \bibinfo{year}{2000}.
\newblock \bibinfo{title}{{The Complexity of Reasoning with Cardinality
  Restrictions and Nominals in Expressive Description Logics.}}
\newblock \bibinfo{journal}{Journal of Artificial Intelligence Research}
  \bibinfo{volume}{12}, \bibinfo{pages}{199--217}.
\bibitem[{Trakhtenbrot(1950)}]{Tra50}
\bibinfo{author}{Trakhtenbrot, B.}, \bibinfo{year}{1950}.
\newblock \bibinfo{title}{{Impossibility of an Algorithm for the Decision
  Problem in Finite Classes.}}
\newblock \bibinfo{journal}{Proceedings of the USSR Academy of Sciences}
  \bibinfo{volume}{70}, \bibinfo{pages}{569--572}.
\bibitem[{{W3C OWL Working Group}(2012)}]{OWL2}
\bibinfo{author}{{W3C OWL Working Group}}, \bibinfo{year}{2012}.
\newblock \bibinfo{title}{{OWL 2 Web Ontology Language Document Overview
  (Second Edition)}}.
\newblock \bibinfo{type}{{W3C} Recommendation}. W3C.
\newblock
  \bibinfo{note}{Https://www.w3.org/TR/2012/REC-owl2-overview-20121211/}.
\bibitem[{Wang(1961)}]{Wan61}
\bibinfo{author}{Wang, H.}, \bibinfo{year}{1961}.
\newblock \bibinfo{title}{{Proving Theorems by Pattern Recognition II.}}
\newblock \bibinfo{journal}{Bell System Technical Journal}
  \bibinfo{volume}{40}, \bibinfo{pages}{1--41}.

\end{thebibliography}

\end{document}